\documentclass{article}

\usepackage[preprint, nonatbib]{neurips_2020}

\usepackage[utf8]{inputenc} 
\usepackage[T1]{fontenc}    
\usepackage{hyperref}       
\usepackage{url}            
\usepackage{booktabs}       
\usepackage{amsfonts}       
\usepackage{nicefrac}       
\usepackage{microtype}      
\usepackage{bm}   
\usepackage{amsmath}
\usepackage{caption}
\usepackage{comment}
\usepackage{amsthm}
\usepackage{amssymb}
\usepackage{subcaption}
\usepackage{makecell}
\usepackage{empheq}
\usepackage{xfrac}
\usepackage{algorithm}
\usepackage{algorithmic}
\usepackage{multirow}
\usepackage{xcolor}
\usepackage[utf8]{inputenc}
\usepackage[english]{babel}
\usepackage{wrapfig}
\usepackage{subfiles}
\newtheorem{assumption}{Assumption}
\newtheorem{theorem}{Theorem}

\newenvironment{app_assumption}[1]{\par\noindent\textbf{Assumption }\space#1}{}

\newenvironment{app_theorem}[1]{\par\noindent\textbf{Theorem:}\space#1}{}
\newenvironment{lemma}[1]{\par\noindent\textbf{Lemma}:\space#1}{}
\newenvironment{corollary}[1]{\par\noindent\textbf{Corollary:}\space#1}{}

\title{Efficient Semi-Implicit Variational Inference}

\author{%
  Vincent Moens \\
  Huawei R\&D UK \\
  \texttt{vincent.moens@huawei.com}
   \And
   Hang Ren \\
   Huawei R\&D UK \\
   \texttt{hang.ren1@huawei.com} \\
   \AND
   Alexandre Maraval \\
   Huawei R\&D UK \\
   \texttt{alexandre.maravel@huawei.com} \\
   \And
   Rasul Tutunov \\
   Huawei R\&D UK \\
   \texttt{rasul.tutunov@huawei.com} \\
   \And
   Jun Wang \\
   Huawei R\&D UK \\
   University College London \\
   \texttt{w.j@huawei.com} \\
   \And
   Haitham Ammar \thanks{Honorary position at UCL.} \\
   Huawei R\&D UK \\
   University College London \\
   \texttt{haitham.ammar@huawei.com} \\
}

\begin{document}

\maketitle

\begin{abstract}
In this paper, we propose CI-VI an efficient and scalable solver for semi-implicit variational inference (SIVI). Our method, first, maps SIVI's evidence lower bound (ELBO) to a form involving a nonlinear functional nesting of expected values and then develops a rigorous optimiser capable of correctly handling bias inherent to nonlinear nested expectations using an extrapolation-smoothing mechanism coupled with gradient sketching. Our theoretical results demonstrate convergence to a stationary point of the ELBO in general non-convex settings typically arising when using deep network models and an order of $\mathcal{O}(t^{-\sfrac{4}{5}})$ gradient-bias-vanishing rate. We believe these results generalise beyond the specific nesting arising from SIVI to other forms. Finally, in a set of experiments, we demonstrate the effectiveness of our algorithm in approximating complex posteriors on various data-sets including those from natural language processing.
\end{abstract}
\section{Introduction}
Variational Inference (VI) is an approximate Bayesian inference framework that recasts reasoning about latent variable models as an instance of numerical optimisation~\cite{Jaakkola1999, Saul1996}. This is achieved by positing a class of variational distributions and optimising an evidence lower-bound (ELBO) that involves log-joint densities rather than intractable posteriors. Classical VI introduces simplifying assumptions (e.g., mean-field and/or conditional conjugacy) to allow for tractable optimisation leading to algorithms that ascend to a stationary point of the  ELBO~\cite{Blei2017, Jaakkola1999}. Though successful in many applications~\cite{multiviewRL, MTRL, Apps}, such assumptions, unfortunately, restrict  ``representation-power'' and thus, lead to models that underestimate the variance of the posterior.

Realising this problem, numerous frameworks aiming at expanding expressiveness of variational families have been proposed. Works in~\cite{Giordano2015, gregor15, Han2016, Hoffman2012,  Jaakkola1999, maaloe16, Ranganath2016, Rezende2015,  Saul1996, Tran2016, Tran2017}, for instance, relax mean-field assumptions and attempt to restore some dependencies in variational distributions but still require analytical probability density functions. 
Others in~\cite{Huszar2017, Li2018, Mohamed2017,  Shi2018, Tran2017}, moreover, introduce implicit models by sampling noise vectors and propagating these through deep networks~\cite{UIVI}. 
These techniques do lead to distributions that are implicit but render computing log-variational densities and their gradients intractable. 
For this reason, authors resort to density ratio estimation; a methodology hard-to-scale and stabilise in high-dimensional settings. 

To avoid density ratio estimation, recent work in~\cite{SIVI} proposed semi-implicit variational inference (SIVI) as a hierarchical framework that obtains variational distributions through a mixing parameter accompanied by known-noise priors~\cite{UIVI}. 
Exploiting this definition, original SIVI optimises a sequence of lower bounds that asymptotically converge to the original ELBO.
Rather than focusing on approximate lower (upper) bounds, an unbiased estimator of the exact gradient of SIVI's ELBO has been newly achieved~\cite{UIVI} through a Hamiltonian Monte-Carlo simulator that draws samples from a reverse conditional acquiring state-of-the-art status. 
Though achieving better results than previous works, unbiased estimators based on Markov Chain Monte Carlo (MCMC) easily become computationally expensive in high-dimensional regimes. Hence, an efficient solver for models supporting semi-implicit variational family distributions, largely, remains an open problem to which we contribute in this paper.

Tackling the above problem, we present CI-VI, an efficient semi-implicit solver with rigorous theoretical guarantees capable of scaling to high-dimensional scenarios. Our method first maps the ELBO in SIVI to a nonlinear nested expectation (or compositional) form\footnote{Please notice the decoupling between the inner and outer functions $f_{\nu}(\cdot)$ and $g_{\omega}(\cdot)$. This requires us to further analyse SIVI's ELBO; see Section~\ref{Sec:Solver}.}, i.e., $\mathbb{E}_{\nu}[f_{\nu}\left(\mathbb{E}_{\omega}[\bm{g}_{\omega}(\bm{\theta})]\right)]$ with $\nu$ and $\omega$ being random variables\footnote{Please note we do not assume any independence between $\nu$ and $\omega$}, $f_{\nu}:\mathbb{R}^{n}\to\mathbb{R}$, $\bm{g}_{w}:\mathbb{R}^{p}\to\mathbb{R}^{n}$ are smooth, not necessarily convex functions,  and $\bm{\theta}$ the optimisation parameter -- and then devises an algorithm capable of handling the bias resulting from the non-linear composition of $\mathbb{E}_{\omega}[\cdot]$ and $\mathbb{E}_{\nu}[\cdot]$ through $f_{\nu}(\cdot)$. 
To do so, we introduce an extrapolation-smoothing step to an ADAM-like~\cite{kingma2014adam, Rasul} solver and, in turn, show vanishing gradient bias in the order of $\mathcal{O}(t^{-\sfrac{4}{5}})$ ultimately enabling convergence to a stationary point. Our resulting algorithm shares similarities to recent work from~\cite{Rasul} but unlocks novel theoretical results analysing gradient-bias terms. 
Though similar in spirit to~\cite{Rasul}, we realise that the original version presented in Algorithm~\ref{Algo:ADAM} fails to scale to high-dimensions due to the need of computing large matrix-vector products (see Gradients' instructions in Algorithm~\ref{Algo:ADAM}). 
Rectifying this problem, we lastly anchor a gradient-sketching mechanism allowing for batched matrix-vector products, and, further, study the theoretical outcomes of such a combination. Finally, we conduct an in-depth empirical study demonstrating that CI-VI outperforms other algorithms from SIVI, nested Monte-Carlo~\cite{NestedMC}, and compositional optimisation~\cite{Rasul, Mengdi} literature.

\vspace{-1em}

\section{Compositional implicit variational inference (CI-VI):}
In this section, we demonstrate the connection between semi-implicit variational inference and compositional stochastic optimisation\footnote{We use compositional and nested stochastic optimisation interchangeably.}. We, first, present the semi-implicit inference framework as detailed in~\cite{UIVI, SIVI}, and then link to compositional optimisation in Section~\ref{Sec:CompForm}. We, finally, feature an efficient adaptive solver and provide its relevant theoretical guarantees in Section~\ref{Sec:Solver}. 
\subsection{Semi-implicit variational inference}\label{subsec:semi_implicit_vi}
To approximate the posterior $p(\bm{z}|\bm{x})$ of a probabilistic model $p(\bm{x}, \bm{z})$, we define a semi-implicit variational distribution $q_{\bm{\theta}}(\bm{z})$ in a hierarchical fashion using a mixing parameter as introduced in~\cite{UIVI}: 
\begin{equation}
\label{Eq:HVI}
    \bm{\epsilon} \sim q(\bm{\epsilon}), \ \ \ \ \bm{z} \sim q_{\bm{\theta}}(\bm{z}| \bm{\epsilon}) \ \ \ \ \implies \ \ \ \ q_{\bm{\theta}}(\bm{z}) = \underbrace{\int q_{\bm{\theta}}(\bm{z}|\bm{\epsilon})q(\bm{\epsilon}) d\bm{\epsilon}}_{\text{intractable}}.
\end{equation}
Equation~\ref{Eq:HVI} reveals the reason behind $q_{\bm{\theta}}(\bm{z})$ being implicit as we can obtain latent variable samples through $\bm{\epsilon}$ but can not (tractably) compute the integral especially when using deep networks in representing $q_{\bm{\theta}}(\bm{z}|\bm{\epsilon})$. In this work, we impose two standard assumptions on the nature of the variational distribution as previously explained in~\cite{UIVI, SIVI}. The first assumes that $q_{\bm{\theta}}(\bm{z}|\bm{\epsilon})$ is reparameterisable. That is, to obtain samples from $\bm{z} \sim q_{\bm{\theta}}(\bm{z}|\bm{\epsilon})$, one can draw an auxiliary variable $\bm{u}$ and then set $\bm{z}$ as a deterministic function $h_{\bm{\theta}}(\cdot)$ of the sampled $\bm{u}$ i.e., $\bm{u} \sim q(\bm{u}),  \ \bm{z} = h_{\bm{\theta}}(\bm{u}; \bm{\epsilon}) \ \equiv  \ \bm{z} \sim q_{\bm{\theta}} (\bm{z}|\bm{\epsilon})$. The second, moreover, assumes that we can evaluate the log-density of the conditional i.e., $\log q_{\bm{\theta}}(\bm{z}|\bm{\epsilon})$ as well as its gradient. As noted in~\cite{UIVI}, such an assumption is not strong in that it holds for many reparameterisable distributions, e.g., Gaussian, Laplace, exponential, and many others. Analogous to standard variational inference, model parameters are fit by minimising the negate of an evidence-lower (ELBO) bound that can be derived as follows: 
\begin{equation}
\label{Eq:ELBO}
    \log p(\bm{x}) = \log \int \frac{q_{\bm{\theta}}(\bm{z})}{q_{\bm{\theta}}(\bm{z})} p(\bm{x}, \bm{z}) d \bm{z}  \geq \int q_{\bm{\theta}}(\bm{z}) \log \left[\frac{p(\bm{x},\bm{z})}{q_{\bm{\theta}}(\bm{z})}\right] d\bm{z} = \mathbb{E}_{q_{\bm{\theta}}(\bm{z})} \left[\log\frac{ p (\bm{x}, \bm{z})}{q_{\bm{\theta}}(\bm{z})} \right].
\end{equation}
Contrary to classical VI that assumes tractable expectations, SIVI introduces additional intractability (e.g., in the entropy term) due to the implicit nature of the variational distribution defined in Equation~\ref{Eq:HVI}. To tackle such intractability, recently the authors in~\cite{UIVI} proposed writing the gradient of the entropy term as an expectation and following an MCMC sampler~\cite{MCMC}, e.g., Hamiltonian Monte Carlo~\cite{HMC} to estimate the gradient of the ELBO. MCMC methods, however, are known to be computationally expensive and can exhibit high variance as they assume no model -- see Section~\ref{Sec:Exp} for a detailed comparison. Rather than following MCMC, in this paper we contribute by showing that the semi-implicit ELBO can be written as an instance of compositional optimisation and devise an adaptive and efficient solver with rigorous theoretical guarantees.

\subsection{SIVI in a compositional nested form}\label{Sec:CompForm}
In this section, we present a novel connection mapping implicit variational inference to compositional stochastic optimisation, paving-the-way for an efficient and scalable solver that we later develop in Section~\ref{Sec:Solver}. To do so, we start by plugging-in the variational distribution from Equation~\ref{Eq:HVI} in the inner-part of the ELBO (i.e., Equation~\ref{Eq:ELBO}) to get: $\log p(\bm{x}) \geq \mathbb{E}_{\bm{z} \sim q_{\bm{\theta}}(\bm{z})} \left[\log\frac{ p (\bm{x}, \bm{z})}{q_{\bm{\theta}}(\bm{z})} \right] = \mathbb{E}_{\bm{z} \sim q_{\bm{\theta}}(\bm{z})} \left[\log\frac{ p (\bm{x}, \bm{z})}{\mathbb{E}_{\bm{\hat{\epsilon}}\sim q(\bm{\epsilon})}[q_{\bm{\theta}}(\bm{z}|\bm{\hat{\epsilon}})]} \right]$, where we used $\bm{\hat{\epsilon}}$ to denote an inner-random variable also sampled according to $q(\bm{\epsilon})$. As noted earlier, we assume that the conditional $q_{\bm{\theta}}(\bm{z}|\bm{\epsilon})$ is reparameterisable through an auxiliary variable $\bm{u}$ and a deterministic function $h_{\bm{\theta}}(\bm{u};\bm{\epsilon})$ with $\bm{\epsilon} \sim q(\bm{\epsilon})$ but independent from $\bm{\hat{\epsilon}}$. Rather than reparametrising both inner and outer expectations, we only reparameterise the outer expectation leading us to: 
\begin{align*}
    \mathbb{E}_{\bm{z} \sim q_{\bm{\theta}}(\bm{z})} \left[\log\frac{ p (\bm{x}, \bm{z})}{\mathbb{E}_{\bm{\hat{\epsilon}}\sim q(\bm{\epsilon})}[q_{\bm{\theta}}(\bm{z}|\bm{\hat{\epsilon}})]} \right] &= \mathbb{E}_{\bm{u} \sim q(\bm{u}), \bm{\epsilon} \sim q(\bm{\epsilon})}\left[\log \frac{p(\bm{x}|h_{\bm{\theta}}(\bm{u}; \bm{\epsilon}))p(h_{\bm{\theta}}(\bm{u};\bm{\epsilon}))}{\mathbb{E}_{\bm{\hat{\epsilon}} \sim q(\bm{\epsilon})}[q_{\bm{\theta}}\left(h_{\bm{\theta}}(\bm{u};\bm{\epsilon})|\bm{\hat{\epsilon}}\right)]}\right] \\ \nonumber
   &\hspace{15em} \text{(reparam.)}
\end{align*}
Remembering that in variational inference one minimises the negate of the ELBO, we can further write:
\small
\begin{align*}
    - \mathbb{E}_{\bm{u} \sim q(\bm{u}), \bm{\epsilon} \sim q(\bm{\epsilon})}\left[\log \frac{p(\bm{x}|h_{\bm{\theta}}(\bm{u}; \bm{\epsilon}))p(h_{\bm{\theta}}(\bm{u};\bm{\epsilon}))}{\mathbb{E}_{\bm{\hat{\epsilon}}}[q_{\bm{\theta}}\left(h_{\bm{\theta}}(\bm{u};\bm{\epsilon})|\bm{\hat{\epsilon}}\right)]}\right] &= \mathbb{E}_{q(\bm{u}), q(\bm{\epsilon})}\left[\log \frac{\mathbb{E}_{\bm{\hat{\epsilon}}\sim q(\bm{\epsilon})}[q_{\bm{\theta}}\left(h_{\bm{\theta}}(\bm{u};\bm{\epsilon})|\bm{\hat{\epsilon}}\right)]}{p(\bm{x}|h_{\bm{\theta}}(\bm{u}; \bm{\epsilon}))p(h_{\bm{\theta}}(\bm{u};\bm{\epsilon}))}\right] \\
    & = \mathbb{E}_{\bm{\mu} \sim q(\bm{\mu})}[\log \mathbb{E}_{\bm{\hat{\epsilon}}}[\mathcal{J}_{\bm{\mu}, \bm{\hat{\epsilon}}}(\bm{\theta})]],  
\end{align*}
\normalsize
where we used $\bm{\mu} = \{\bm{u}, \bm{\epsilon}\}$ to concatenate outer random variables with $q(\bm{\mu}) = q(\bm{u}, \bm{\epsilon}) = q(\bm{u})q(\bm{\epsilon})$. Hence, the optimisation problem involved in SIVI can be written as: $ \min_{\bm{\theta}} \mathbb{E}_{\bm{\mu}}  [\log \mathbb{E}_{\bm{\hat{\epsilon}}}[\mathcal{J}_{\bm{\mu}, \bm{\hat{\epsilon}}}(\bm{\theta})]]$. Superficially, the aforementioned problem looks compositional in nature due to the non-linear (through the logarithm) nesting of both expectations. It is worth emphasizing, however, that the standard nested form introduced in Section~1 assumes an inherent decoupling between the inner and outer expectations, i.e., $\mathbb{E}_{\nu}[f_{\nu}\left(\mathbb{E}_{\omega}[\bm{g}_{\omega}(\bm{\theta})]\right)]$. In light of this realisation, we now introduce a formalisation capable of achieving this decoupling. To do so, we consider a pool of $n$-$\bm{\mu}$ samples distributed according to $q(\bm{\mu})$: $\text{Pool} = \{ \bm{\mu}_{i} =  \left \langle \bm{u}_{i}, \bm{\epsilon}_{i} \right\rangle \}_{i=1}^{n}$. Now, we define a vector-valued function, $\bm{g}_{\bm{\hat{\epsilon}}}(\bm{\theta})$, of size $n$ corresponding to the evaluations of $\mathcal{J}_{\bm{\mu}, \bm{\hat{\epsilon}}}(\bm{\theta})$ on each of the samples from the pool, i.e., $\forall j \in [1,n]$ we have:
$\bm{g}_{\bm{\hat{\epsilon}}} (\bm{\theta}) = [\mathcal{J}_{\bm{\mu}_{1}, \bm{\hat{\epsilon}}}(\bm{\theta}), \dots, \mathcal{J}_{\bm{\mu}_{n}, \bm{\hat{\epsilon}}}(\bm{\theta})]^{\mathsf{T}}, \ \text{with $  \mathcal{J}_{\bm{\mu}_{j}, \bm{\hat{\epsilon}}}(\bm{\theta})= \frac{q_{\bm{\theta}}\left(h_{\bm{\theta}}(\bm{u}_{j};\bm{\epsilon}_{j})|\bm{\hat{\epsilon}}\right)}{p(\bm{x}|h_{\bm{\theta}}(\bm{u}_{j}; \bm{\epsilon}_{j}))p(h_{\bm{\theta}}(\bm{u}_{j};\bm{\epsilon}_{j}))}$.}$ To achieve the decoupling between inner an outer expectations, we allow  $f_{\nu}(\bm{y}) = [\log \bm{y}]^{\mathsf{T}}\bm{e}_{\nu}$ with $\bm{e}_{\nu}$ being the $\nu$'th basis vector in $\mathbb{R}^n$, i.e., a vector of all zeros except a value of one in the $\nu$'th position. Sampling uniformly from the pool, we can finally write SIVI's optimisation problem in a compositional form as:
\begin{equation}
\label{Eq:Final}
    \min_{\bm{\theta}} \mathbb{E}_{\nu \sim \text{Uniform}[1,n]}\left[f_{\nu}\left(\mathbb{E}_{\bm{\hat{\epsilon}} \sim q(\bm{\epsilon})}[\bm{g}_{\bm{\hat{\epsilon}}}(\bm{\theta})]\right)\right] \equiv \min_{\bm{\theta}} \mathcal{L}(\bm{\theta}), 
\end{equation}
with $f_{\nu}(\bm{y}) = [\log \bm{y}]^{\mathsf{T}}\bm{e}_{\nu}$ and $\bm{g}_{\bm{\hat{\epsilon}}} (\bm{\theta}) = [\mathcal{J}_{\bm{\mu}_{1}, \bm{\hat{\epsilon}}}(\bm{\theta}), \dots, \mathcal{J}_{\bm{\mu}_{n}, \bm{\hat{\epsilon}}}(\bm{\theta})]^{\mathsf{T}}$. Clearly, our problem becomes exact only when assuming an infinite number of samples, i.e., $n \rightarrow \infty$. As such, one would naturally choose $n$ to be large-enough for a small variance estimator of the loss. This, in turn, adds complexity in designing a solver that now has to handle \emph{nested expectations and high-dimensional   regimes.} In the next section, we introduce such an algorithm through a novel combination of extrapolation-smoothing and gradient sketching mechanisms.

\subsection{An adaptive solver}\label{Sec:Solver}
When designing a solver for the optimisation problem in Equation~\ref{Eq:Final}, we consider three essential criteria. First, we would like a simple-to-implement (i.e., single loop) yet effective and scalable algorithm. Second, we aim to have a bias-controlling procedure\footnote{Here, it is to be understood that bias is due to estimating nonlinear nested expectations.} and third, we need a rigorous and theoretically-grounded solver. When surveying optimisation literature, we realise that a promising direction is a first-order method as opposed to zero~\cite{gabillon2019derivative} or second-order~\cite{TutunovBJ16} ones. Such a realisation is grounded in the fact that first-order methods only require gradient information, are typically simple to implement through a single-loop, and perform competitively in large-scale machine learning applications~\cite{ADAgrad, kingma2014adam, RMSProp}. Among first-order methods, one can further categorise adaptive~\cite{ADAgrad, RMSProp, zeiler2012adadelta} and momentum-based~\cite{APG, Nesterov} algorithms. In spite of numerous theoretical developments~\cite{allenzhu2016variance, fang2018spider, liu2018stochastically}, ADAM – an adaptive optimiser originally proposed in~\cite{kingma2014adam}, and then theoretically grounded in~\cite{OnConvAdamand_Beyond, NonAdaptiveGR} – (arguably) retains state-of-the-art status. Therefore, following an adaptive-like update scheme to solving the problem in Equation~\ref{Eq:Final} promises ease of implementation and scalability to real-world scenarios. 

Though meeting two out of the three criteria above, a simple adaptation of standard optimisation techniques to a compositional problem of the form in Equation~\ref{Eq:Final} is challenging due to the bias incurred from a naive Monte-Carlo sampling of non-linear nested expectations; see~\cite{NestedMC} for a detailed discussion. Of course, such a problem is not unique to this paper and has been previously studied in ~\cite{NestedMC, Mendi_2017}. Current methods, however, are either not-scalable to high-dimensional large-data problems~\cite{davis2007methods, nakatsukasa2018approximate} (e.g., require full gradients -- over all data -- for variance reduction), or solve a relaxed version that presumes a finite sum empirical-risk-minimisation\footnote{Please note that relaxed empirical-risk versions target a different problem all-together~\cite{Mendi_2017}. In other words, solving a finite-sum approximation does not guarantee convergence for the nested expectation problem presented in Equation~\ref{Eq:Final}.} problem~\cite{NestedMC, Mendi_2017}. To meet our requirements, we next present a novel solver that combines extrapolation-smoothing for bias-reduction and gradient-sketching for efficiency and scalability. It is worth noting that our optimiser shares similarities to the work by~\cite{Rasul} but refines analysis to derive bias-handling results (Equation~\ref{Eq:Bias}) and introduces additional constructs (e.g., gradient-sketching mechanisms). Such additions require new proof foundations and re-derivations that we present in the appendix for completeness. 
\paragraph{Algorithmic development} We aim to offer an adaptive algorithm that exhibits similar (theoretical and practical) performance guarantees to ADAM but that is also capable of correctly handling the bias inherent to the problem in Equation~\ref{Eq:ELBO}. To do so, we introduce an update scheme that resembles ADAM but incorporates auxiliary variables that are updated in a subsequent step. Being at an iteration $t$, our algorithm first executes the following updates:
\small
\begin{empheq}[innerbox=\fbox,
left=\text{Primary}\Rightarrow]{align*}
 &\bm{m}_{t}  = \gamma_{t}^{(1)}\bm{m}_{t-1} + \left(1-\gamma_{t}^{(1)}\right)\overline{\nabla \mathcal{L}(\bm{\theta}_{t})} \ , \ \bm{v}_{t} = \gamma_{t}^{(2)}\bm{v}_{t-1} + \left(1 - \gamma_{t}^{(2)}\right)[\overline{\nabla \mathcal{L}(\bm{\theta}_{t})}]^{2} \\ \nonumber
    &\hspace{12em}\rightsquigarrow \bm{\theta}_{t+1} = \bm{\theta}_{t} - \alpha_{t}  \frac{\bm{m}_{t}}{\sqrt{\bm{v}_{t}} + \xi}, \\ \nonumber
    &\hspace{0em}\text{with $\alpha_t$ being a learning rate, $\xi \in \mathbb{R}_{>0}$, $\gamma_{t}^{(1)}$ and $\gamma_{t}^{(2)}$ denoting hyper-parameters.}
\end{empheq}
\normalsize
Assuming the availability of sub-sampled gradients of the loss (i.e., $\overline{\nabla \mathcal{L}(\bm{\theta}_{t})}$), the above set of instructions simply performs an ADAM-like update on the model's free parameters\footnote{It is worth noting that later in our theoretical analysis we provide a rigorous scheme for tuning all hyper-parameters. We also follow such a schedule in our experiments.} $\bm{\theta}$ starting from an initialisation for $\bm{m}_{t}$ and $\bm{v}_{t}$. The problem, however, arises when aiming to acquire unbiased gradients of the objective in Equation~\ref{Eq:Final}~\cite{NestedMC, UIVI, SIVI}. To illustrate this, consider computing the actual gradient of $\mathcal{L}(\bm{\theta})$ at some iteration $t$. This can be written as: 
$\nabla \mathcal{L}(\bm{\theta}_{t}) = \mathbb{E}_{\bm{\hat{\epsilon}}}[\nabla \bm{g}_{\bm{\hat{\epsilon}}}(\bm{\theta}_{t})]^{\mathsf{T}}\mathbb{E}_{\nu}[\nabla f_{\nu}\left(\mathbb{E}_{\bm{\hat{\epsilon}}}[\bm{g}_{\bm{{\hat{\epsilon}}}}(\bm{\theta}_{t})]\right)]$. It is clear that one can easily implement a Monte-Carlo estimator of the first part of the gradient, i.e., $\mathbb{E}_{\bm{\hat{\epsilon}}}[\nabla \bm{g}_{\bm{\hat{\epsilon}}}(\bm{\theta}_{t})]$. The second term, on the other hand, is much harder to estimate due to its nested nature:
\begin{equation}
\label{Eq:SecondPart}
    \mathbb{E}_{\nu}[\nabla f_{\nu}\left(\mathbb{E}_{\bm{\hat{\epsilon}}}[\bm{g}_{\bm{{\epsilon}}_{2}}(\bm{\theta}_{t})]\right)] = \mathbb{E}_{\nu}\left[\frac{1}{\left(\mathbb{E}_{\bm{\hat{\epsilon}}}[\bm{g}_{\bm{\hat{\epsilon}}}(\bm{\theta}_{t})]\right)_{\nu}}\right], \ \ \text{where $(\bm{v})_{\nu}$ denotes the $\nu^{th}$ component of $\bm{v}$.}
\end{equation}
Of course, a simple Monte-Carlo estimator\footnote{We mean by a simple estimator the following: $\mathbb{E}_{\nu}\left[\frac{1}{\left(\mathbb{E}_{\bm{\hat{\epsilon}}}[\bm{g}_{\bm{\hat{\epsilon}}}(\bm{\theta}_{t})]\right)_{\nu}}\right] \approx \frac{1}{N}\sum_{i=1}^{N} \left[\frac{1}{\left(\frac{1}{M}\sum_{j=1}^{M} \bm{g}_{j}(\bm{\theta}_{t})\right)_{i}}\right]$.} of the gradient's second part is biased. Our methodology in tackling this challenge is to find an ``unbiased'' \textit{approximation} with properties allowing us to control such a bias at appropriate rates. To do so, we follow an extrapolation-smoothing scheme~\cite{Rasul, Mendi_2017, Mengdi} originally established in time series~\cite{ExtrSmooth} and differential equations literature~\cite{ExtraSmooth}. These methods introduce a two-step procedure to approximate an unknown quantity, e.g., $\mathbb{E}_{\nu}[\nabla f_{\nu}\left(\mathbb{E}_{\bm{\hat{\epsilon}}}[\bm{g}_{\bm{{\hat{\epsilon}}}}(\bm{\theta}_{t})]\right)]$ in our case. In the first step, a linear extrapolation query vector, $\bm{z}$, is computed while in the second, a smoothed average is evaluated around the extrapolated $\bm{z}$. Precisely, given two model parameter updates $\bm{\theta}_{t}$ and $\bm{\theta}_{t+1}$ we execute the following: 
\small
\begin{empheq}[innerbox=\fbox,
left=\text{Auxiliary}\Rightarrow]{align*}
& \bm{z}_{t+1} = \underbrace{\left(1 -\sfrac{1}{\beta}_{t}\right)\bm{\theta}_{t} + \sfrac{1}{\beta_{t}}\bm{\theta}_{t+1}}_{\text{Extrapolation}} \ \ \text{and} \ \   \bm{y}_{t+1} = \underbrace{(1-\beta_{t})\bm{y}_{t} + \beta_{t}\overline{\bm{g}_{t}(\bm{z}_{t+1})}}_{\text{Smoothing}}, \\
& \text{with $\beta_{t}$ being a free parameter, and $\overline{\bm{g}_{t}(\bm{z}_{t+1})}$ is a sampled estimator of $\mathbb{E}_{\bm{\hat{\epsilon}}}[{\bm{g}}_{\bm{\hat{\epsilon}}}\left(\bm{z}_{t+1}\right)]$}.
\end{empheq}
\normalsize
Simply, the smoothing step is attempting to track $\mathbb{E}_{\bm{\hat{\epsilon}}}[{\bm{g}}_{\bm{\hat{\epsilon}}}\left(\bm{z}_{t+1}\right)]$ which can then be substituted in Equation~\ref{Eq:SecondPart} to approximate the second term of the gradient. Interestingly, we evaluate this smoothing step around a linearly extrapolated vector $\bm{z}_{t+1}$ and not only on the updated model parameters $\bm{\theta}_{t+1}$. Though an evaluation around $\bm{\theta}_{t+1}$ can guarantee convergence~\cite{Mendi_2017}, we demonstrate that following the above extrapolation scheme leads to faster convergence rates by enabling a better control of the bias. Informally, we can, under further technical consideration, demonstrate that the difference between true and sub-sampled gradients abides by: 
\begin{equation}
\label{Eq:Bias}
    \mathbb{E}_{\text{total}}[||\nabla \mathcal{L}(\bm{\theta}_{t}) - \overline{\nabla \bm{g}_{t}(\bm{\theta}_{t})}^{\mathsf{T}}\overline{\nabla f_{t}(\bm{y}_{t})} ||_{2}^{2}] \approx \mathcal{O}(t^{-\sfrac{4}{5}}),
\end{equation}
where $\overline{\nabla \bm{g}_{t}(\bm{\theta}_{t})}$ and $\overline{\nabla f_{t}(\bm{y}_{t})}$ denote estimate gradients of the functions $\bm{g}_{\bm{\hat{\epsilon}}}(\cdot)$ and $f_{\nu}(\cdot)$, and $\mathbb{E}_{\text{total}}[\cdot]$ the expectation under all incurred randomness in the algorithm. Importantly, the result in Equation~\ref{Eq:Bias} shows that the bias in our gradient estimator vanishes with an increased number of iterations. This, in turn, allows us to prove convergence of the resulting algorithm as studied in Section~\ref{Sec:TheoGurantees}. We now introduce the complete algorithm combining both of the above primary and auxiliary steps. To that end, we assume a schedule of learning rates $\eta\text{-schedule} = \{\langle \alpha_{t}, \beta_{t},\gamma^{(1)}_{t}, \gamma_{t}^{(2)}\rangle\}_{t=1}^{T}$, and a time-varying set of batch-sizes $K\text{-size} = \{\langle K_{t}^{(1)}, K_{t}^{(2)}, K_{t}^{(3)} \rangle\}$ needed to mini-batch $\nabla f_{\nu}(\cdot)$, $\nabla \bm{g}_{\bm{\hat{\epsilon}}}(\cdot)$, and $\mathbb{E}_{\bm{\hat{\epsilon}}}[\bm{g}_{\bm{\hat{\epsilon}}}(\cdot)]$ respectively. Our computations are achieved through two oracles that can return gradients and function values when needed\footnote{Please note that in Section~\ref{Sec:TheoGurantees} we provide explicit schedules for each of the hyper-parameters introduced.}: 
\begin{align*}
    \text{Oracle}_{f}\left(\bm{y}_{t}, K_{t}^{(1)}\right) &= \{\langle \nu_{t_{i}}, \nabla f_{\nu_{t_{i}}}(\bm{y}_{t})\rangle\}_{i=1}^{K_{t}^{(1)}}, \ \ \text{with $\{\nu_{t_{i}}\}_{i=1}^{K_{t}^{(1)}}$ being i.i.d.} \\
    \text{Oracle}_{g}\left(\bm{z}_{t}, K_{t}^{(2)}\right) &= \{\langle \bm{\hat\epsilon}_{t_{i}}, \bm{g}_{\bm{\hat{\epsilon}}_{ t_{i}}}(\bm{z}_{t}), \nabla \bm{g}_{\bm{\hat{\epsilon}}_{ t_{i}}}(\bm{z}_{t})\rangle\}_{i=1}^{K_{t}^{(2)}}, \ \ \text{with $\{\bm{\hat{\epsilon}}_{t_{i}}\}_{i=1}^{K_{t}^{(2)}}$ also i.i.d.}
\end{align*}

In addition, we define $\delta \in (0,1)$ used to measure solution accuracy and a small positive constant $\xi$ for numerical stability. The overall procedure is practical requiring only one implementation loop and is summarised in Algorithm~\ref{Algo:ADAM}. It operates in three main steps. In the first, sub-sampled gradients are computed by calling $\text{Oracle}_{f}(\cdot)$ and $\text{Oracle}_{g}(\cdot)$. When estimated, the second step computes primary updates leading to improved model parameters $\bm{\theta}_{t+1}$. Given $\bm{\theta}_{t}$ and $\bm{\theta}_{t+1}$, the third step executes extrapolation-smoothing to update an estimate of $\mathbb{E}_{\bm{\hat{\epsilon}}}[\bm{g}_{\bm{\hat{\epsilon}}}(\cdot)]$ guaranteeing vanishing bias with increased iterations. It is to be noted that the smoothing step requires an additional call to $\text{Oracle}_{g}(\cdot)$ to sample $\overline{\bm{g}_{t}(\bm{z}_{t+1})}$. The updated smoothed variable $\bm{y}_{t+1}$ is then used in subsequent iterations where the overall process repeats. 
\begin{algorithm}[h!]
\caption{CI-VI: Compositional Implicit Variational Inference}
\label{Algo:ADAM}
\begin{algorithmic}

\STATE \textbf{Inputs:} Initial variable $\bm{\theta}_{1}$, $\delta \in (0, 1)$, $\xi$, $T= \mathcal{O}(\delta^{-\sfrac{5}{4}})$, $\eta\text{-schedule}$ and $K\text{-size}$

\STATE \textbf{Initialisation:} Initialise $\bm{z}_{1} = \bm{\theta}_{1}$, $\bm{y}_{1} = \bm{0}$, and $\bm{m}_{0} = \bm{v}_{0} = 0$

\STATE \textbf{for} $t=1$ to $T$ \textbf{do:}

\STATE \hspace{1em} \color{orange} Compute sub-sampled gradients by calling oracles: \hspace{5em} 	$\triangleright \ \text{Gradients}$ 

\begin{equation*}
\begin{alignedat}{2}
\overline{\nabla f_{t}(\bm{y}_{t})}  & = \sfrac{1}{K_{t}^{(1)}} \sum_{i=1}^{K_{t}^{(1)}} \nabla f_{\nu_{t_{i}}}(\bm{y}_{t})
&\quad&\hspace{5em}\raisebox{-2\normalbaselineskip}[1pt][1pt]{
      \text{$\implies \overline{\nabla \mathcal{L}(\bm{\theta}_{t})} = \overline{\nabla \bm{g}_{t}(\bm{\theta}_{t})}^{\mathsf{T}}
      \overline{\nabla f_{t}(\bm{y}_{t})}$}
}
\\ \overline{\nabla \bm{g}_{t}(\bm{\theta}_{t})}  & = \sfrac{1}{K_{t}^{(2)}} \sum_{j=1}^{K_{t}^{(2)}} \nabla \bm{g}_{\bm{\hat\epsilon}_{t_{j}}}(\bm{\theta}_{t})
\end{alignedat}
\end{equation*}

\STATE \hspace{1em} 	\color{red} Perform the following primary updates: \hspace{10em}  $\triangleright \ \text{Primary Update}$  \begin{equation*}
\begin{alignedat}{2}
\bm{m}_{t}  & = \gamma_{t}^{(1)}\bm{m}_{t-1} + \left(1-\gamma_{t}^{(1)}\right)\color{orange}\overline{\nabla \mathcal{L}(\bm{\theta}_{t})}
&\quad&\raisebox{-.6\normalbaselineskip}[1pt][1pt]{
     \text{$\implies \bm{\theta}_{t+1} = \bm{\theta}_{t} - \alpha_{t} \frac{\bm{m}_{t}}{\sqrt{\bm{v}_{t}}+\xi}$}
} \\
\bm{v}_{t} &= \gamma_{t}^{(2)}\bm{v}_{t-1} + \left(1 - \gamma_{t}^{(2)}\right)\color{orange}[\overline{\nabla \mathcal{L}(\bm{\theta}_{t})}]^{2}
\end{alignedat}
\end{equation*}
\STATE \color{blue} \hspace{1em} Perform the following auxiliary updates: \hspace{10em} 	$\triangleright \ \text{Auxiliary Update}$  \begin{equation*}
\bm{z}_{t+1} = \underbrace{\left(1 -\sfrac{1}{\beta}_{t}\right)\bm{\theta}_{t} + \sfrac{1}{\beta_{t}}\bm{\theta}_{t+1}}_{\text{Extrapolation}} \ \ \text{and} \ \   \bm{y}_{t+1} = \underbrace{(1-\beta_{t})\bm{y}_{t} + \beta_{t}\overline{\bm{g}_{t}(\bm{z}_{t+1})}}_{\text{Smoothing}}
\end{equation*}
\STATE \color{black} \textbf{Output:} Return solution as a uniform sample from $\left\{\bm{\theta}_{t}\right\}_{t=1}^{T}$
\end{algorithmic}
\end{algorithm}
\paragraph{On practicability} Algorithm~\ref{Algo:ADAM}, though successful, assumes an idealised setting in which computing products of gradient estimates is feasible. In most reasonably-sized problems, however, such products are prohibitively expensive due to the problem's dimensionality and number of samples available (e.g., order of thousands). To remedy this problem, we next present a matrix-sketching mechanism (not introduced previously in both variational inference and compositional optimisation literature) from randomised linear algebra~\cite{RAlgebra} enabling scalability. Here, we simply replace gradient computations (in orange) in Algorithm~\ref{Algo:ADAM} with the set of instructions\footnote{In Algorithm~\ref{Algo:ADAM2}, we use $\bm{A}(:, k)$ to denote the $k^{th}$ column of some matrix $\bm{A}$, and $\boldsymbol{b}(k)$ the $k^{th}$ component of a vector $\bm{b}$.} in Algorithm~\ref{Algo:ADAM2} that eases implementation in that it allows for randomised matrix-vector products which come-in handy in big-data problems. Rather than needing to sum overall entries in the corresponding product, we can now simply enable a batch-like version through sketching. Chiefly, as we show in Section~\ref{Sec:Exp}, we can further specialise Algorithm~\ref{Algo:ADAM2} by accounting for the sparsity pattern of our gradients for improved adjustability to the SIVI setting. 
\begin{algorithm}[h!]
\caption{Gradient-Sketching for Large-Scale SIVI}
\label{Algo:ADAM2}
\begin{algorithmic}
\STATE \textbf{Inputs:} $\text{Oracle}_{f}(\bm{y}_{t}, K_{t}^{(1)})$, $\text{Oracle}_{g}(\bm{\theta}_{t}, K_{t}^{(2)})$, \#-samples $d_{t} \in [1,n]$
\STATE Sample subset $S_t\subseteq [1,\ldots, n]$ of size $d_t$, where $\mathbb{P}\text{r}(k\in S_t) = \sfrac{1}{n}$ for all $k=1,\ldots,n$.
\STATE \textbf{Return:} $\overline{\nabla \mathcal{L}(\bm{\theta}_{t})} = \sfrac{1}{K^{(1)}_t}\sum_{a=1}^{K^{(1)}_t}[\sfrac{q}{d_t}\sum_{k\in S_t}\nabla \boldsymbol{g}_{\bm{\hat\epsilon}_{ t_a}}^{\mathsf{T}}(\bm{\theta}_t)(:, k)\nabla f_{\nu_{t_a}}(\boldsymbol{y}_t)(k)]$ 
\end{algorithmic}
\end{algorithm}

Of course, sketching mechanisms can affect speeds of convergence due to additionally induced randomness. To gain insight into such phenomena, we next provide rigorous theoretical guarantees demonstrating convergence of the resulting algorithm (i.e., Algorithm~\ref{Algo:ADAM} with Sketching for gradient computation) and quantifying oracle complexities.

\paragraph{Theoretical guarantees} \label{Sec:TheoGurantees}
We demonstrate that Algorithm~\ref{Algo:ADAM} converges to a stationary point of the \emph{non-convex objective}\footnote{Non-convexity is abundant in our problem due to the usage of neural networks. In these scenarios, one aims at a stationary point as even assessing a local-minimum is NP-Hard~\cite{NonConvex}.} in Equation~\ref{Eq:Final}. As is standard in optimisation literature, we provide our results in terms of the number of oracle calls needed to convergence\footnote{Please note that analysing other metrics, e.g., generalisation bounds is an interesting avenue for future research.}. Due to space constraints, we defer the proof to the appendix. Here, we provide the statement of the main theorem. Our main results are based on the following common assumptions: 
\begin{assumption}\label{assum_1}
We make the following assumptions\footnote{Please note that for additional clarity, our assumptions are further elaborated in the appendix.}: 1) $\left|f_{\nu}(\bm{y})\right| \leq B_{f}$, and $\left|\left|\nabla f_{\nu}(\bm{y})\right|\right| \leq M_f$, for all $\bm{y}$ and $\nu$; 2) $f_{\nu}(\cdot)$ is $L_{f}$-smooth, and $\bm{g}_{\bm{\hat{\epsilon}}}(\bm{x})$ is $M_{g}$- Lipschitz continuous, and $L_{g}$-smooth; 3) oracle sample-pairs are independent; and 4) oracles return unbiased gradient estimates with bounded variances.  
\end{assumption}

Now, we present the main theorem analysing convergence and oracle complexities of Algorithm~\ref{Algo:ADAM}: 
\begin{theorem}[Convergence \& Oracle Complexities]
Consider a parameter setup given by: $\alpha_{t} = \sfrac{C_{\alpha}}{t^{\frac{1}{5}}}, \beta_{t} = C_{\beta}, K_{t}^{(1)} = C_{1}t^{\frac{4}{5}}, K_{t}^{(2)} = C_{2}t^{\frac{4}{5}},  K_{t}^{(3)}  = C_{3} t^{\frac{4}{5}}, \gamma_{t}^{(1)}= C_{\gamma}\mu^{t}, \ \gamma_{2}^{(t)}  = 1 - \sfrac{C_{\alpha}}{t^{\frac{2}{5}}}(1 - C_{\gamma}\mu^{t})^{2}$, for some positive constants $C_{\alpha}, C_{\beta}, C_{1}, C_{2}, C_{3}, C_{\gamma}, \mu$ such that $C_{\beta} < 1$ and $\mu \in (0,1)$. For any $\delta \in (0,1)$, Algorithm~\ref{Algo:ADAM} running gradient-sketching (i.e., using Algorithm \ref{Algo:ADAM2} to compute gradient products) with a sample-size $d_{t} = \mathcal{O}(1)$ outputs, in expectation, a $\delta$-approximate first-order stationary point $\tilde{\bm{\bm{\theta}}}$ of $\mathcal{L}(\bm{\bm{\theta}})$. That is: $\mathbb{E}_{\textrm{total}}[||\nabla\mathcal{L}(\tilde{\bm{\theta}})||_{2}^{2}] \leq \delta$, with ``total'' representing \emph{all} incurred randomness. Moreover, Algorithm~\ref{Algo:ADAM} acquires $\tilde{\bm{\theta}}$ with an overall oracle complexity of the order $\mathcal{O}\left(\delta^{-\sfrac{9}{4}}\right)$. 
\end{theorem}

\vspace{-1em}
\section{Experiments and results}\label{Sec:Exp}
In this section, we present an empirical study demonstrating the effectiveness of CI-VI that we implement in PyTorch~\cite{pytorch}. We benchmark on three broad tasks covering toy examples, Bayesian logistic regression and variational autoencoders. Of course, our derivations need to be specialised to each of these tasks. We provide such constructs in the appendix due to space constraints. We compare against standard semi-implicit variational inference algorithms (e.g., SIVI and UIVI) in addition to methods from nested Monte-Carlo~\cite{NestedMC} and compositional optimisation~\cite{Mendi_2017, Mengdi}. To improve stability and computational efficiency, our implementation of CI-VI tracks log-gradients instead of $\overline{\nabla\mathcal{L}(\bm{\theta}_{t}})$. Furthermore, due to the special structures of $f_{\nu}(\cdot)$ (e.g., logarithmic function) and $g_{\hat{\bm{\epsilon}}}$, we can further improve gradient sketching by only (uniformly) sampling non-zero elements from $\nabla g_{\bm{\hat{\epsilon}}} (\cdot)$ rather than the whole $d_{t}$-set $S_{t}$ in Algorithm~\ref{Algo:ADAM2}. Due to space constraints, such adaptations in addition to exact experimental settings can be found in the appendix. We ran all experiments on a single NVIDIA GeForce RTX 2080 GPU. Crucially, CI-VI is highly efficient consuming 30 seconds for toy experiments and, at most, 1.5 hours for text modelling when using variational autoencoders.

\underline{\textbf{Toy Experiments:}} In this set of experiments, we apply our method to minimise a KL-divergence between semi-implicit and ground truth distributions on two-modal, star, and banana as taken from~\cite{UIVI}.
\begin{wrapfigure}{r}{0.6\textwidth}
  \begin{center}
    \includegraphics[trim = {1em, 23em, 1em, 21em}, clip=true , width=0.61\textwidth]{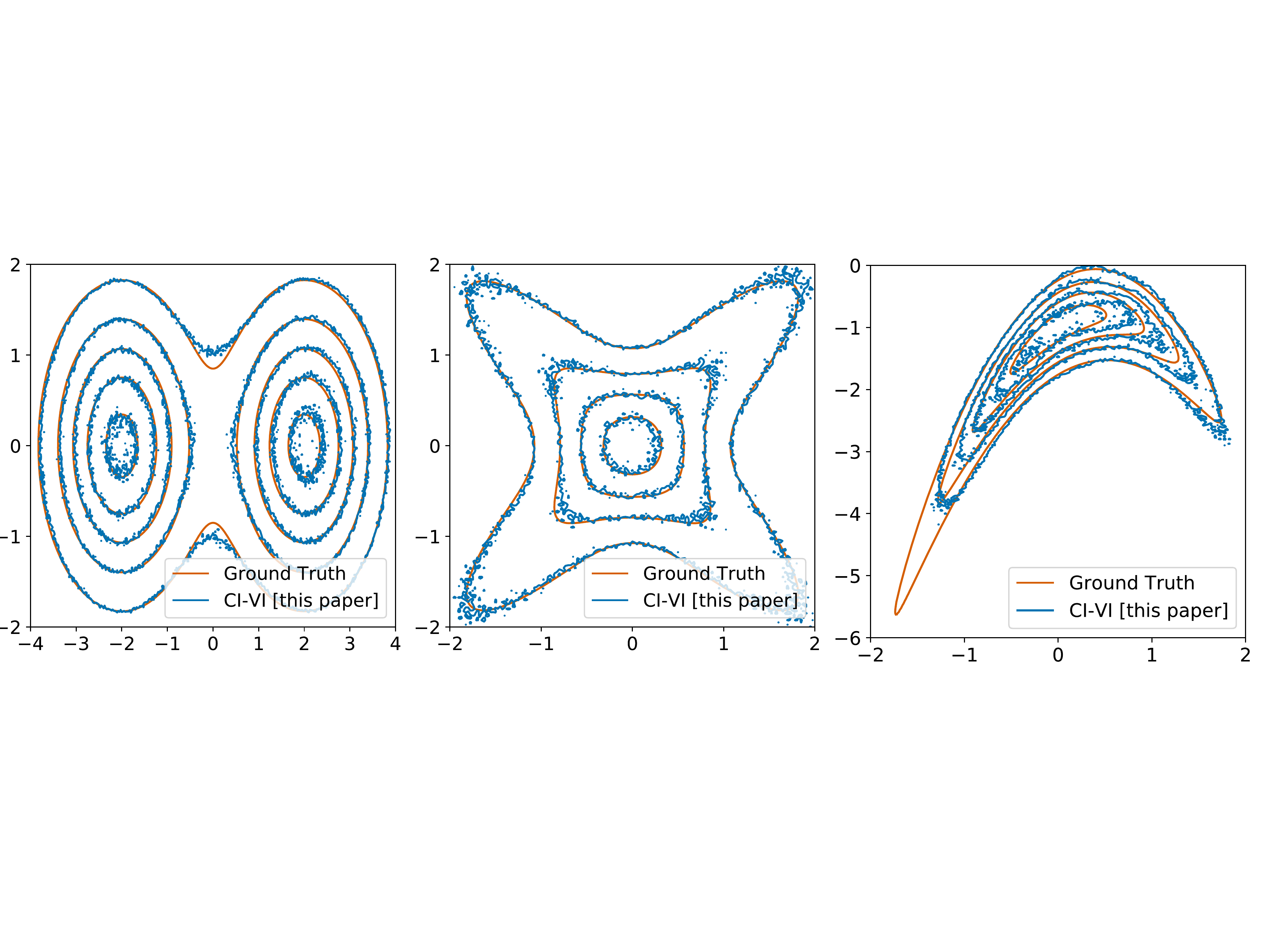}
  \end{center}
  \vspace{-5pt}
  \caption{Results demonstrating that CI-VI can approximate sophisticated distributions from~\cite{UIVI}. Left: Two-Modal distribution, Middle: Star distribution, and Right: Banana distribution.}
\end{wrapfigure}
For all semi-implicit distributions, we chose $q(\bm{\epsilon})$ to be a multi-variate zero-mean identity-covariance-matrix Gaussian, i.e., $q(\bm{\epsilon}) = \mathcal{N}(\bm{0}, \bm{I}_{3 \times 3})$. The conditional distribution $q_{\bm{\theta}}(\bm{z}|\bm{\epsilon})$, is also assumed Gaussian with a neural network (two layers 50 by 50 hidden units) parameterised mean and a parameterised diagonal covariance matrix, where $q_{\bm{\theta}}(\bm{z}|\epsilon) = \mathcal{N}(\bm{\mu}_{\bm{\theta}_{1}}(\bm{\epsilon}), \text{diag}(\bm{\theta}_{2}))$ with $\bm{\theta}_{1}$ denoting neural network parameters, and $\bm{\theta}_{2}$ another set of free parameters. Figures 1 compares the contour plots of the optimised variational distribution with ground-truth. We can clearly see that CI-VI can accurately capture sophisticated patterns like skewness, kurtosis and multi-modality.

\underline{\textbf{Bayesian Logistic Regression:}} With our method performing well on toy examples, we ran CI-VI in Bayesian logistic regression that aims to acquire posteriors on classification tasks. Given a data-set $\mathcal{D}=\{\bm{x}_i, y_i\}_{i=1}^N$ where $\bm{x}_i=( \bm{x}_{i1}, \dots, \bm{x}_{iD})^\mathsf{T}$ is a $D$-dimensional feature vector and $y_i \in \{0, 1\}$ a binary label, Bayesian logistic regression considers a probabilistic model $p(\mathcal{D}, z)=p(\bm{z})\prod_{i=1}^N p(y_i|\bm{x}_i, \bm{z})=p(\bm{z})\prod_{i=1}^N \text{Bernoulli}\left[(1+\exp(-\bm{x}_i^\mathsf{T}\bm{z}))^{-1}\right]$ and infers the posterior of $\bm{z}$ given the observed $\mathcal{D}$. We experimented with the Waveform, Spam, and Nodal data-sets from~\cite{SIVI}.
\begin{figure}
    \centering
    \includegraphics[trim = {1em, 1em, 1em, 1em}, clip=true, width=0.95\textwidth]{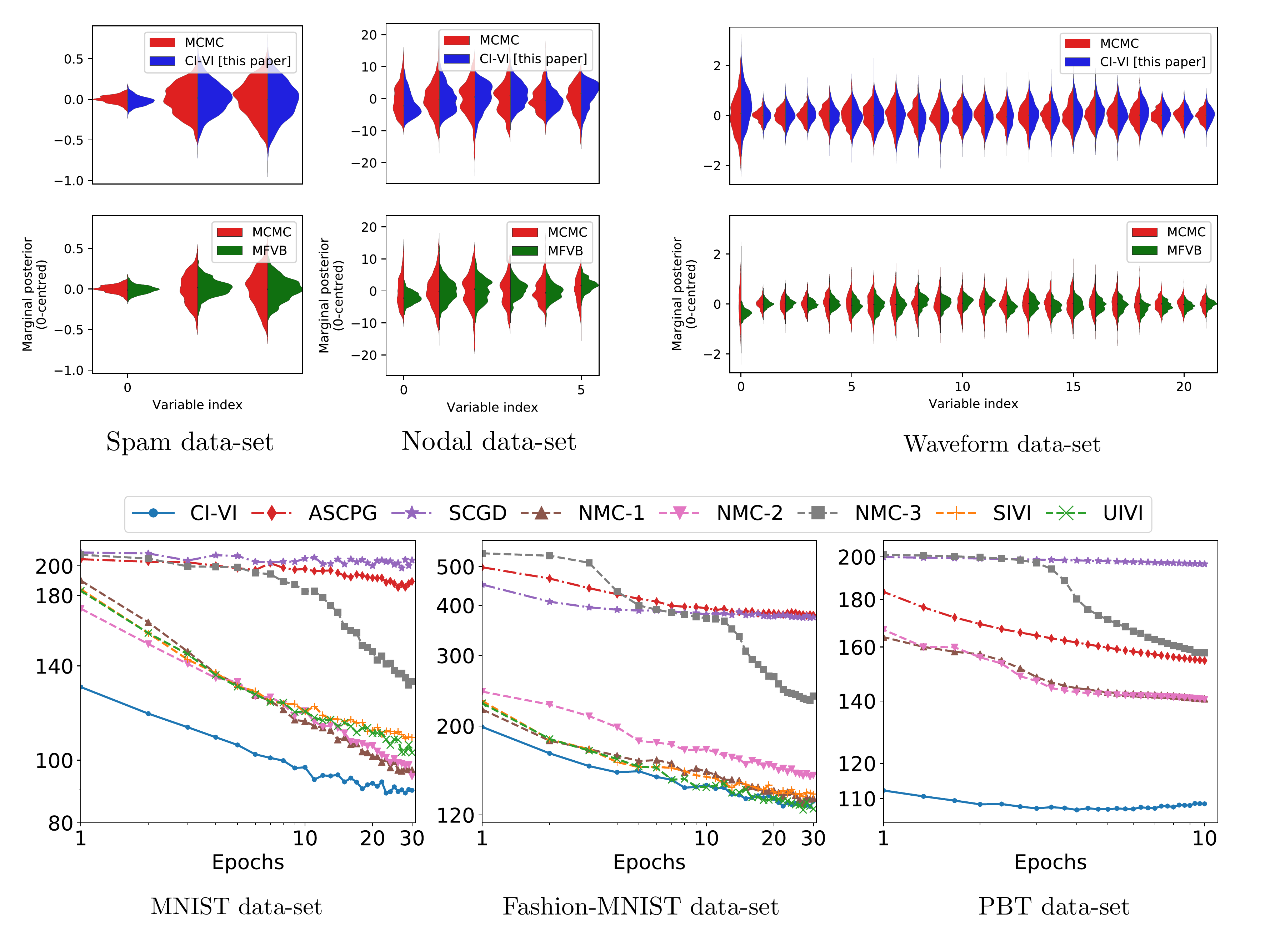}
    \caption{Results depicting performance of CI-VI on both Bayesian logistic regression (top 2 rows) and variational autoenconders (bottom row). We realise that CI-VI is closer in its estimation to MCMC than is mean-field variational Bayes (MFVB). On the variational autoencoders side, we realise the CI-VI outperforms others on MNIST, Fashion-MNIST, and PBT. We use NMC-1, NMC-2, NMC-3 to denote nested-Monte-Carlo algorithms that use ADAM~\cite{kingma2014adam}, RMS-Prop~\cite{OnConvAdamand_Beyond}, and SGD respectively.}
    \vspace{-0.5em}
    \label{fig:my_label}
\end{figure}
We fixed the prior $p(\bm{z})$ to be a zero-mean Gaussian given by: $p(\bm{z}) = \mathcal{N}(\bm{0}, 100\times\bm{I}_{D\times D})$. For the semi-implicit setting, $\bm{\epsilon}$ followed a standard Gaussian whose dimension varied across data-sets. $q_{\bm{\theta}}(\bm{z}|\bm{\epsilon})$ was again a Gaussian with parameterised mean (two layer neural network with 200 units each) but with a full covariance matrix: $q_{\bm{\theta}}(\bm{z}|\bm{\epsilon}) = \mathcal{N}(\bm{\mu}_{\bm{\theta}_{1}}(\bm{\epsilon}), \bm{L}_{\bm{\theta}_{2}}\bm{L}_{\bm{\theta}_{2}}^{\mathsf{T}})$. Due to space constraints, the full set of results can be found in the appendix. In the first two rows in Figures~1, we demonstrate violin plots on all three data-sets. Clearly CI-VI captures the variance better than Mean-Field Variational Bias (MFVB in the figure) when compared to MCMC distributions \footnote{Please note we also show a marginalised pair-wise posterior plot in the appendix.} across all latent variables. 

\underline{\textbf{Semi-implicit Variational Autoencoders:}} In our final evaluation, we extensively experimented with variational autoencoders \cite{VAE} but ones that exhibited semi-implicit variational distributions. In fact, it has been shown that upon the usage of semi-implicit variational distributions, the gap between the ELBO and marginal data likelihood can further be reduced~\cite{UIVI, SIVI}. We experimented with three data-sets, two of which are standard (MNIST and Fashion-MNIST), while the third considered a text modelling task with the Penn-Tree-Bank (PTB) as presented in~\cite{fuflexible}. All structural details in each of these scenarios can be found in the appendix. 
Our results depicted in Figures~2 (bottom-row) compare CI-VI with semi-implicit solvers~\cite{UIVI, SIVI}, nested Monte-Carlo algorithms, and compositional optimisers. Again, it is clear that CI-VI outperforms others in terms of the number of epochs needed for convergence. Interestingly, such a gap is further signified on the PBT data-set\footnote{Please note that we do not show the results of SIVI and UIVI in Figure~1 (h) as it was hard to get these to correctly operate on NLP tasks due to their source implementations. Staying fair to these methods, we opted-out of demonstrating their performance.}.

\vspace{-0.3em}
\section{Conclusions and Future Work}
We proposed CI-VI, a compositional solver for scalable and efficient semi-implicit variational inference. Our method rewrites SIVI as an instance of a compositional optimisation and devises a solver that correctly handles nested bias through an extrapolation-smoothing and a gradient sketching mechanism. We tested our method on a variety of tasks, including text modelling from natural language processing. In all these instances, we showed CI-VI's effectiveness. In papers to follow, we plan to further scale our method to dialogue problems from NLP and to extend our analysis to time-series models. We also think our nested-expectation theoretical results can be broadly applied beyond this paper to cover topics from experimental design. We will also tackle this direction in the future. 


\bibliographystyle{plain}
\bibliography{main}

\appendix
\section{Practical implementation}
To improve numerical stability and computational efficiency, our implementation of CI-VI has been adapted to properties of functions $f_{\nu}(\cdot)$ and $g_{\bm{\hat\epsilon}}(\cdot)$ defined in Section 2.2. These adaptions are clarified below:

\paragraph{Log trick:} The output from $g_{\bm{\hat\epsilon}}(\cdot)$ is a vector of density ratios whose value can be extreme, especially when the $\bm{z}$ and $\bm{x}$ are high-dimensional rvs. In order to do inference in such probabilistic model, we propose a CI-VI implementation which conducts most of the computations in log-scale. The sub-sampled gradient $\overline{\nabla \mathcal{L}(\bm{\theta}_{t})}$ in Algorithm 1 can be reformulated as: 
\begin{align*}
    \overline{\nabla \mathcal{L}(\bm{\theta}_{t})} = \overline{\nabla \bm{g}_t(\bm{\theta}_{t})}^{\mathsf{T}}\overline{\nabla f_{t}(\bm{y}_{t})} 
    &= [\nabla\log\overline{\bm{g}_t(\bm{\theta}_t)}^{\mathsf{T}}]_{p \times n} \cdot
    \underbrace{\exp 
    \left[\log \overline{\bm{g}_t(\bm{\theta}_t)} + \log \overline{\nabla f_{t}(\bm{y}_{t})}\right]_{n \times 1}}_{\bm{k}_t} \\
    \text{where} \quad \overline{f_{t}(\bm{y}_{t})} &= \sfrac{1}{K_{t}^{(1)}} \sum_{i=1}^{K_{t}^{(1)}}  [\log \bm{y}_{t} ]^{\mathsf{T}}\bm{e}_{\nu_{t_{i}}}
\end{align*}
The $j$-th element of $\bm{k}_t$ is then given by Equation \ref{Eq:Logscale}. Since $\bm{y}_{t}$ is a smoothed average of $\overline{\bm{g}_t(\bm{\theta}_t)}$, their logarithms will cancel each other before the exponentiation is taken. 
\begin{align}
\label{Eq:Logscale}
    (\bm{k}_t)_j=
     \begin{cases}
        0, \quad \textit{if} \ \ \sum_{i=1}^{K_{t}^{(1)}} \mathcal{I}(\nu_{t_{i}}=j) = 0\\
        \exp(\log \overline{\bm{g}_t(\bm{\theta}_t)}_j -\log{K_{t}^{(1)}} - {\log(\bm{y}_{t}})_j + \log \sum_{i=1}^{K_{t}^{(1)}} \mathbb{I}(\nu_{t_{i}}=j))  , \quad \textit{otherwise}\\
     \end{cases}
\end{align}

Finalising this log-scale implementation, we track the value of $\log \bm{y}_{t}$ instead of $\bm{y}_{t}$ for the auxiliary update as follow:
\begin{align*}
    \bm{a} \coloneqq \log(1-\beta_{t}) + \log\bm{y}_{t}, \quad
    \bm{b} \coloneqq \log \beta_{t} + \log \overline{\bm{g}_{t}(\bm{z}_{t+1})}, \quad
    \log\bm{y}_{t+1} = \log [
    \exp(\bm{a}) + \exp(\bm{b}) ]
\end{align*}
The update of $\log\bm{y}_{t}$ may encounter numerical overflow, which can be handled by Taylor series approximation. For the $j$-th overflowed dimension, $\bm{r}_j \coloneqq \exp[\max (\bm{a}_j, \bm{b}_j) - \min (\bm{a}_j, \bm{b}_j)], \ \ \log(\bm{y}_{t+1})_j = \max (\bm{a}_j, \bm{b}_j) + \sum_{i=1}^N \frac{(-1)^{i-1} \bm{r}_{j}^{i}}{i} + \mathcal{O}(\bm{r}_j^{N+1})$. 

For large-scale CI-VI, updating all dimensions of $\log\bm{y}_{t}$ is challenging due to computational and memory constraints. Rather than performing a full update for $\log\bm{y}_{t}$, we follow a batch-like mechanism that proved effective in our experiments. Namely, we split the index set $[1, n]$ into smaller chunks and sample $\nu_t$ from one of these chunks to execute the updates. This way, only those dimensions indexed by the current chunk need to be updated rather than the whole high-dimensional vector $\log\bm{y}_{t}$. Of course, such chunks have to vary across iterations so as to guarantee the update of $\log\bm{y}_{t}$. To do so, we switch to a new chunk occasionally and re-initialize the dimensions of $\log\bm{y}_{t}$ indexed by the new chunk using the value of $\log \overline{\bm{g}_{t}(\bm{z}_{t})}$.

\paragraph{Sparse gradient-sketching: }Another insight from Equation \ref{Eq:Logscale} is that the $\bm{k}_t$ is a sparse vector with non-zero elements indexed by sampled $\nu_{t}$. The gradient-sketching in Algorithm 2 can be adapted by removing the zero elements in $\bm{k}_t$ and corresponding columns in $[\nabla\log\overline{\bm{g}_t(\bm{\theta}_t)}^{\mathsf{T}}]_{p \times n}$ first. The set $S_t$ is then sampled uniformly from the left indices.

\section{Experimental settings and results}
\paragraph{Toy Experiments:} In toy experiments, we minimize the KL-divergence between semi-implict variational distributions $q_{\bm{\theta}}(\bm{z})$ and ground-truth distributions $p(\bm{z})$. The compositional objective can be written as:
\begin{align*}
    \mathcal{KL} [q_{\bm{\theta}}(\bm{z})||p(\bm{z})] = \mathbb{E}_{\bm{u} \sim q(\bm{u}), \bm{\epsilon} \sim q(\bm{\epsilon})}\left[\log \mathbb{E}_{\bm{\hat \epsilon} \sim q(\bm{\epsilon})} \left[ \frac{q_{\bm{\theta}}(h_{\bm{\theta}}(\bm{u}; \bm{\epsilon})|\bm{\hat{\epsilon}})}{p(h_{\bm{\theta}}(\bm{u}; \bm{\epsilon}))} \right] \right].
\end{align*}

\paragraph{Bayesian Logistic Regression:} Given a data-set $\mathcal{D}=\{\bm{x}_i, y_i\}_{i=1}^N$, we can write the compositional form of negative ELBO as:
\begin{align*}
    -ELBO &= \mathbb{E}_{q_{\bm{\theta}}(\bm{z})} \left[\log\frac{q_{\bm{\theta}}(\bm{z})}{ p (\mathcal{D} |\bm{z})p (\bm{z})} \right] \\
    & = \mathbb{E}_{\bm{u} \sim q(\bm{u}), \bm{\epsilon} \sim q(\bm{\epsilon})} \left[ \log \mathbb{E}_{\bm{\hat\epsilon} \sim q(\bm{\epsilon})} \left[ \frac{q_{\bm{\theta}}(h_{\bm{\theta}}(\bm{u}; \bm{\epsilon})|\bm{\hat\epsilon})}{p(\mathcal{D} | h_{\bm{\theta}}(\bm{u}; \bm{\epsilon})) \cdot p(h_{\bm{\theta}}(\bm{u}; \bm{\epsilon}))} \right] \right],
\end{align*}
where $p(\mathcal{D}|\bm{z})=\prod_{i=1}^{N} p(y_i|\bm{x}_i, \bm{z})=\prod_{i=1}^N \text{Bernoulli}\left[(1+\exp(-\bm{x}_i^\mathsf{T}\bm{z}))^{-1}\right]$. 

\paragraph{Semi-implicit Variational Autoencoder:} Given a dataset $\mathcal{D}=\{{\bm{x}_i}\}_{i=1}^N$, the negative ELBO of a single datapoint $\bm{x}_i$ is given by:
\begin{align*}
  -ELBO(\bm{x}_i) = 
    \mathbb{E}_{\bm{u}_i \sim q(\bm{u}), \bm{\epsilon}_i \sim q(\bm{\epsilon})} \left[\log \mathbb{E}_{\bm{\hat\epsilon}_i \sim p({\bm{\hat\epsilon}_i})} \left[ \frac{q_{\bm{\theta}}(h_{\bm{\theta}}(\bm{u}_i; \bm{\epsilon}_i)|\bm{\hat\epsilon}_i, \bm{x}_i)}{p_{\bm{\hat\theta}}(\bm{x}_i | h_{\bm{\theta}}(\bm{u}_i; \bm{\epsilon}_i)) \cdot p(h_{\bm{\theta}}(\bm{u}_i; \bm{\epsilon}_i))}
    \right] \right],
\end{align*}
in which the encoder parameter $\bm{\theta}$ and decoder parameter $\bm{\hat\theta}$ are optimized jointly. We can further construct an estimator of the negative ELBO of the full data-set:
\small
\begin{align*}
    -ELBO(\mathcal{D}) = N \cdot \mathbb{E}_{\bm{x}_i \sim \text{Uniform}(\mathcal{D}),\bm{u}_i \sim q(\bm{u}), \bm{\epsilon}_i \sim q(\bm{\epsilon})} \left[\log \mathbb{E}_{\bm{\hat\epsilon}_i \sim p({\bm{\hat\epsilon}_i})} \left[ \frac{q_{\bm{\theta}}(h_{\bm{\theta}}(\bm{u}_i; \bm{\epsilon}_i)|\bm{\hat\epsilon}_i, \bm{x}_i)}{p_{\bm{\hat\theta}}(\bm{x}_i | h_{\bm{\theta}}(\bm{u}_i; \bm{\epsilon}_i)) \cdot p(h_{\bm{\theta}}(\bm{u}_i; \bm{\epsilon}_i))}
    \right] \right],
\end{align*}
\normalsize

\clearpage
\textbf{All the experimental settings are listed below:}
\begin{table}[htb!]
\centering
\begin{tabular}{|c|c|c|} 
    \hline
    Two-Modal& Star& Banana\\
    \hline
    $\begin{aligned}
        &0.5\mathcal{N}\left(\begin{bmatrix} -2\\0 \end{bmatrix}, I_{2\times2}\right) +\\
        &0.5\mathcal{N}\left(\begin{bmatrix} 2\\0 \end{bmatrix}, I_{2\times2}\right)
    \end{aligned}$&
    $\begin{aligned}
        &0.5\mathcal{N}\left(0, \begin{bmatrix} 2 &1.8\\1.8 &2 \end{bmatrix}\right) +\\ &0.5\mathcal{N}\left(0, \begin{bmatrix} 2 &-1.8\\-1.8 &2 \end{bmatrix}\right)
    \end{aligned}$&
    $\begin{aligned}
        &\bm{z}=\begin{bmatrix}\bm{z}_1\\\bm{z}_2-\bm{z}_1^2-1\end{bmatrix}\\ \begin{bmatrix}\bm{z}_1\\\bm{z}_2\end{bmatrix} &\sim \mathcal{N}\left(0, \begin{bmatrix} 1 &0.9\\0.9 &1 \end{bmatrix}\right)
    \end{aligned}$\\
    \hline
\end{tabular}
\caption{Ground-truth distributions used in toy experiments}
\label{Table:KL_exp}
\end{table}
\small
\begin{table}[htb!]
\centering
\begin{tabular}{|c|c|c|c|} 
    \hline
        &Two-Modal  &Star   &Banana\\
    \hline
    $q(\bm{\epsilon})$& 
    $\mathcal{N}\left(\bm{0},I_{3\times3}\right)$& $\mathcal{N}\left(\bm{0},I_{3\times3}\right)$&
    $\mathcal{N}\left(\bm{0},I_{3\times3}\right)$\\
    \hline
    $q_{\bm{\theta}}(\bm{z}|\bm{\epsilon})$&
    $\mathcal{N}(\bm{\mu}_{\bm{\theta}_{1}}(\bm{\epsilon}), diag(\bm{\theta}_{2}))$&
    $\mathcal{N}(\bm{\mu}_{\bm{\theta}_{1}}(\bm{\epsilon}), diag(\bm{\theta}_{2}))$&
    $\mathcal{N}(\bm{\mu}_{\bm{\theta}_{1}}(\bm{\epsilon}), diag(\bm{\theta}_{2}))$\\
    \hline
    $\bm{\mu}_{\bm{\theta}_{1}}$&
    \makecell{Hidden units: (50, 50) \\ Hidden activation: ReLU \\ Initializer: Xavier\_normal}&
    \makecell{Hidden units: (50, 50) \\ Hidden activation: ReLU \\ Initializer: Xavier\_normal}&
    \makecell{Hidden units: (50, 50) \\ Hidden activation: ReLU \\ Initializer: Xavier\_normal}\\
    \hline
    \makecell{Hyper\\params}& 
    \makecell{$K^{(1)}_{t}=1\cdot10^2$ \\ $K^{(2)}_{t}=1\cdot10^3$ \\ $C_{\alpha}=3\cdot10^{-4}$ \\ $C_{\beta}=0.99$ \\ $C_{\gamma}=0.9$ \\ $\mu=0.999$}&
    \makecell{$K^{(1)}_{t}=2\cdot10^2$ \\ $K^{(2)}_{t}=2\cdot10^3$ \\ $C_{\alpha}=2\cdot10^{-4}$ \\ $C_{\beta}=0.999$ \\ $C_{\gamma}=0.9$ \\ $\mu=0.999$}&
    \makecell{$K^{(1)}_{t}=2\cdot10^2$ \\ $K^{(2)}_{t}=2\cdot10^3$ \\ $C_{\alpha}=3\cdot10^{-4}$ \\ $C_{\beta}=0.999$ \\ $C_{\gamma}=1.0$ \\ $\mu=0.999$}\\
    \hline
    \makecell{Time\\elapsed}&
    \makecell{0.037 sec/iter $\times$ 200 iters \\ 7.432 sec}&
    \makecell{0.040 sec/iter $\times$ 300 iters \\ 11.972 sec}&
    \makecell{0.042 sec/iter $\times$ 300 iters \\ 12.604 sec}\\
    \hline
\end{tabular}
\caption{Experimental settings for toy examples}
\label{Table:KL_exp2}
\end{table}
\normalsize

\small
\begin{table}[htb!]
\centering
\begin{tabular}{|c|c|c|c|} 
    \hline
        & Spam & Nodal & Waveform\\
    \hline
    $q(\bm{\epsilon})$ & $\mathcal{N}\left(\bm{0}, 100 \cdot I_{3\times3}\right)$ & $\mathcal{N}\left(\bm{0}, 100 \cdot I_{3\times3}\right)$ & $\mathcal{N}\left(\bm{0}, 100 \cdot I_{10\times10}\right)$ \\
    \hline
    $q_{\bm{\theta}}(\bm{z}|\bm{\epsilon})$ & $\mathcal{N}(\bm{\mu}_{\bm{\theta}_{1}}(\bm{\epsilon}), \bm{L}_{\bm{\theta}_{2}}\bm{L}_{\bm{\theta}_{2}}^{\mathsf{T}}), ~ \bm{z} \in \mathbb{R}^{3}$ & $\mathcal{N}(\bm{\mu}_{\bm{\theta}_{1}}(\bm{\epsilon}), \bm{L}_{\bm{\theta}_{2}}\bm{L}_{\bm{\theta}_{2}}^{\mathsf{T}}), ~ \bm{z} \in \mathbb{R}^{6}$ &
    $\mathcal{N}(\bm{\mu}_{\bm{\theta}_{1}}(\bm{\epsilon}), \bm{L}_{\bm{\theta}_{2}}\bm{L}_{\bm{\theta}_{2}}^{\mathsf{T}}), ~ \bm{z} \in \mathbb{R}^{22}$ \\
    \hline
     $\bm{\mu}_{\bm{\theta}_{1}}(\bm{\epsilon})$&
    \makecell{Hidden units: (200, 200) \\ Hidden activation: ReLU \\ Initializer: Xavier\_normal}&
    \makecell{Hidden units: (200, 200) \\ Hidden activation: ReLU \\ Initializer: Xavier\_normal}&
    \makecell{Hidden units: (200, 200) \\ Hidden activation: ReLU \\ Initializer: Xavier\_normal}\\
    \hline
    \makecell{Hyper\\params}& 
    \makecell{$K^{(1)}_{t}=5\cdot10^1$ \\ $K^{(2)}_{t}=5\cdot10^2$ \\ $C_{\alpha}^{(\bm{\theta}_1)}=1.5\cdot10^{-4}$ \\ $C_{\alpha}^{(\bm{\theta}_2)}=2\cdot10^{-1}$ \\ $C_{\beta}=0.999$ \\ $C_{\gamma}^{(\bm{\theta}_1)}=0.7$ \\ $C_{\gamma}^{(\bm{\theta}_2)}=0.6$ \\ $\mu=0.999$}&
    \makecell{$K^{(1)}_{t}=2\cdot10^2$ \\ $K^{(2)}_{t}=2\cdot10^3$ \\ $C_{\alpha}^{(\bm{\theta}_1)}=1.7\cdot10^{-4}$ \\ $C_{\alpha}^{(\bm{\theta}_2)}=1.7\cdot10^{-4}$ \\ $C_{\beta}=0.99$ \\ $C_{\gamma}^{(\bm{\theta}_1)}=0.75$ \\ $C_{\gamma}^{(\bm{\theta}_2)}=0.85$ \\ $\mu=0.999$}&
    \makecell{$K^{(1)}_{t}=1\cdot10^2$ \\ $K^{(2)}_{t}=1\cdot10^3$ \\ $C_{\alpha}^{(\bm{\theta}_1)}=3\cdot10^{-4}$ \\ $C_{\alpha}^{(\bm{\theta}_2)}=2.5\cdot10^{-4}$ \\ $C_{\beta}=0.999$ \\ $C_{\gamma}^{(\bm{\theta}_1)}=0.85$ \\ $C_{\gamma}^{(\bm{\theta}_2)}=0.85$ \\ $\mu=0.999$} \\
    \hline
    \makecell{Time\\elapsed}&
    \makecell{0.024 sec/iter $\times$ 600 iters \\ 14.180 sec}&
    \makecell{0.015 sec/iter $\times$ 600 iters \\ 8.986 sec}&
    \makecell{0.015 sec/iter $\times$ 3000 iters \\ 45.761 sec}\\
    \hline

\end{tabular}
\caption{Experimental settings for Bayesian logistic regression}
\label{Table:LR_exp}
\end{table}
\normalsize

\begin{figure}[htb!]
     \centering
     \begin{subfigure}[b]{0.49\textwidth}
         \centering
         \includegraphics[width=\textwidth]{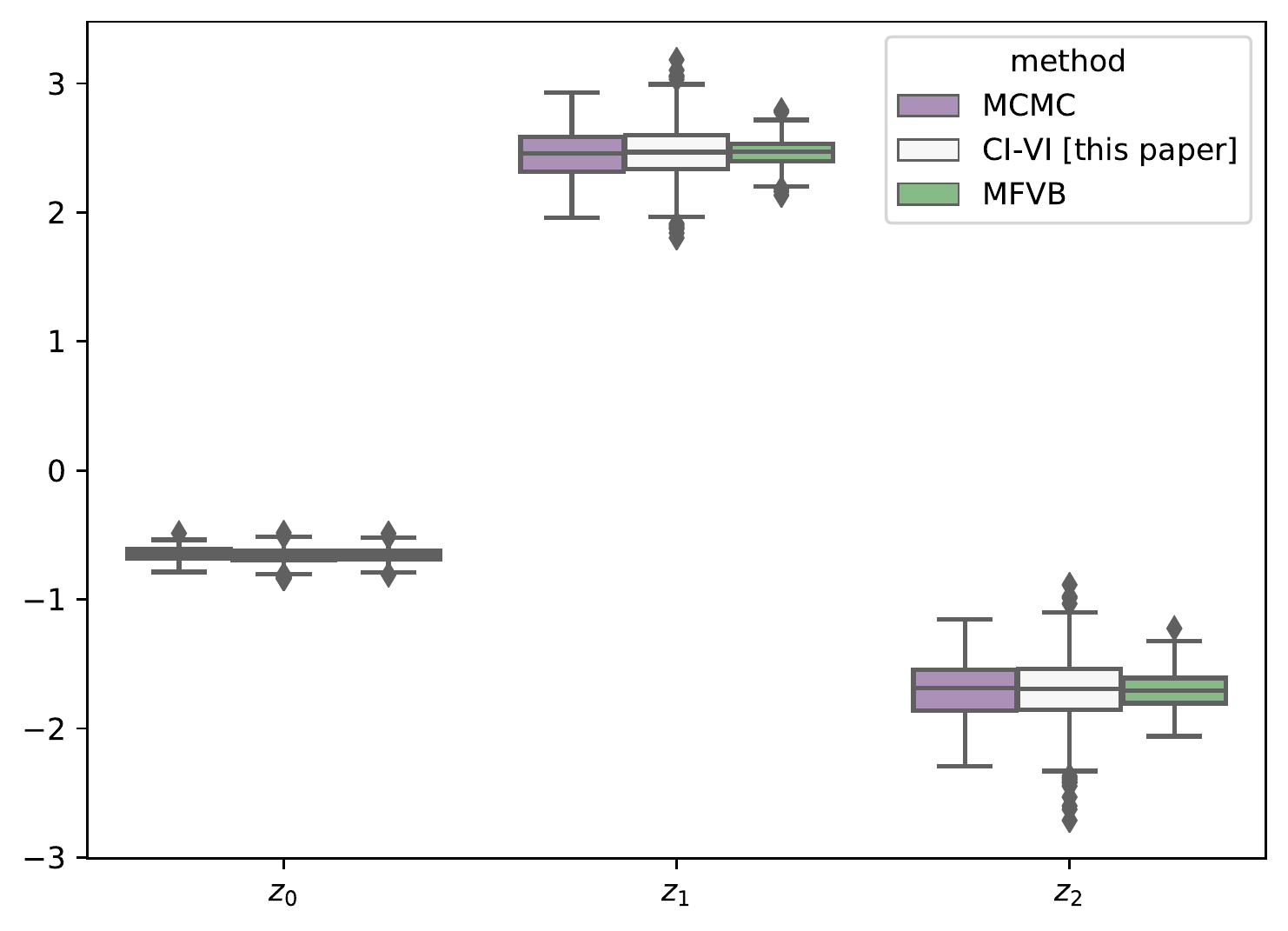}
         \label{fig:spam_boxplot}
     \end{subfigure}
     \hfill
     \begin{subfigure}[b]{0.49\textwidth}
         \centering
         \includegraphics[width=\textwidth]{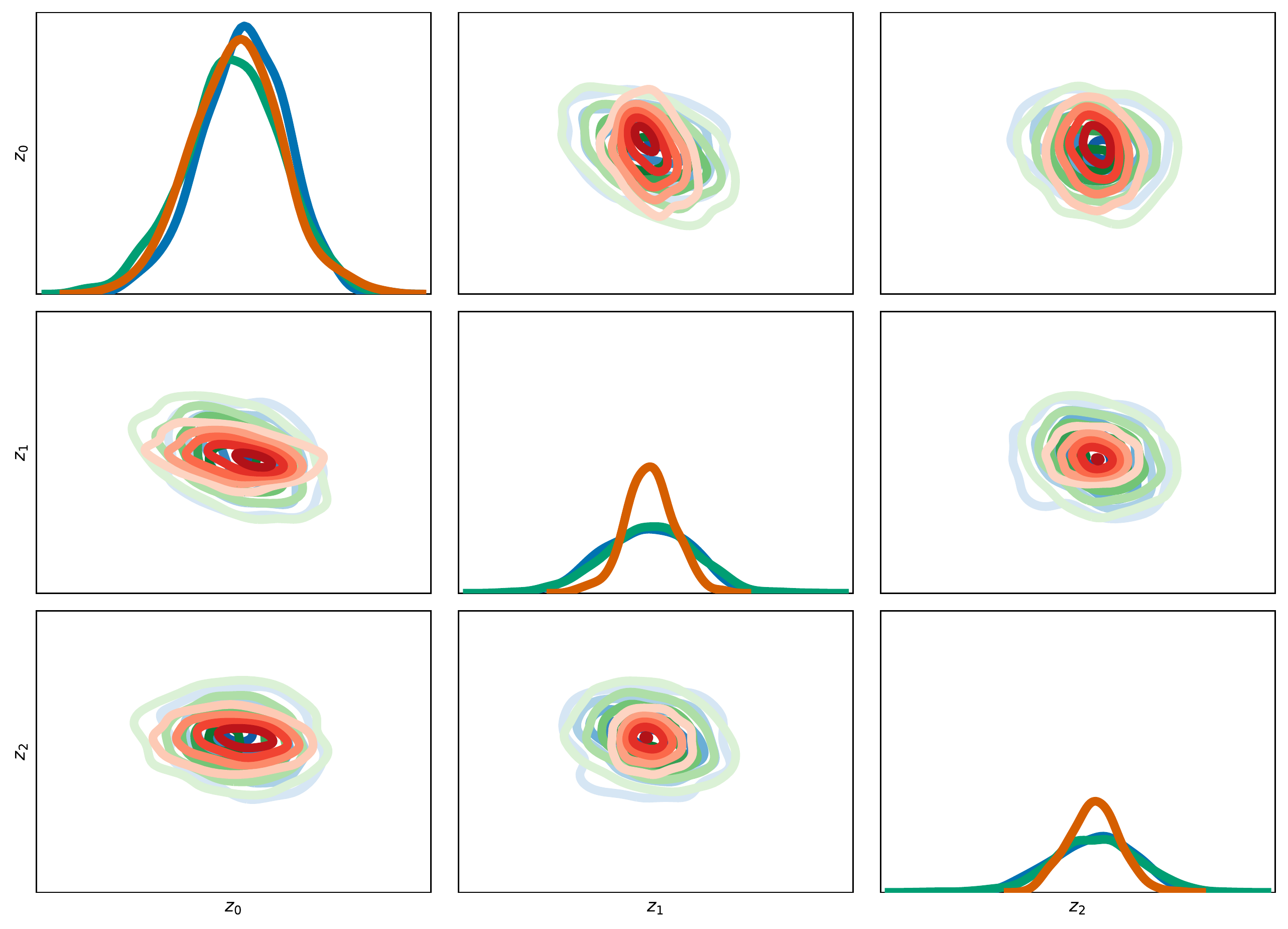}
         \label{fig:spam_pairplot}
     \end{subfigure}
     \begin{subfigure}[b]{0.49\textwidth}
         \centering
         \includegraphics[width=\textwidth]{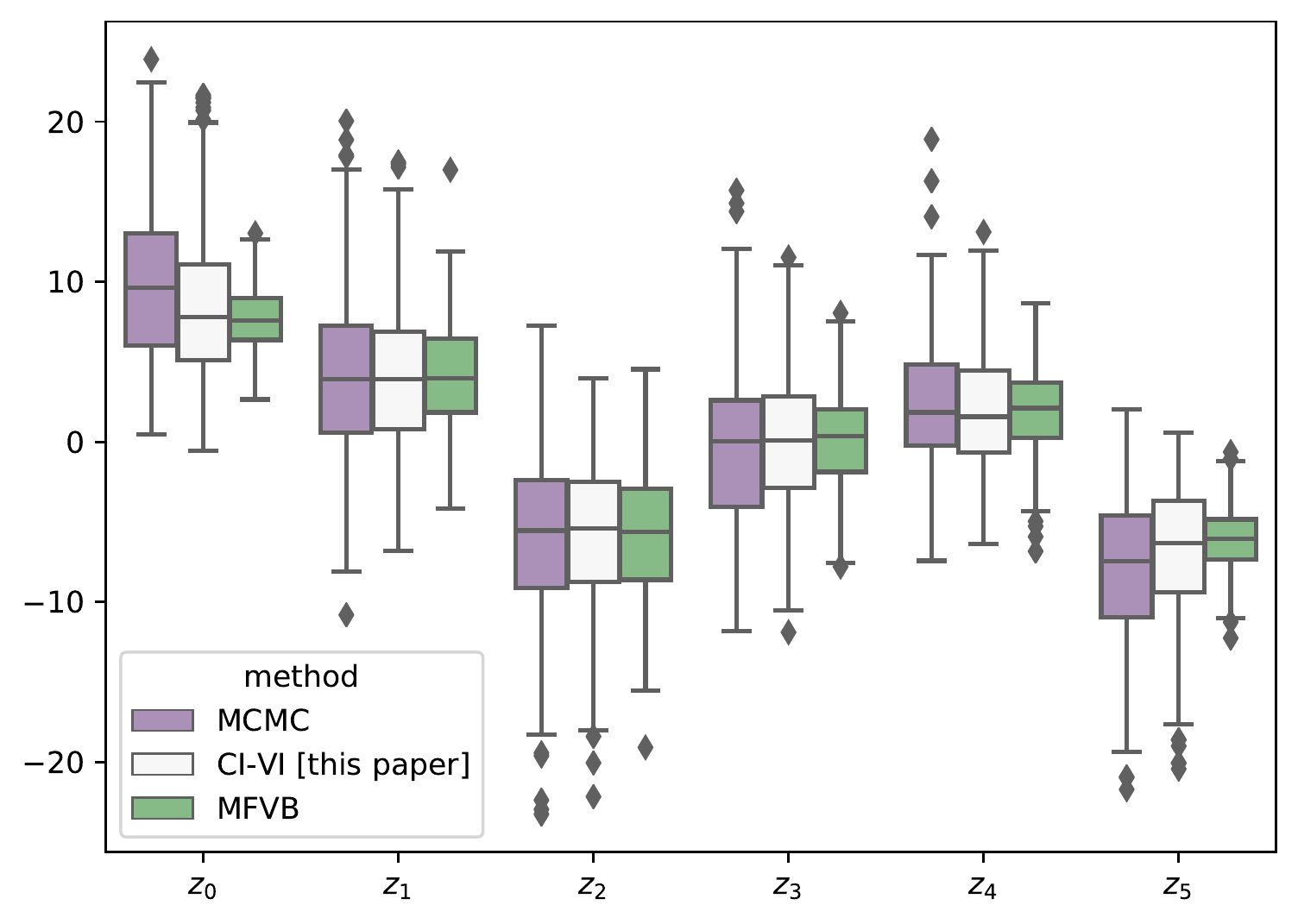}
         \label{fig:nodal_boxplot}
     \end{subfigure}
     \hfill
     \begin{subfigure}[b]{0.49\textwidth}
         \centering
         \includegraphics[width=\textwidth]{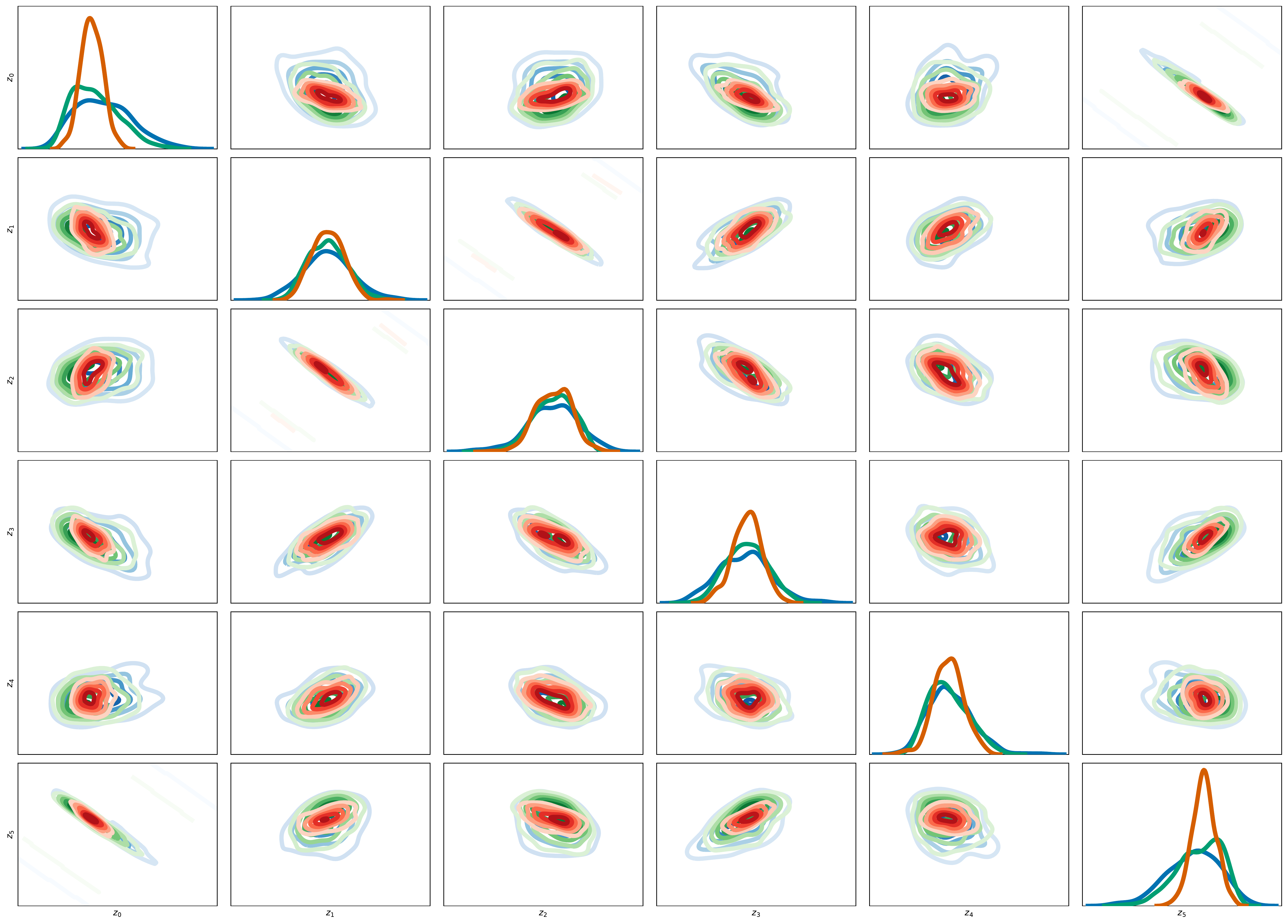}
         \label{fig:nodal_pairplot}
     \end{subfigure}
    \caption{Boxplot of marginal posteriors (left) and marginal \& pairwise joint-posteriors (right) using MCMC, CIVI and MFVB on Spam data-set (top row) and Nodal data-set (bottom row). On the right (grid plots), MCMC is blue, CIVI is green and MFVB is orange.}
\end{figure}

\begin{figure}[htb!]
    \centering
    \includegraphics[width=1.0\textwidth]{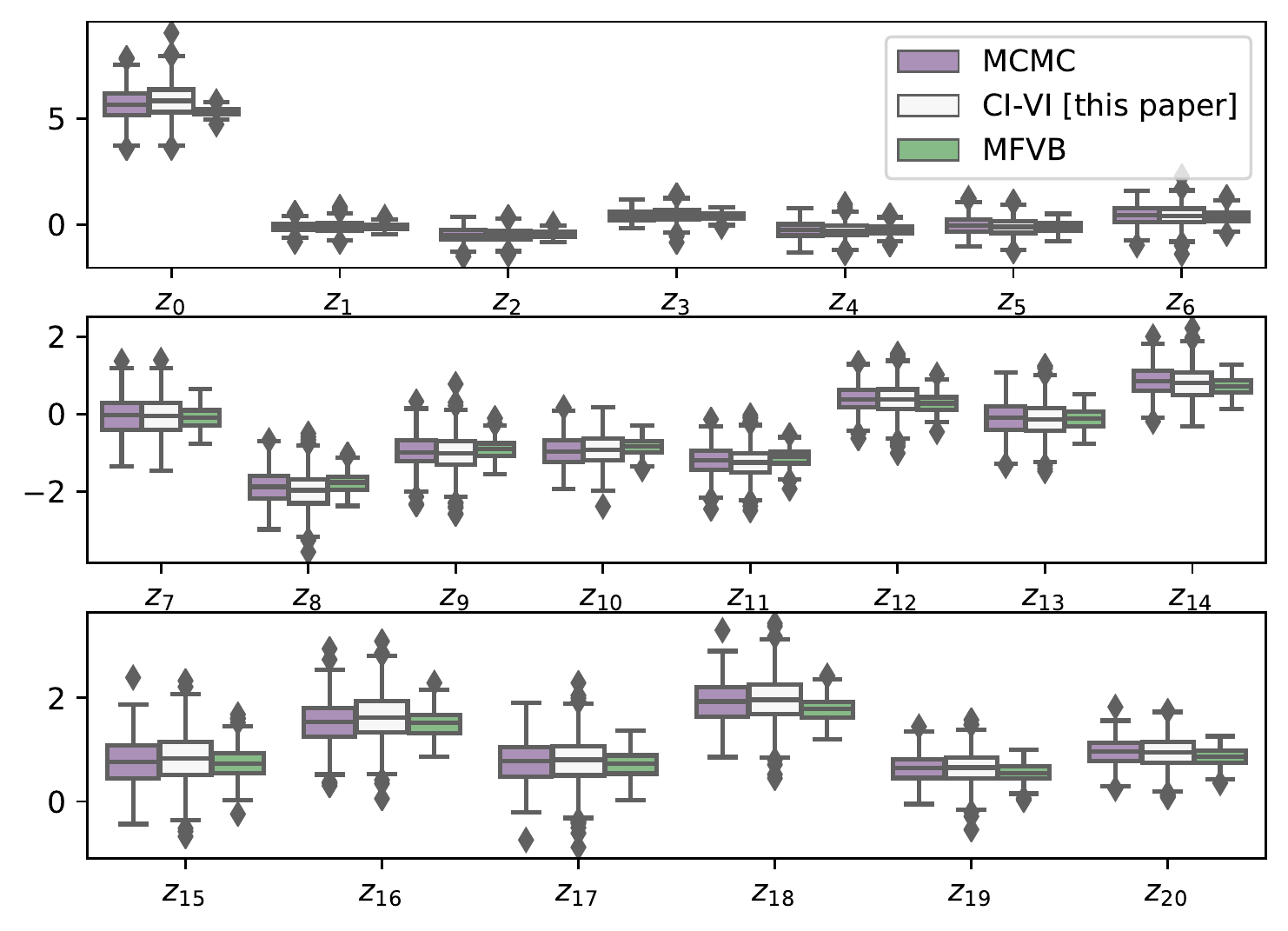}
    \caption{Boxplot of marginal posteriors using MCMC, CIVI and MFVB on Waveform data-set.}
    \label{fig:waveform_boxplot}
\end{figure}

\begin{figure}[ht!]
    \centering
    \includegraphics[width=1.4\textwidth, angle=90]{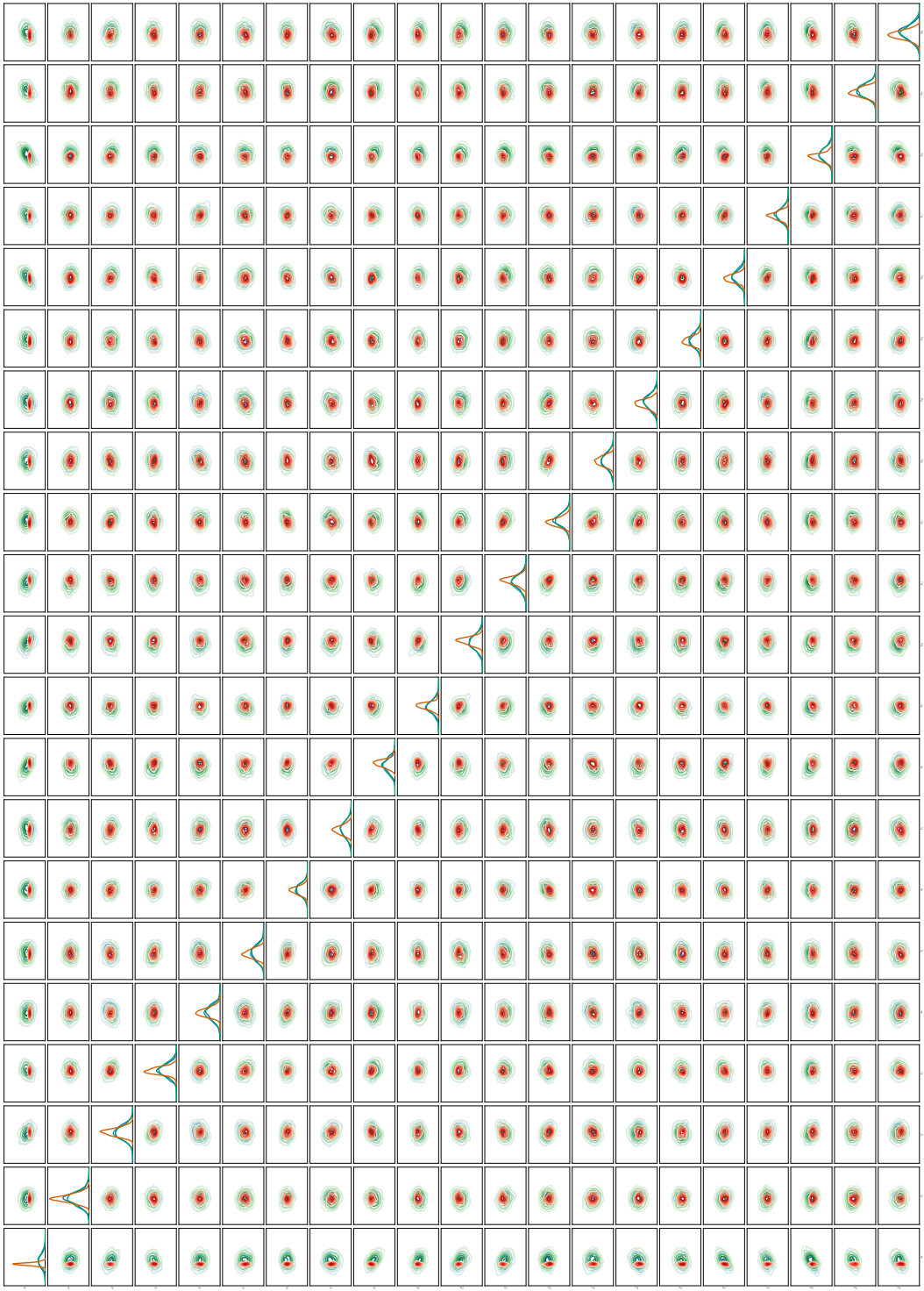}
    \caption{Marginal posteriors and pairwise joint-posteriors using MCMC, CIVI and MFVB on Waveform dataset. MCMC is blue, CIVI is green and MFVB is orange.}
    \label{fig:waveform_pairplot}
\end{figure}
\clearpage

\small
\begin{table}[htb!]
\centering
\begin{tabular}{|c|c|c|c|} 
    \hline
        & MNIST & FasionMNIST & PBT\\
    \hline
    $q(\bm{\epsilon})$ & $\mathcal{N}\left(\bm{0}, I_{10\times10}\right)$ & $\mathcal{N}\left(\bm{0},  I_{10\times10}\right)$ & $\mathcal{N}\left(\bm{0}, I_{10\times10}\right)$ \\
    \hline
    $q_{\bm{\theta}}(\bm{z}|\bm{\epsilon}, \bm{x}_i)$ & 
    \makecell{$\mathcal{N}(\bm{\mu}_{\bm{\theta}_1}(\bm{\epsilon}, \bm{x}_i), diag(\bm{\theta}_{2}))$ \\ $\bm{z} \in \mathbb{R}^{10}$} & 
    \makecell{$\mathcal{N}(\bm{\mu}_{\bm{\theta}_1}(\bm{\epsilon}, \bm{x}_i), diag(\bm{\theta}_{2}))$ \\ $\bm{z} \in \mathbb{R}^{10}$} & 
    \makecell{$\mathcal{N}(\bm{\mu}_{\bm{\theta}_1}(\bm{\epsilon}, \bm{x}_i), diag(\bm{\theta}_{2}))$ \\ $\bm{z} \in \mathbb{R}^{20}$} \\
    \hline
    $\bm{\mu}_{\bm{\theta}_{1}}(\bm{\epsilon}, \bm{x}_i)$&
    \makecell{Hidden units: (200, 200) \\ Hidden activation: ReLU \\ Initializer: Xavier\_normal}&
    \makecell{Hidden units: (200, 200) \\ Hidden activation: ReLU \\ Initializer: Xavier\_normal}&
    \makecell{Embedding dim: 256 \\ LSTM hidden dim: 256 \\ MLP hidden dim: 256 \\ MLP activation: ReLU}\\
    \hline
    $p_{\bm{\hat\theta}}(\bm{x}_i | \bm{z})$&
    $\prod_{i=1}^D \text{Bernoulli}[\pi_{\bm{\hat\theta}}(\bm{z})]$& 
    $\prod_{i=1}^D \text{Bernoulli}[\pi_{\bm{\hat\theta}}(\bm{z})]$& 
    $\prod_{i=1}^D \text{Categorical}[\bm{\pi}_{\bm{\hat\theta}}(\bm{z})]$\\
    \hline
    $\bm{\pi}_{\bm{\hat\theta}}(\bm{z})$&
    \makecell{Hidden units: (200, 200) \\ Hidden activation: ReLU \\ Initializer: Xavier\_normal}&
    \makecell{Hidden units: (200, 200) \\ Hidden activation: ReLU \\ Initializer: Xavier\_normal}&
    \makecell{Embedding dim: 256 \\ LSTM hidden dim: 256 \\ }\\
    \hline
    \makecell{Hyper\\params}& 
    \makecell{$K^{(1)}_{t}=300$ \\ $K^{(2)}_{t}=30$ \\ $C_{\alpha}^{(\bm{\theta}_1)}=1\cdot10^{-5}$ \\ $C_{\alpha}^{(\bm{\theta}_2)}=1\cdot10^{-4}$ \\ $C_{\alpha}^{(\bm{\hat\theta})}=5\cdot10^{-4}$ \\
    $C_{\beta}=0.999$ \\ $C_{\gamma}^{(\bm{\theta}_1)}=0.2$ \\ $C_{\gamma}^{(\bm{\theta}_2)}=0.2$ \\ $C_{\gamma}^{(\bm{\hat\theta})}=0.2$\\
    $\mu=0.999$}&
    \makecell{$K^{(1)}_{t}=300$ \\ $K^{(2)}_{t}=30$ \\ $C_{\alpha}^{(\bm{\theta}_1)}=1\cdot10^{-5}$ \\ $C_{\alpha}^{(\bm{\theta}_2)}=1\cdot10^{-4}$ \\ $C_{\alpha}^{(\bm{\hat\theta})}=5\cdot10^{-4}$ \\
    $C_{\beta}=0.999$ \\ $C_{\gamma}^{(\bm{\theta}_1)}=0.2$ \\ $C_{\gamma}^{(\bm{\theta}_2)}=0.2$ \\ $C_{\gamma}^{(\bm{\hat\theta})}=0.2$\\
    $\mu=0.999$}&
    \makecell{$K^{(1)}_{t}=320$ \\ $K^{(2)}_{t}=100$ \\ $C_{\alpha}^{(\bm{\theta}_1)}=2\cdot10^{-4}$ \\ $C_{\alpha}^{(\bm{\theta}_2)}=2\cdot10^{-4}$ \\ $C_{\alpha}^{(\bm{\hat\theta})}=5\cdot10^{-4}$ \\
    $C_{\beta}=0.99$ \\ $C_{\gamma}^{(\bm{\theta}_1)}=0.1$ \\ $C_{\gamma}^{(\bm{\theta}_2)}=0.1$ \\ $C_{\gamma}^{(\bm{\hat\theta})}=0.1$\\
    $\mu=0.999$} \\
    \hline
    \makecell{Time\\elapsed}&
    \makecell{0.0218 sec/iter $\times$ 18000 iters \\ 392 sec}&
    \makecell{0.0236 sec/iter $\times$ 18000 iters \\ 425 sec}&
    \makecell{0.0989 sec/iter $\times$ 40000 iters \\ 3956 sec}\\
    \hline
\end{tabular}
\caption{Experimental settings for variational autoencoder.}
\label{Table:VAE_exp}
\end{table}
\normalsize

\begin{table}[ht!]
    \centering
    \begin{tabular}{|c|c|c|c|c|c|c|c|}
        \hline
        Method & Time (ms) per iteration \\
        \hline
        SIVI                & 155 \\
        UIVI                & 69 \\
        ASCPG               & 157 \\
        SCGD                & 152 \\
        CI-VI [this paper]  & 56 \\
        \hline
    \end{tabular}
    \caption{Average time per iteration for training VAE on Intel i9-9900X 3.50GHz CPU}. 
    \label{Table:VAE_times}
\end{table}

\clearpage
\section{Detailed Descriptions of Assumptions}
Due to the lack of space, we provide  more detailed description of all assumptions required to establish theoretical convergence results for the proposed \text{CI-VI} Algorithm in this section.

Recall, we target the following nested optimisation problem:
\begin{equation}
    \min_{\boldsymbol{\theta}\in\mathbb{R}^{p}}\mathcal{L}(\boldsymbol{\theta}) = \mathbb{E}_{\nu}\left[f_{\nu}\left(\mathbb{E}_{\boldsymbol{\hat{\epsilon}}}\left[\boldsymbol{g}_{\boldsymbol{\hat{\epsilon}}}(\boldsymbol{\theta})\right] \right)\right]
\end{equation}
where for any $\nu,\boldsymbol{\hat{\epsilon}}$ we have $f_{\nu}(\cdot):\mathbb{R}^{n}\to \mathbb{R}$ and $\boldsymbol{g}_{\boldsymbol{\hat{\epsilon}}}(\cdot):\mathbb{R}^{p}\to \mathbb{R}^{n}$ and random variables $\nu,\boldsymbol{\hat{\epsilon}}$ follow some unknown distributions $\nu\sim \text{p}_{\nu}(\cdot)$ and $\boldsymbol{\hat{\epsilon}}\sim \text{p}_{\boldsymbol{\hat{\epsilon}}}(\cdot)$ correspondingly. Please notice, we do not assume that these distributions are independent. For brevity, let us denote $f(\boldsymbol{y}) = \mathbb{E}_{\nu}[f_{\nu}(\boldsymbol{y})]$ and $\boldsymbol{g}(\boldsymbol{\theta}) = \mathbb{E}_{\boldsymbol{\hat{\epsilon}}}[\boldsymbol{g}_{\boldsymbol{\hat{\epsilon}}}(\boldsymbol{\theta})]$, then it is easy to see:
\begin{align*}
    &\mathcal{L}(\boldsymbol{\theta}) = f(\boldsymbol{g}(\boldsymbol{\theta})),\\\nonumber
    &\nabla\mathcal{L}(\boldsymbol{\theta}) = \mathbb{E}_{\boldsymbol{\hat{\epsilon}}}[\nabla \boldsymbol{g}_{\boldsymbol{\hat{\epsilon}}}(\boldsymbol{\theta})^{\mathsf{T}}]\mathbb{E}_{\nu}[\nabla f_{\nu}(\mathbb{E}_{\boldsymbol{\hat{\epsilon}}}[\boldsymbol{g}_{\boldsymbol{\hat{\epsilon}}}(\boldsymbol{\theta})])] = \nabla \boldsymbol{g}(\boldsymbol{\theta})^{\mathsf{T}}\nabla f(\boldsymbol{g}(\boldsymbol{\theta})).
\end{align*}\\
Let the following assumption holds:
\begin{app_assumption}\label{assump_foo_1}\textbf{\text{1:}} 
\begin{enumerate}
    \item Function $f_{\nu}(\cdot)$ is bounded, i.e.$\forall\boldsymbol{y}\in\mathbb{R}^n$:
    \begin{equation*}
        |f_{\nu}(\boldsymbol{y})| \le B_f.
    \end{equation*}
    for any $v$.
    \item Function $f_{\nu}(\cdot)$ is $L_{f}-$ smooth. i.e.$\forall\boldsymbol{y}_1,\boldsymbol{y}_2\in\mathbb{R}^n$:
    \begin{equation*}
        ||\nabla f_{\nu}(\boldsymbol{y}_1) - \nabla f_{\nu}(\boldsymbol{y}_2)||_2 \le L_{f}||\boldsymbol{y}_1 - \boldsymbol{y}_2 ||_2.
    \end{equation*}
    for any $v$.
    \item Function $f_{\nu}(\cdot)$ has bounded gradient, i.e $\forall \boldsymbol{y}\in\mathbb{R}^n$:
    \begin{align*}
        &||\nabla f_{\nu}(\boldsymbol{y})||_2 \le M_{f}.
    \end{align*}
    for any $v$.
    \item Mapping $\boldsymbol{g}_{\boldsymbol{\hat{\epsilon}}}(\boldsymbol{\theta})$ is $M_g$ Lipschitz continuous, i.e. $\forall\boldsymbol{\theta},\boldsymbol{z}\in\mathbb{R}^{p}$:
    \begin{equation*}
        ||\boldsymbol{g}_{\boldsymbol{\hat{\epsilon}}}(\boldsymbol{\theta}) - \boldsymbol{g}_{\boldsymbol{\hat{\epsilon}}}(\boldsymbol{z})||_2 \le M_g||\boldsymbol{\theta} - \boldsymbol{z}||_2.
    \end{equation*}
    for any ${\boldsymbol{\hat{\epsilon}}}$.
    \item Mapping $\boldsymbol{g}_{\boldsymbol{\hat{\epsilon}}}(\boldsymbol{\theta})$ is $L_{g}-$ smooth. i.e.$\forall\boldsymbol{\theta,z}\in\mathbb{R}^p$:
    \begin{equation*}
        ||\nabla \boldsymbol{g}_{\boldsymbol{\hat{\epsilon}}}(\boldsymbol{\theta}) - \nabla \boldsymbol{g}_{\boldsymbol{\hat{\epsilon}}}(\boldsymbol{z})||_2 \le L_{g}||\boldsymbol{\theta} - \boldsymbol{z} ||_2.
    \end{equation*}
    for any ${\boldsymbol{\hat{\epsilon}}}$.
\end{enumerate}
\end{app_assumption}
Because distributions $\nu\sim \text{p}_{\nu}(\cdot)$ and $\boldsymbol{\hat{\epsilon}}\sim \text{p}_{\boldsymbol{\hat{\epsilon}}}(\cdot)$ are unknown, we assume the presence of two first order oracles $\mathcal{FOO}_{f}$ and $\mathcal{FOO}_{g}$, such that given fixed vectors $\boldsymbol{z}_t\in\mathbb{R}^p$, $\boldsymbol{y}_t\in\mathbb{R}^n$ and integer numbers $K^{(1)}_t$ and $K^{(2)}_{t}$ at time step $t$ they return the following collections:
\begin{align}\label{oracles}
    &\mathcal{FOO}_{f}[\boldsymbol{y}_t, K^{(1)}_{t}] = \{\nu_{t_a}, \nabla f_{\nu_{t_a}}(\boldsymbol{y}_t)\}^{K^{(1)}_t}_{a=1}, \ \ \ \ \ \text{where } \ \{\nu_{t_a}\}^{K^{(1)}_t}_{a=1} \ \ \ \text{are i.i.d} \\\nonumber
    &\mathcal{FOO}_{g}[\boldsymbol{z}_t, K^{(2)}_t] = \{{\boldsymbol{\hat{\epsilon}}_{t_a}},\boldsymbol{g}_{{\boldsymbol{\hat{\epsilon}}_{t_a}}}(\boldsymbol{z}_t), \nabla \boldsymbol{g}_{{\boldsymbol{\hat{\epsilon}}_{t_a}}}(\boldsymbol{z}_t)\}^{K^{(2)}_t}_{a=1}, \ \ \ \ \ \text{where } \ \{{\boldsymbol{\hat{\epsilon}}_{t_a}}\}^{K^{(2)}_t}_{a=1} \ \ \ \text{are i.i.d}
\end{align}
The complexity of the proposed algorithm will be evaluated in terms of total number of calls to first order oracles $\mathcal{FOO}_{f}[\cdot,\cdot]$, $\mathcal{FOO}_{g}[\cdot,\cdot]$.\\

\begin{app_assumption}\label{assump_foo_2}\textbf{\text{2:}}
At any given time step $t$, the oracles $\mathcal{FOO}_{f}[\boldsymbol{y}]_t$ and $\mathcal{FOO}_{g}[\boldsymbol{z}]_t$ satisfy the following two conditions for any $\boldsymbol{z}\in\mathbb{R}^p$ and $\boldsymbol{y}\in\mathbb{R}^n$:
\begin{enumerate}
    \item Sample collections 
    \begin{align*}
        \left(\{\nu_{1_a}\}^{K^{(1)}_1}_{a=1}, \{\boldsymbol{\hat{\epsilon}}_{1_a}\}^{K^{(2)}_1}_{a=1}\right),\left(\{\nu_{2_a}\}^{K^{(1)}_2}_{a=1}, \{\boldsymbol{\hat{\epsilon}}_{2_a}\}^{K^{(2)}_2}_{a=1}\right),\ldots, \left(\{\nu_{t_a}\}^{K^{(1)}_t}_{a=1}, \{{\boldsymbol{\hat{\epsilon}}_{t_a}}\}^{K^{(2)}_t}_{a=1}\right)
    \end{align*}
    are independent. 
    \item Unbiased estimates: for any $\boldsymbol{\theta}\in \mathbb{R}^{p}$, $\boldsymbol{y}\in \mathbb{R}^{n}$:
    \begin{align*}
    &\mathbb{E}_{{\boldsymbol{\hat{\epsilon}}_{t_a}}}\left[\boldsymbol{g}_{{\boldsymbol{\hat{\epsilon}}_{t_a}}}(\boldsymbol{\theta})\right] = \mathbb{E}_{\boldsymbol{\hat{\epsilon}}}\left[\boldsymbol{g}_{\boldsymbol{\hat{\epsilon}}}(\boldsymbol{\theta})\right],\\\nonumber
    &\mathbb{E}_{{\boldsymbol{\hat{\epsilon}}_{t_a}},\nu_{t_b}}\left[\nabla \boldsymbol{g}^{\mathsf{T}}_{{\boldsymbol{\hat{\epsilon}}_{t_a}}}(\boldsymbol{\theta})\nabla f_{\nu_{t_b}}(\boldsymbol{y})\right] = \mathbb{E}_{\boldsymbol{\hat{\epsilon}}}\left[\nabla \boldsymbol{g}^{\mathsf{T}}_{\boldsymbol{\hat{\epsilon}}}(\boldsymbol{\theta})\right]\mathbb{E}_{\nu}\left[\nabla f_{\nu}(\boldsymbol{y})\right].
    \end{align*}
    for any $t$ and $a\in [1,\ldots, K^{(2)}_{t}]$, $b\in [1,\ldots, K^{(1)}_{t}]$.
    \item Bounded variance of  stochastic gradients: for any $\boldsymbol{\theta}\in \mathbb{R}^{p}$, $\boldsymbol{y}\in \mathbb{R}^{n}$:
    \begin{align*}
        &\mathbb{E}_{\nu_{t_b}}||\nabla f_{\nu_{t_b}}(\boldsymbol{y}) - \nabla\mathbb{E}_{\nu}\left[f_{\nu}(\boldsymbol{y})\right]||^2_2 \le \sigma^2_1,\\\nonumber
        &\mathbb{E}_{{\boldsymbol{\hat{\epsilon}}_{t_a}}}||\nabla \boldsymbol{g}_{{\boldsymbol{\hat{\epsilon}}_{t_a}}}(\boldsymbol{\theta}) - \nabla\mathbb{E}_{\boldsymbol{\hat{\epsilon}}}\left[\boldsymbol{g}_{\boldsymbol{\hat{\epsilon}}}(\boldsymbol{\theta})\right]||^2_2 \le \sigma^2_2,\\\nonumber
        &\mathbb{E}_{{\boldsymbol{\hat{\epsilon}}_{t_a}}}\left[||\boldsymbol{g}_{{\boldsymbol{\hat{\epsilon}}_{t_a}}}(\boldsymbol{\theta}) - \boldsymbol{g}(\boldsymbol{\theta})||^2_2\right]\le \sigma^2_3.
    \end{align*}
    for any $t$ and $a\in [1,\ldots, K^{(2)}_{t}]$, $b\in [1,\ldots, K^{(1)}_{t}]$.
\end{enumerate}
\end{app_assumption}

\section{Theoretical Guarantees}
In this section, we establish all theoretical results needed for proving the main theorem and then present its proof.

\subsection{L-smoothness of function \texorpdfstring{$\mathcal{L}(\cdot)$}{TEXT}}
Our first result provides the important property of the overall compositional function $\mathcal{L}(\cdot)$ which will be used later in the convergence analysis of the proposed \text{CI-VI} Algorithm.
\begin{lemma}\label{lip_smoothnesss_claim}
Let Assumptions \textbf{\text{1}} and \textbf{\text{2}} hold, then function $\mathcal{L}(\cdot)$ is $L - $Lipschitz smooth, i,e:
\begin{equation}
    ||\nabla\mathcal{L}(\boldsymbol{\theta}_1) - \nabla\mathcal{L}(\boldsymbol{\theta}_2)||_2 \le L||\boldsymbol{\theta}_1 - \boldsymbol{\theta}_2||_2\ \ \ \ \ \ \ \ \ \forall \boldsymbol{\theta}_1,\boldsymbol{\theta}_2\in\mathbb{R}^p
\end{equation}
with $L = M^2_gL_f + L_gM_f$.
\end{lemma}
\begin{proof}
Assumption \textbf{\text{1}} implies that $||\nabla \boldsymbol{g}_{\boldsymbol{\hat{\epsilon}}}(\boldsymbol{\theta})||_2 \le M_g$ for any $\boldsymbol{\theta}\in\mathbb{R}^p$. Hence, using Jensen inequality as well as property of the norm we have:
\begin{align*}
    &||\nabla\mathcal{L}(\boldsymbol{\theta}_1) - \nabla\mathcal{L}(\boldsymbol{\theta}_2)||_2 = \\\nonumber
    &||\mathbb{E}_{\boldsymbol{\hat{\epsilon}}}\left[\nabla \boldsymbol{g}^{\mathsf{T}}_{\boldsymbol{\hat{\epsilon}}}(\boldsymbol{\theta}_1)\right]\mathbb{E}_{\nu}\left[\nabla f_{\nu}\left(\mathbb{E}_{\boldsymbol{\hat{\epsilon}}}\left[\boldsymbol{g}_{\boldsymbol{\hat{\epsilon}}}(\boldsymbol{\theta}_1)\right]\right)\right]  - \mathbb{E}_{\boldsymbol{\hat{\epsilon}}}\left[\nabla \boldsymbol{g}^{\mathsf{T}}_{\boldsymbol{\hat{\epsilon}}}(\boldsymbol{x_2})\right]\mathbb{E}_{\nu}\left[\nabla f_{\nu}\left(\mathbb{E}_{\boldsymbol{\hat{\epsilon}}}\left[\boldsymbol{g}_{\boldsymbol{\hat{\epsilon}}}(\boldsymbol{\theta}_2)\right]\right)\right]||_2\le\\\nonumber
    &||\mathbb{E}_{\boldsymbol{\hat{\epsilon}}}\left[\nabla \boldsymbol{g}^{\mathsf{T}}_{\boldsymbol{\hat{\epsilon}}}(\boldsymbol{\theta}_1)\right]\mathbb{E}_{\nu}\left[\nabla f_{\nu}\left(\mathbb{E}_{\boldsymbol{\hat{\epsilon}}}\left[\boldsymbol{g}_{\boldsymbol{\hat{\epsilon}}}(\boldsymbol{\theta}_1)\right]\right)\right] - \mathbb{E}_{\boldsymbol{\hat{\epsilon}}}\left[\nabla \boldsymbol{g}^{\mathsf{T}}_{\boldsymbol{\hat{\epsilon}}}(\boldsymbol{\theta}_1)\right]\mathbb{E}_{\nu}\left[\nabla f_{\nu}\left(\mathbb{E}_{\boldsymbol{\hat{\epsilon}}}\left[\boldsymbol{g}_{\boldsymbol{\hat{\epsilon}}}(\boldsymbol{\theta}_2)\right]\right) \right]||_2 + \\\nonumber
    &||\mathbb{E}_{\boldsymbol{\hat{\epsilon}}}\left[\nabla \boldsymbol{g}^{\mathsf{T}}_{\boldsymbol{\hat{\epsilon}}}(\boldsymbol{\theta}_1)\right]\mathbb{E}_{\nu}\left[\nabla f_{\nu}\left(\mathbb{E}_{\boldsymbol{\hat{\epsilon}}}\left[\boldsymbol{g}_{\boldsymbol{\hat{\epsilon}}}(\boldsymbol{\theta}_2)\right]\right)\right] - \mathbb{E}_{\boldsymbol{\hat{\epsilon}}}\left[\nabla \boldsymbol{g}^{\mathsf{T}}_{\boldsymbol{\hat{\epsilon}}}(\boldsymbol{x_2})\right]\mathbb{E}_{\nu}\left[\nabla f_{\nu}\left(\mathbb{E}_{\boldsymbol{\hat{\epsilon}}}\left[\boldsymbol{g}_{\boldsymbol{\hat{\epsilon}}}(\boldsymbol{\theta}_2)\right]\right)\right]||_2\le \\\nonumber
    &\mathbb{E}_{\boldsymbol{\hat{\epsilon}}}\left[||\nabla \boldsymbol{g}^{\mathsf{T}}_{\boldsymbol{\hat{\epsilon}}}(\boldsymbol{\theta}_1)||_2\right]||\mathbb{E}_{\nu}\left[\nabla f_{\nu}\left(\mathbb{E}_{\boldsymbol{\hat{\epsilon}}}\left[\boldsymbol{g}_{\boldsymbol{\hat{\epsilon}}}(\boldsymbol{\theta}_1)\right]\right)\right] - \mathbb{E}_{\nu}\left[\nabla f_{\nu}\left(\mathbb{E}_{\boldsymbol{\hat{\epsilon}}}\left[\boldsymbol{g}_{\boldsymbol{\hat{\epsilon}}}(\boldsymbol{\theta}_2)\right]\right)\right]||_2 + \\\nonumber
    &||\mathbb{E}_{\boldsymbol{\hat{\epsilon}}}\left[\nabla \boldsymbol{g}^{\mathsf{T}}_{\boldsymbol{\hat{\epsilon}}}(\boldsymbol{\theta}_1)\right] - \mathbb{E}_{\boldsymbol{\hat{\epsilon}}}\left[\nabla \boldsymbol{g}^{\mathsf{T}}_{\boldsymbol{\hat{\epsilon}}}(\boldsymbol{x_2})\right]||_2\mathbb{E}_{\nu}\left[||\nabla f_{\nu}\left(\mathbb{E}_{\boldsymbol{\hat{\epsilon}}}\left[\boldsymbol{g}_{\boldsymbol{\hat{\epsilon}}}(\boldsymbol{\theta}_2)\right]\right)||_2\right] \le\\\nonumber
    &M_g\mathbb{E}_{\nu}\left[||\nabla f_{\nu}\left(\mathbb{E}_{\boldsymbol{\hat{\epsilon}}}\left[\boldsymbol{g}_{\boldsymbol{\hat{\epsilon}}}(\boldsymbol{\theta}_1)\right]\right) - \nabla f_{\nu}\left(\mathbb{E}_{\boldsymbol{\hat{\epsilon}}}\left[\boldsymbol{g}_{\boldsymbol{\hat{\epsilon}}}(\boldsymbol{\theta}_2)\right]\right)||_2\right] + \\\nonumber
    &M_f\mathbb{E}_{\boldsymbol{\hat{\epsilon}}}\left[||\nabla \boldsymbol{g}^{\mathsf{T}}_{\boldsymbol{\hat{\epsilon}}}(\boldsymbol{\theta}_1) - \nabla \boldsymbol{g}^{\mathsf{T}}_{\boldsymbol{\hat{\epsilon}}}(\boldsymbol{\theta}_2)||_2\right] \le M_gL_f||\mathbb{E}_{\boldsymbol{\hat{\epsilon}}}\left[\boldsymbol{g}_{\boldsymbol{\hat{\epsilon}}}(\boldsymbol{\theta}_1)\right] - \mathbb{E}_{\boldsymbol{\hat{\epsilon}}}\left[\boldsymbol{g}_{\boldsymbol{\hat{\epsilon}}}(\boldsymbol{\theta}_2)\right]||_2 + \\\nonumber
    &M_f\mathbb{E}_{\boldsymbol{\hat{\epsilon}}}\left[||\nabla \boldsymbol{g}^{\mathsf{T}}_{\boldsymbol{\hat{\epsilon}}}(\boldsymbol{\theta}_1) - \nabla \boldsymbol{g}^{\mathsf{T}}_{\boldsymbol{\hat{\epsilon}}}(\boldsymbol{\theta}_2)||_2\right] \le M_gL_f\mathbb{E}_{\boldsymbol{\hat{\epsilon}}}\left[||\boldsymbol{g}_{\boldsymbol{\hat{\epsilon}}}(\boldsymbol{\theta}_1) - \boldsymbol{g}_{\boldsymbol{\hat{\epsilon}}}(\boldsymbol{\theta}_2)||_2\right] + \\\nonumber
    &M_f\mathbb{E}_{\boldsymbol{\hat{\epsilon}}}\left[||\nabla \boldsymbol{g}^{\mathsf{T}}_{\boldsymbol{\hat{\epsilon}}}(\boldsymbol{\theta}_1) - \nabla \boldsymbol{g}^{\mathsf{T}}_{\boldsymbol{\hat{\epsilon}}}(\boldsymbol{\theta}_2)||_2\right] \le M_gL_f\mathbb{E}_{\boldsymbol{\hat{\epsilon}}}\left[||\boldsymbol{g}_{\boldsymbol{\hat{\epsilon}}}(\boldsymbol{\theta}_1) - \boldsymbol{g}_{\boldsymbol{\hat{\epsilon}}}(\boldsymbol{\theta}_2)||_2\right] + M_fL_g||\boldsymbol{\theta}_1 - \boldsymbol{\theta}_2||_2 \le \\\nonumber
    &M^2_gL_f||\boldsymbol{\theta}_1 - \boldsymbol{\theta}_2||_2 + M_fL_g||\boldsymbol{\theta}_1 - \boldsymbol{\theta}_2||_2 = \left(M^2_gL_f + M_fL_g\right)||\boldsymbol{\theta}_1 - \boldsymbol{\theta}_2||_2 = L||\boldsymbol{\theta}_1 - \boldsymbol{\theta}_2||_2.
\end{align*}
\end{proof}


\subsection{Upper Bound for \texorpdfstring{$\mathbb{E}\left[\left|\left|\boldsymbol{g}(\boldsymbol{\theta}_{t+1}) -\boldsymbol{y}_{t+1}\right|\right|^2_{2}\right]$}{TEXT}.}
This lemma was adjusted for ADAM type algorithm based on techniques proposed in \cite{Mendi_2017}, and it shows the effect of an extrapolation smoothing technique (auxiliary variables $\boldsymbol{z}_t, \boldsymbol{y}_t$) used in \text{CI-VI} Algorithm. The result of this lemma is crucial for evaluating the asymptotic behavior of the approximation error between  $\boldsymbol{g}(\boldsymbol{\theta}_t)$ and  $\boldsymbol{y}_t$. \\

\begin{lemma}\label{Lemma_bound}
Consider auxiliary variable updates given by \text{CI-VI} Algorithm: 
\begin{align*}
    \boldsymbol{z}_{t+1} &= \left(1 - \frac{1}{\beta_{t}}\right)\boldsymbol{\theta}_{t} + \frac{1}{\beta_{t}}\boldsymbol{\theta}_{t+1}, \\ \boldsymbol{y}_{t+1} &= (1 - \beta_{t})\boldsymbol{y}_{t} + \beta_{t} \overline{\boldsymbol{g}_{t}(\boldsymbol{z}_{t+1})}. 
\end{align*}
Let $\mathbb{E}[\cdot]$ denote the expectation with respect to \emph{all} randomness incurred in iterations 1,\ldots,T in Algorithm 1.  For any $t$, the following holds: 
\begin{align}\label{norm_difference_bound}
    &\mathbb{E}\left[\left|\left|\boldsymbol{g}(\boldsymbol{\theta}_{t+1}) -\boldsymbol{y}_{t+1}\right|\right|^2_{2}\right] \leq \frac{L_{g}^{2}}{2}\mathbb{E}\left[\mathcal{D}_{t+1}^{2}\right]\\\nonumber
    & \hspace{15em}+ 2 \mathbb{E}\left[\left|\left|\boldsymbol{\mathcal{E}}_{t+1}\right|\right|_{2}^{2}\right],
\end{align}
with $\mathcal{D}_{t+1}$, $||\boldsymbol{\mathcal{E}}_{t+1}||^2_2$ satisfy the following recurrent inequalities:
\begin{align*}
    &\mathcal{D}_{t+1}  \leq  (1 - \beta_{t}) \mathcal{D}_{t} +\frac{2q^2M_{g}^{2}M_{f}^{2}}{\xi^{2}}\frac{\alpha_{t}^{2}}{\beta_{t}} + \beta_{t}\mathcal{F}_{t}^{2},\\\nonumber
    &\mathbb{E}\left[\left|\left|\boldsymbol{\mathcal{E}}_{t+1}\right|\right|_{2}^{2}\right]  \leq \left(1 - \beta_{t}\right)^{2}\mathbb{E}\left[\left|\left|\boldsymbol{\mathcal{E}}_{t}\right|\right|_{2}^{2}\right] + \frac{\beta_{t}^{2}}{K^{(2)}_t} \sigma_{3}^{2},\\\nonumber
    &\mathcal{F}_{t}^{2} \leq (1 - \beta_{t-1})\mathcal{F}_{t-1}^{2} + \frac{4 q^2M_{g}^{2}M_{f}^{2}}{\xi^{2}}\frac{\alpha_{t-1}^{2}}{\beta_{t-1}},
 \end{align*}
 and $\mathcal{D}_1 = 0,\ \  \mathbb{E}_{\text{total}}\left[||\boldsymbol{\mathcal{E}}_1||^2_2\right] = ||\boldsymbol{g}(\boldsymbol{\theta}_1)||^2_2, \ \ \mathcal{F}_1 = 0$.
\end{lemma}

\begin{proof}
Let us introduce the sequence of coefficients $\{\theta\}^{t}_{j=1}$ such that
\begin{equation*}
    \theta^{(t)}_{j} = \begin{cases}
        \beta_{j}\prod_{i=j+1}^{t}(1 - \beta_{i}) &  \text{if  } 1 \le j < t.\\
        \beta_{t} & \text{if  } j = t.
  \end{cases}
\end{equation*}
and we assume $\beta_1 = 1$ for simplicity. Denote $S_{t} = \sum_{j=1}^t\theta^{(t)}_{j}$, then:
\begin{align*}
    &S_{t} = \sum_{j=1}^t\theta^{(t)}_{j} = \\\nonumber
    &\beta_{t} + (1-\beta_{t})\beta_{t-1} + (1-\beta_{t})(1- \beta_{t-1})\beta_{t-2} + \cdots + (1-\beta_{t})(1- \beta_{t-1})\ldots (1 - \beta_2)\beta_{1} = \\\nonumber
    &\beta_{t} + (1 - \beta_{t})\left[\beta_{t-1} + (1- \beta_{t-1})\beta_{t-2} + \cdots + (1- \beta_{t-1})\ldots (1 - \beta_2)\beta_{1}\right] = \beta_{t} + (1 - \beta_{t})S_{t-1}.
\end{align*}
and $S_1 = \beta_1 = 1$. Since $\beta_1 = 1$ it implies $S_1 = S_2 = \ldots = S_t =  1$. By assuming $\Tilde{g}_{0}(\boldsymbol{z}_1) = \boldsymbol{0}_{q}$, one can represent $\boldsymbol{\theta}_{t+1}$ and $\boldsymbol{y}_{t+1}$ as a convex combinations of $\{\boldsymbol{z}_j\}^{t+1}_{j=1}$ and $\{\overline{\boldsymbol{g}_{j}(\boldsymbol{z}_{j+1})}\}^{t}_{j=1}$ respectively:
\begin{align}
    &\boldsymbol{\theta}_{t+1} = \sum_{j=1}^{t}\theta^{(t)}_{j}\boldsymbol{z}_{j+1},\ \ \ \text{ and } \ \ \ \boldsymbol{y}_{t+1} = \sum_{j=1}^{t}\theta^{(t)}_{j}\overline{\boldsymbol{g}_{j}(\boldsymbol{z}_{j+1})}.
\end{align}
Hence, using Taylor expansion for mapping $\boldsymbol{g}(\boldsymbol{z}_{j+1})$ around $\boldsymbol{\theta}_{t+1}$ we have:
\begin{align*}
    &\boldsymbol{y}_{t+1} = \sum_{j=1}^t\theta^{(t)}_{j}\overline{\boldsymbol{g}_{j}(\boldsymbol{z}_{j+1})} = \sum_{j=1}^t\theta^{(t)}_{j}\left[\boldsymbol{g}(\boldsymbol{z}_{j+1}) - \boldsymbol{g}(\boldsymbol{z}_{j+1}) + \overline{\boldsymbol{g}_{j}(\boldsymbol{z}_{j+1})}\right] = \sum_{j=1}^t\theta^{(t)}_{j}\boldsymbol{g}(\boldsymbol{z}_{j+1}) + \\\nonumber
    &\sum_{j=1}^t\theta^{(t)}_{j}\left[\overline{\boldsymbol{g}_{j}(\boldsymbol{z}_{j+1})} - \boldsymbol{g}(\boldsymbol{z}_{j+1})\right] = \sum_{j=1}^t\theta^{(t)}_{j}\left(\boldsymbol{g}(\boldsymbol{\theta}_{t+1}) + \nabla \boldsymbol{g}(\boldsymbol{\theta}_{t+1})[\boldsymbol{z}_{j+1} - \boldsymbol{\theta}_{t+1}] + o\left(||\boldsymbol{z}_{j+1} - \boldsymbol{\theta}_{t+1}||^2_2\right)\right)+ \\\nonumber
    &\sum_{j=1}^t\theta^{(t)}_{j}\left[\overline{\boldsymbol{g}_{j}(\boldsymbol{z}_{j+1})} - \boldsymbol{g}(\boldsymbol{z}_{j+1})\right] = \boldsymbol{g}(\boldsymbol{\theta}_{t+1}) + \nabla \boldsymbol{g}(\boldsymbol{\theta}_{t+1})\left[\sum_{j=1}^t\theta^{(t)}_{j}\boldsymbol{z}_{j+1} - \sum_{j=1}^t\theta^{(t)}_{j}\boldsymbol{\theta}_{k+1}\right] + \\\nonumber
    &\sum_{j=1}^t\theta^{(t)}_{j}o\left(||\boldsymbol{z}_{j+1} - \boldsymbol{\theta}_{t+1}||^2_2\right) + \sum_{j=1}^t\theta^{(t)}_{j}\left[\overline{\boldsymbol{g}_{j}(\boldsymbol{z}_{j+1})} - \boldsymbol{g}(\boldsymbol{z}_{j+1})\right] =  \sum_{j=1}^t\theta^{(t)}_{j}\left[\overline{\boldsymbol{g}_{j}(\boldsymbol{z}_{j+1})} - \boldsymbol{g}(\boldsymbol{z}_{j+1})\right] +\\\nonumber
    &\sum_{j=1}^t\theta^{(t)}_{j}o\left(||\boldsymbol{z}_{j+1} - \boldsymbol{\theta}_{t+1}||^2_2\right) + \boldsymbol{g}(\boldsymbol{\theta}_{t+1}).
\end{align*}
Therefore, using that $\boldsymbol{g}(\cdot)$ is $L_g-$smooth function:
\begin{align*}
    ||\boldsymbol{y}_{t+1} - \boldsymbol{g}(\boldsymbol{\theta}_{t+1})||_2 \le \frac{L_g}{2}\sum_{j=1}^{t}\theta^{(t)}_j||\boldsymbol{z}_{j+1} - \boldsymbol{\theta}_{t+1}||^2_2 + \left|\left|\sum_{j=1}^t\theta^{(t)}_{j}\left[\overline{\boldsymbol{g}_{j}(\boldsymbol{z}_{j+1})} - \boldsymbol{g}(\boldsymbol{z}_{j+1})\right]\right|\right|_2. 
\end{align*}
and, applying $(a+b)^2\le 2a^2 + 2b^2$:
\begin{align*}
    ||\boldsymbol{y}_{t+1} - \boldsymbol{g}(\boldsymbol{\theta}_{t+1})||^2_2 \le \frac{L^2_g}{2}\left(\underbrace{\sum_{j=1}^{t}\theta^{(t)}_j||\boldsymbol{z}_{j+1} - \boldsymbol{\theta}_{t+1}||^2_2}_{\mathcal{D}_{t+1}}\right)^2 + 2\left|\left|\underbrace{\sum_{j=1}^t\theta^{(t)}_{j}\left[\overline{\boldsymbol{g}_{j}(\boldsymbol{z}_{j+1})} - \boldsymbol{g}(\boldsymbol{z}_{j+1})\right]}_{\boldsymbol{\mathcal{E}}_{t+1}}\right|\right|^2_2. \ \ \
\end{align*}
Taking expectation $\mathbb{E}$ from both sides gives:
\begin{align}\label{expression_bound_1}
    \mathbb{E}\left[||\boldsymbol{y}_{t+1} - \boldsymbol{g}(\boldsymbol{\theta}_{t+1})||^2_2\right] \le \frac{L^2_g}{2}\mathbb{E}\left[\mathcal{D}^2_{t+1}\right] + 2\mathbb{E}\left[\left|\left|\boldsymbol{\mathcal{E}}_{t+1}\right|\right|^2_2\right].
\end{align}
where 
\begin{align*}
    \mathcal{D}_{t+1} = \sum_{j=1}^{t} \theta_{j}^{(t)}\left|\left|\boldsymbol{z}_{j+1} - \boldsymbol{\theta}_{t+1}\right|\right|_{2}^{2}, \ \ \  \boldsymbol{\mathcal{E}}_{t+1} = \sum_{j=1}^{t} \theta_{j}^{(t)}\left[\overline{\boldsymbol{g}_{j}(\boldsymbol{z}_{j+1})} - \boldsymbol{g}(\boldsymbol{z}_{j+1})\right].
\end{align*}
Let us bound both terms in expression (\ref{expression_bound_1}). Due to $\theta^{t}_{j} = (1-\beta_t)\theta^{(t-1)}_{j}$ for $j<t$ for the expression $\mathcal{D}_{t+1}$ we have:
\begin{align*}
    &\mathcal{D}_{t+1} = \sum_{j=0}^{t}\theta^{(t)}_j||\boldsymbol{z}_{j+1} - \boldsymbol{\theta}_{t+1}||^2_2 = \\\nonumber
    &\sum_{j=1}^{t-1}\theta^{(t)}_j||\boldsymbol{z}_{j+1} - \boldsymbol{\theta}_{t+1}||^2_2 + \beta_t||\boldsymbol{z}_{t+1} - \boldsymbol{\theta}_{t+1}||^2_2 = (1 - \beta_t)\sum_{j=1}^{t-1}\theta^{(t-1)}_j||\boldsymbol{z}_{j+1} - \boldsymbol{\theta}_{t+1}||^2_2 + \beta_t||\boldsymbol{z}_{t+1} - \boldsymbol{\theta}_{t+1}||^2_2\\\nonumber
    &=(1 - \beta_t)\sum_{j=1}^{t-1}\theta^{(t-1)}_j||\boldsymbol{z}_{j+1} - \boldsymbol{\theta}_{t+1}||^2_2 + \frac{(1 - \beta_t)^2}{\beta_t}||\boldsymbol{\theta}_{t+1} - \boldsymbol{\theta}_{t}||^2_2 = (1 - \beta_t)\sum_{j=1}^{t-1}\theta^{(t-1)}_j||\boldsymbol{z}_{j+1} - \boldsymbol{\theta}_{t}||^2_2 + \\\nonumber
    &\frac{(1 - \beta_t)^2}{\beta_t}||\boldsymbol{\theta}_{t+1} - \boldsymbol{\theta}_{t}||^2_2 + (1 - \beta_t)\sum_{j=1}^{t-1}\theta^{(t-1)}_j\left[||\boldsymbol{z}_{j+1} - \boldsymbol{\theta}_{t+1}||^2_2 - ||\boldsymbol{z}_{j+1} - \boldsymbol{\theta}_{t}||^2_2 \right] = (1 - \beta_t)\mathcal{D}_{t} + \\\nonumber
    &\frac{(1 - \beta_t)^2}{\beta_t}||\boldsymbol{\theta}_{t+1} - \boldsymbol{\theta}_{t}||^2_2 + (1 - \beta_t)\sum_{j=1}^{t-1}\theta^{(t-1)}_j\left[||\boldsymbol{z}_{j+1} - \boldsymbol{\theta}_{t+1}||_2 - ||\boldsymbol{z}_{j+1} - \boldsymbol{\theta}_{t}||_2\right]\times\\\nonumber
    &\left[||\boldsymbol{z}_{j+1} - \boldsymbol{\theta}_{t+1}||_2 + ||\boldsymbol{z}_{j+1} - \boldsymbol{\theta}_{t}||_2 \right] \le (1 - \beta_t)\mathcal{D}_{t} + \frac{(1 - \beta_t)^2}{\beta_t}||\boldsymbol{\theta}_{t+1} - \boldsymbol{\theta}_{t}||^2_2 + \\\nonumber
    &(1 - \beta_t)\sum_{j=1}^{t-1}\theta^{(t-1)}_j||\boldsymbol{\theta}_{t+1} - \boldsymbol{\theta}_t||_2\left[||\boldsymbol{\theta}_{t+1} - \boldsymbol{\theta}_t||_2 + 2||\boldsymbol{\theta}_{t} - \boldsymbol{z}_{j+1}||_2\right] = (1 - \beta_t)\mathcal{D}_{t} + \frac{(1 - \beta_t)^2}{\beta_t}||\boldsymbol{\theta}_{t+1} - \boldsymbol{\theta}_{t}||^2_2 \\\nonumber
    &+(1 - \beta_t)||\boldsymbol{\theta}_{t+1} - \boldsymbol{\theta}_t||^2_2 + 2(1 - \beta_t)||\boldsymbol{\theta}_{t+1} - \boldsymbol{\theta}_t||_2\sum_{j=1}^{t-1}\theta^{(t-1)}_j||\boldsymbol{\theta}_{t} - \boldsymbol{z}_{j+1}||_2 = (1 - \beta_t)\mathcal{D}_t + \\\nonumber
    &\frac{1-\beta_t}{\beta_t}||\boldsymbol{\theta}_{t+1} - \boldsymbol{\theta}_{t}||^2_2 + 2(1 - \beta_t)||\boldsymbol{\theta}_{t+1} - \boldsymbol{\theta}_t||_2\sum_{j=1}^{t-1}\theta^{(t-1)}_j||\boldsymbol{\theta}_{t} - \boldsymbol{z}_{j+1}||_2. 
    \end{align*}
Applying $2ab \le \frac{1}{\beta_t}a^2 + \beta_tb^2$:
    \begin{align*}
    &\mathcal{D}_{t+1} \le \\\nonumber
    &(1 - \beta_t)\mathcal{D}_t + \frac{1-\beta_t}{\beta_t}||\boldsymbol{\theta}_{t+1} - \boldsymbol{\theta}_{t}||^2_2 + (1 - \beta_t)\left[\frac{||\boldsymbol{\theta}_{t+1} - \boldsymbol{\theta}_t||^2_2}{\beta_t} + \beta_t\left(\sum_{j=1}^{t-1}\theta^{(t-1)}_j||\boldsymbol{\theta}_{t} - \boldsymbol{z}_{j+1}||_2\right)^2\right] = \\\nonumber
    &(1 - \beta_t)\mathcal{D}_t + 2\frac{1-\beta_t}{\beta_t}||\boldsymbol{\theta}_{t+1} - \boldsymbol{\theta}_{t}||^2_2 + (1 - \beta_t)\beta_t\left(\sum_{j=1}^{t-1}\theta^{(t-1)}_j||\boldsymbol{\theta}_{t} - \boldsymbol{z}_{j+1}||_2\right)^2 \le \\\nonumber
    &(1 - \beta_t)\mathcal{D}_t + \frac{2}{\beta_t}||\boldsymbol{\theta}_{t+1} - \boldsymbol{\theta}_{t}||^2_2 + \beta_t\left(\underbrace{\sum_{j=1}^{t-1}\theta^{(t-1)}_j||\boldsymbol{\theta}_{t} - \boldsymbol{z}_{j+1}||_2}_{\mathcal{F}_{t}}\right)^2.
\end{align*}
Applying the primal variable update with (\ref{bound_expressions}) gives:
\begin{align}\label{d_t_expression}
    &\mathcal{D}_{t+1} \le \\\nonumber
    &(1 - \beta_t)\mathcal{D}_t + \frac{2\alpha^2_t}{\beta_t}\left|\left|\frac{\boldsymbol{m}_t}{\sqrt{\boldsymbol{v}_t} + \xi}\right|\right|^2_2 + \beta_t\mathcal{F}^2_t \le (1 - \beta_t)\mathcal{D}_t + \frac{2q^2M^2_gM^2_f}{\xi^2}\frac{\alpha^2_t}{\beta_t} +  \beta_t\mathcal{F}^2_t.
\end{align}
Next, for expression $\mathcal{F}_t$ we have (using $\theta^{t-1}_{j} = (1-\beta_{t-1})\theta^{(t-2)}_{j}$ for $j<t-1$):
\begin{align*}
    &\mathcal{F}_t = 
    \sum_{j=1}^{t-1}\theta^{(t-1)}_j||\boldsymbol{\theta}_{t} - \boldsymbol{z}_{j+1}||_2 = \\\nonumber
    &\sum_{j=1}^{t-2}\theta^{(t-1)}_j||\boldsymbol{\theta}_{t} - \boldsymbol{z}_{j+1}||_2 + \theta^{(t-1)}_{t-1}||\boldsymbol{\theta}_{t} - \boldsymbol{z}_{t}||_2 = \sum_{j=1}^{t-2}\theta^{(t-1)}_j||\boldsymbol{\theta}_{t} - \boldsymbol{z}_{j+1}||_2 + \beta_{t-1}||\boldsymbol{\theta}_{t} - \boldsymbol{z}_{t}||_2 = \\\nonumber
    &(1 - \beta_{t-1})\sum_{j=1}^{t-2}\theta^{(t-2)}_j||\boldsymbol{\theta}_{t} - \boldsymbol{z}_{j+1}||_2 + \beta_{t-1}||\boldsymbol{\theta}_{t} - \boldsymbol{z}_{t}||_2 \le (1 - \beta_{t-1})||\boldsymbol{\theta}_t - \boldsymbol{\theta}_{t-1}||_2 + \\\nonumber 
    &(1 - \beta_{t-1})\sum_{j=1}^{t-2}\theta^{(t-2)}_j\left[||\boldsymbol{\theta}_{t-1} - \boldsymbol{z}_{j+1}||_2 + ||\boldsymbol{\theta}_{t} - \boldsymbol{\theta}_{t-1}||_2\right] = (1 - \beta_{t-1})\left(\mathcal{F}_{t-1} + 2||\boldsymbol{\theta}_t - \boldsymbol{\theta}_{t-1}||_2\right). 
\end{align*}
and $\mathcal{F}_1 = 0$. Hence, applying $(a+b)^2 \le (1 + \alpha)a^2 + (1 + \frac{1}{\alpha})b^2$ for $\alpha = \beta_{t-1} > 0$ and using primal variable update with (\ref{bound_expressions}).
\begin{align}\label{f_t_expression}
    &\mathcal{F}^2_{t} \le \\\nonumber
    &(1 + \beta_{t-1})(1 -  \beta_{t-1})^2\mathcal{F}^2_{t-1} + 4\left(1 + \frac{1}{\beta_{t-1}}\right)(1 - \beta_{t-1})^2||\boldsymbol{\theta}_t - \boldsymbol{\theta}_{t-1}||^2_2 \le \\\nonumber
    &(1 - \beta_{t-1})\mathcal{F}^2_{t-1} + \frac{4}{\beta_{t-1}}||\boldsymbol{\theta}_t - \boldsymbol{\theta}_{t-1}||^2_2 = (1 - \beta_{t-1})\mathcal{F}^2_{t-1} + \frac{4\alpha^2_{t-1}}{\beta_{t-1}}\left|\left|\frac{\boldsymbol{m}_{t-1}}{\sqrt{\boldsymbol{v}_{t-1}}+\xi}\right|\right|^2_2 \le \\\nonumber
    &(1 - \beta_{t-1})\mathcal{F}^2_{t-1} + \frac{4q^2M^2_gM^2_f}{\xi^2}\frac{\alpha^2_{t-1}}{\beta_{t-1}}.
\end{align}
Finally, for  $\boldsymbol{\mathcal{E}}_{t+1}$ we have (using $\theta^{t}_{j} = (1-\beta_t)\theta^{(t-1)}_{j}$ for $j<t$):
\begin{align*}
    &\boldsymbol{\mathcal{E}}_{t+1} = \\\nonumber &\sum_{j=1}^t\theta^{(t)}_{j}\left[\overline{\boldsymbol{g}_{j}(\boldsymbol{z}_{j+1})} - \boldsymbol{g}(\boldsymbol{z}_{j+1})\right]  = \sum_{j=1}^{t-1}\theta^{(t)}_{j}\left[\overline{\boldsymbol{g}_{j}(\boldsymbol{z}_{j+1})} - \boldsymbol{g}(\boldsymbol{z}_{j+1})\right] + \beta_t[\overline{\boldsymbol{g}_{t}(\boldsymbol{z}_{t+1})} - \boldsymbol{g}(\boldsymbol{z}_{t+1})] = \\\nonumber
    &\sum_{j=1}^{t-1}(1 - \beta_t)\theta^{(t-1)}_{j}\left[\overline{\boldsymbol{g}_{j}(\boldsymbol{z}_{j+1})} - \boldsymbol{g}(\boldsymbol{z}_{j+1})\right] + \beta_t[\overline{\boldsymbol{g}_{t}(\boldsymbol{z}_{t+1})} - \boldsymbol{g}(\boldsymbol{z}_{t+1})] = \\\nonumber
    &(1 - \beta_t)\sum_{j=1}^{t-1}\theta^{(t-1)}_{j}\left[\overline{\boldsymbol{g}_{j}(\boldsymbol{z}_{j+1})} - \boldsymbol{g}(\boldsymbol{z}_{j+1})\right] + \beta_t\left[\overline{\boldsymbol{g}_{t}(\boldsymbol{z}_{t+1})} - \boldsymbol{g}(\boldsymbol{z}_{t+1})\right] = \\\nonumber 
    &(1 - \beta_t)\boldsymbol{\mathcal{E}}_{t} + \beta_t\left[\overline{\boldsymbol{g}_{t}(\boldsymbol{z}_{t+1})} - \boldsymbol{g}(\boldsymbol{z}_{t+1})\right].
\end{align*}
Due to the fact, that all samplings done at iteration $t_1$ are independent from samplings done at iteration $t_2\neq t_1$, then consider expectation with all randomness induced at iteration $t$ (with fixed iterative value $\boldsymbol{\theta}_t$):
\begin{equation}\label{expec_prop_new}
    \mathbb{E}_{t}\left[\cdot\right] = \mathbb{E}_{K^{(1)}_{t},d_t,K^{(2)}_t}\left[\cdot |\boldsymbol{\theta}_{t}\right]. 
\end{equation}
for any $t$. Using that $\boldsymbol{\mathcal{E}}_t$ is independent from the randomness induced at iteration $t$ we have: 
\begin{align*}
    &\mathbb{E}_t\left[\left|\left|\boldsymbol{\mathcal{E}}_{t+1}\right|\right|^2_2\right] = \\\nonumber
    &\mathbb{E}_t\left[\left((1 - \beta_t)\boldsymbol{\mathcal{E}}^{\mathsf{T}}_{t} + \beta_t\left[\overline{\boldsymbol{g}_{t}(\boldsymbol{z}_{t+1})} - \boldsymbol{g}(\boldsymbol{z}_{t+1})\right]^{\mathsf{T}}\right)\left((1 - \beta_t)\boldsymbol{\mathcal{E}}_{t} + \beta_t\left[\overline{\boldsymbol{g}_{t}(\boldsymbol{z}_{t+1})} - \boldsymbol{g}(\boldsymbol{z}_{t+1})\right]\right) \right] = \\\nonumber
    &(1- \beta_t)^2\left|\left|\boldsymbol{\mathcal{E}}_{t}\right|\right|^2_2 + 2\beta_t(1 - \beta_t)\boldsymbol{\mathcal{E}}^{\mathsf{T}}_{t}\mathbb{E}_t\left[\overline{\boldsymbol{g}_{t}(\boldsymbol{z}_{t+1})} - \boldsymbol{g}(\boldsymbol{z}_{t+1})\right] + \beta^2_t\mathbb{E}_t\left[\left|\left|\overline{\boldsymbol{g}_{t}(\boldsymbol{z}_{t+1})} - \boldsymbol{g}(\boldsymbol{z}_{t+1})\right|\right|^2_2\right] = \\\nonumber
    &(1- \beta_t)^2\left|\left|\boldsymbol{\mathcal{E}}_{t}\right|\right|^2_2 + \beta^2_t\mathbb{E}_t\left[\left|\left|\overline{\boldsymbol{g}_{t}(\boldsymbol{z}_{t+1})} - \boldsymbol{g}(\boldsymbol{z}_{t+1})\right|\right|^2_2\right].
\end{align*}
where $\mathbb{E}_t\left[\overline{\boldsymbol{g}_{t}(\boldsymbol{z}_{t+1})} - \boldsymbol{g}(\boldsymbol{z}_{t+1})\right] = \boldsymbol{0}$ due to Assumption 2.2. Assumption 2.3 implies $\mathbb{E}_{t}\left[\left|\left|\overline{\boldsymbol{g}_{t}(\boldsymbol{z}_{t+1})} - \boldsymbol{g}(\boldsymbol{z}_{t+1})\right|\right|^2_2\right] \le \frac{1}{K^{(2)}_t}\sigma^2_3$, therefore,
\begin{align*}
    &\mathbb{E}_t\left[\left|\left|\boldsymbol{\mathcal{E}}_{t+1}\right|\right|^2_2\right] \le (1- \beta_t)^2\left|\left|\boldsymbol{\mathcal{E}}_{t}\right|\right|^2_2 + \frac{\beta^2_t}{K^{(2)}_t}\sigma^2_3.
\end{align*}
Taking expectation $\mathbb{E}$ from both sides of the above inequality and using (\ref{expec_prop_new}) and the law of total expectation, we have:
\begin{align}\label{bound_on_e_term}
    &\mathbb{E}\left[\left|\left|\boldsymbol{\mathcal{E}}_{t+1}\right|\right|^2_2\right] \le (1- \beta_t)^2\mathbb{E}\left[\left|\left|\boldsymbol{\mathcal{E}}_{t}\right|\right|^2_2\right] + \frac{\beta^2_t}{K^{(2)}_t}\sigma^2_3.
\end{align}
Combining (\ref{expression_bound_1}), (\ref{d_t_expression}), (\ref{f_t_expression}), and (\ref{bound_on_e_term}) gives the statement of the Lemma.
\end{proof}


\subsection{Asymptotic rate for recurrent inequalities.}
Next lemma provides a tool to upper-bound the recurrent expressions and it was established in  \cite{Mendi_2017}. We present it here for completeness. \\
\begin{lemma}\label{rec_lem}
Let $\eta_t = \frac{C_{\eta}}{t^a}$, $\zeta_t = \frac{C_{\zeta}}{t^b}$, where $C_{\eta} > 1 + b - a$, $C_{\zeta} > 0$, $(b - a)\notin(-1,0)$ and $0 < a\le 1$. Consider the following recurrent inequality:
\begin{equation*}
    A_{t+1} \le (1 - \eta_t + C_1\eta^2_t)A_t + C_2\zeta_t.
\end{equation*}
where $C_1, C_2 \ge 0$. Then, there is a constant $C_{A} > 0 $ such that  $A_t \le  \frac{C_{A}}{t^{b-a}}$.
\end{lemma}
\begin{proof}
Let us introduce constant $C_{A}$ such that
\begin{equation*}
    C_{A} = \max_{t\le (C_1C^2_{\eta})^{\frac{1}{a}}+1}A_{t}t^{b - a} + \frac{C_2C_{\zeta}}{C_{\eta} - 1 - b + a}.
\end{equation*}
The claim will be proved by induction. Consider two cases here:
\begin{enumerate}
    \item \textbf{If $t \le (C_1C^2_{\eta})^{\frac{1}{a}}$: } Then, by from the definition of constant $C_{A}$ it follows immediately:
    \begin{align*}
        A_{t} \le C_{A}t^{a-b} = \frac{C_{A}}{t^{b-a}}.
    \end{align*}
    \item \textbf{If $t > (C_1C^2_{\eta})^{\frac{1}{a}}$: } Assume that $A_t \le \frac{C_{A}}{t^{b-a}}$ for some  $t > (C_1C^2_{\eta})^{\frac{1}{a}}$. Hence:
    \begin{align}\label{A_t_expression_1}
        &A_{t+1} \le \\\nonumber
        &(1 - \eta_t + C_1\eta^2_t)A_t + C_2\zeta_t = \left(1 - \frac{C_{\eta}}{t^a} + C_1\frac{C^2_{\eta}}{t^{2a}}\right)A_t + C_2\frac{C_{\zeta}}{t^b} \le \\\nonumber
        &\left(1 - \frac{C_{\eta}}{t^a} + C_1\frac{C^2_{\eta}}{t^{2a}}\right)\frac{C_{A}}{t^{b-a}} + C_2\frac{C_{\zeta}}{t^b} = \frac{C_{A}}{t^{b-a}} - \frac{C_AC_{\eta}}{t^{b}} + \frac{C_1C_AC^2_{\eta}}{t^{a + b}} + \frac{C_2C_{\zeta}}{t^{b}} = \\\nonumber
        &\frac{C_A}{(t+1)^{b-a}} - C_A\left[\underbrace{\frac{1}{(t+1)^{b-a}} - \frac{1}{t^{b-a}} + \frac{C_{\eta}}{t^{b}} - \frac{C_1C^2_{\eta}}{t^{a+b}}}_{\Delta_{t+1}}\right] + \frac{C_2C_{\zeta}}{t^b} = \\\nonumber
        &\frac{C_A}{(t+1)^{b-a}} - C_A\Delta_{t+1} + \frac{C_2C_{\zeta}}{t^b} = \frac{C_A}{(t+1)^{b-a}} - \Delta_{t+1}\left(C_A - \frac{C_2C_{\zeta}}{\Delta_{t+1}t^b}\right).
    \end{align}
    Since function $f(t) = \frac{1}{t^{c}}$ is convex for $c\notin (-1,0)$ and $t>0$ one can apply the first order condition of convexity:
    \begin{align*}
        f(t+1) \ge f(t) + f^{'}(t) \ \ \Longrightarrow \ \ \ \frac{1}{(t+1)^{c}} \ge \frac{1}{t^c} - c\frac{1}{t^{c+1}}.
    \end{align*}
    Hence, using $a\le 1$ and $t > (C_1C^2_{\eta})^{\frac{1}{a}}$ for $\Delta_{t+1}$ we have:
    \begin{align*}
        &\Delta_{t+1} = \frac{1}{(t+1)^{b-a}} - \frac{1}{t^{b-a}} + \frac{C_{\eta}}{t^{b}} - \frac{C_1C^2_{\eta}}{t^{a+b}} \ge -(b-a)\frac{1}{t^{b-a+1}} + \frac{C_{\eta}}{t^{b}} - \frac{C_1C^2_{\eta}}{t^{a+b}} \ge \\\nonumber
        &-\frac{b-a}{t^{b-a+1}} + \frac{C_{\eta}}{t^{b}} - \frac{1}{t^{b}} \ge -\frac{b-a}{t^{b}} + \frac{C_{\eta}}{t^{b}} - \frac{1}{t^{b}} = \left(C_{\eta} - 1 - b + a\right)\frac{1}{t^b} > 0.
    \end{align*}
    Moreover,
    \begin{align*}
        \frac{C_2C_{\zeta}}{\Delta_{t+1}t^b} \le \frac{C_2C_{\zeta}}{C_{\eta} - 1 - b + a} \le C_A
    \end{align*}
    Combining these two result in (\ref{A_t_expression_1}) gives:
    \begin{equation*}
        A_{t+1} \le \frac{C_A}{(t+1)^{b-a}} - \Delta_{t+1}\left(C_A - \frac{C_2C_{\zeta}}{\Delta_{t+1}t^b}\right) \le \frac{C_A}{(t+1)^{b-a}}. 
    \end{equation*}
    which proves the induction step.
\end{enumerate}
\end{proof}


\subsection{Asymptotic upper bound on \texorpdfstring{$\mathbb{E}\left[||\boldsymbol{g}(\boldsymbol{\theta}_t) - \boldsymbol{y}_t||^2_2\right]$}{TEXT}.}
Combining the results of Lemmas \ref{Lemma_bound} and \ref{rec_lem} immediately gives the following we establish the asymptotic bound on the accuracy of approximation of $\boldsymbol{g}(\boldsymbol{\theta}_t)$ by $\boldsymbol{y}_t$:

\begin{corollary}\label{cor_1}
Consider Algorithm 1 with step sizes $\alpha_t = \frac{C_{\alpha}}{t^a}$,$\beta_t = \frac{C_{\beta}}{t^b}$ and $K^{(2)}_t = C_{2}t^{e}$ for some constants $C_{\alpha},C_{\beta}, C_{2}, a,b,e > 0$ such that $(2a-2b)\notin (-1,0)$, $0 < b \le 1$. Let $\mathbb{E}\left[\cdot\right]$ be an expectation with respect to all randomness induced in Algorithm 1. Then
\begin{align}
    \mathbb{E}\left[||\boldsymbol{g}(\boldsymbol{\theta}_t) - \boldsymbol{y}_t||^2_2\right] \le  \frac{L^2_gC^2_{\mathcal{D}}}{2}\frac{1}{t^{4a-4b}} + 2C^2_{\mathcal{E}}\frac{1}{t^{b+e}}.
\end{align}
for some constants $C_{\mathcal{D}},C_{\mathcal{E}}> 0$.
\end{corollary}

\begin{proof}
Using $\alpha_t = \frac{C_{\alpha}}{t^a}$, $\beta_t = \frac{C_{\beta}}{t^b}$ in the recurrent inequalities for  $\mathcal{F}^2_t, \mathcal{D}_{t}$ and $\mathbb{E}\left[\left|\left|\boldsymbol{\mathcal{E}}_{t+1}\right|\right|^2_2\right]$ gives:

\begin{enumerate}
    \item For  $\mathcal{F}^2_{t+1}$:
          \begin{align*}
          &\mathcal{F}^2_{t+1} \le \\\nonumber
          &(1- \beta_{t})\mathcal{F}^2_{t} + \frac{4q^2M^2_gM^2_f}{\xi^2}\frac{\alpha^2_{t}}{\beta_{t}} = \left(1 - \frac{C_{\beta}}{t^b}\right)\mathcal{F}^2_{t} + \frac{4q^2M^2_gM^2_fC^2_{\alpha}}{\xi^2C_{\beta}}\frac{1}{t^{2a-b}}.
          \end{align*}
          and applying Lemma \ref{rec_lem} gives 
          \begin{equation}\label{F_t_expression_assymp}
          \mathcal{F}^2_{t} \le \frac{C_{\mathcal{F}}}{t^{2a-2b}} \ \ \ \ \ \text{ where } \ \ \ \ \ C_{\mathcal{F}} = \frac{4q^2M^2_gM^2_fC^2_{\alpha}}{\xi^2C_{\beta}(C_{\beta} - 1 - 2a + 2b)}.
          \end{equation}
    
    \item For $\mathcal{D}_{t+1}$:
        \begin{align*}
        &\mathcal{D}_{t+1} \le \\\nonumber
        &\left(1 - \frac{C_{\beta}}{t^{b}}\right)\mathcal{D}_{t} + \frac{2q^2M^2_gM^2_fC^2_{\alpha}}{\xi^2C_{\beta}}\frac{1}{t^{2a-b}} +  C_{\beta}C_{\mathcal{F}}\frac{1}{t^{2a-b}} = \\\nonumber
        &\left(1 - \frac{C_{\beta}}{t^{b}}\right)\mathcal{D}_{t} + \left[\frac{2q^2M^2_gM^2_fC^2_{\alpha}}{\xi^2C_{\beta}} + C_{\beta}C_{\mathcal{F}}\right] \frac{1}{t^{2a-b}}.
        \end{align*}
        and applying Lemma \ref{rec_lem} gives 
        \begin{equation}\label{D_t_expression_assymp}
        \mathcal{D}_{t} \le \frac{C_{\mathcal{D}}}{t^{2a-2b}} \ \ \ \ \ \text{ where } \ \ \ \ \ C_{\mathcal{D}} = \frac{2q^2M^2_gM^2_fC^2_{\alpha} + \xi^2C^2_{\beta}C_{\mathcal{F}}}{\xi^2C_{\beta}(C_{\beta} - 1 - 2a + 2b)}.
        \end{equation}
    
    \item For $\mathbb{E}\left[\left|\left|\boldsymbol{\mathcal{E}}_{t+1}\right|\right|^2_2\right]$:
    \begin{align*}
        &\mathbb{E}\left[\left|\left|\boldsymbol{\mathcal{E}}_{t+1}\right|\right|^2_2\right] \le (1- \beta_t)^2\mathbb{E}\left[\left|\left|\boldsymbol{\mathcal{E}}_{t}\right|\right|^2_2\right] + \frac{\beta^2_t}{K^{(2)}_t}\sigma^2_3 = \\\nonumber
        &\left(1 - \frac{2C_{\beta}}{t^{b}} + \frac{C^2_{\beta}}{t^{2b}}\right)\mathbb{E}\left[\left|\left|\boldsymbol{\mathcal{E}}_{t}\right|\right|^2_2\right] + \frac{C^2_{\beta}\sigma^2_3}{C_2}\frac{1}{t^{2b + e}} = \\\nonumber
        &\left(1 - \frac{\tilde{C}_{\beta}}{t^{b}} + \frac{\tilde{C}^2_{\beta}}{4t^{2b}}\right)\mathbb{E}\left[\left|\left|\boldsymbol{\mathcal{E}}_{t}\right|\right|^2_2\right] + \frac{C^2_{\beta}\sigma^2_3}{C_2}\frac{1}{t^{2b + e}}.
    \end{align*}
     and applying Lemma \ref{rec_lem} gives:
     \begin{equation}\label{mathbb_e_expres_assymp}
         \mathbb{E}\left[\left|\left|\boldsymbol{\mathcal{E}}_{t}\right|\right|^2_2\right] \le \frac{C_{\mathcal{E}}}{t^{b + e}} \ \ \text{ where } \ \ C_{\mathcal{E}} = \max_{t \le (C^2_{\beta})^{\frac{1}{b}} + 1}\mathbb{E}\left[\left|\left|\boldsymbol{\mathcal{E}}_{t}\right|\right|^2_2\right]t^{b+e} + \frac{C^2_{\beta}\sigma^2_3}{C_2(2C_{\beta} - 1 - b - e)}.
     \end{equation}
\end{enumerate}
Next, combining results (\ref{D_t_expression_assymp}) and (\ref{mathbb_e_expres_assymp}) in (\ref{norm_difference_bound}) gives:
\begin{align*}
    &\mathbb{E}\left[||\boldsymbol{g}(\boldsymbol{\theta}_t) - \boldsymbol{y}_t||^2_2\right] \le \\\nonumber
    &\frac{L^2_g}{2}\mathbb{E}\left[\mathcal{D}^2_{t}\right] + 2\mathbb{E}\left[\left|\left|\boldsymbol{\mathcal{E}}_{t}\right|\right|^2_2\right] \le \frac{L^2_gC^2_{\mathcal{D}}}{2}\frac{1}{t^{4a-4b}} + 2C^2_{\mathcal{E}}\frac{1}{t^{b+e}}.
\end{align*}
\end{proof}


\subsection{Proof of the Main Theorem}
Finally, combing the result of Lemma \ref{lip_smoothnesss_claim} and Corollary \ref{cor_1} we establish the convergence properties of \text{CI-VI} Algorithm:

\begin{app_theorem}\label{main_theorem}
Consider a parameter setup given by: $\alpha_{t} = \sfrac{C_{\alpha}}{t^{\frac{1}{5}}}, \beta_{t} = C_{\beta}, K_{t}^{(1)} = C_{1}t^{\frac{4}{5}}, K_{t}^{(2)} = C_{2}t^{\frac{4}{5}},  K_{t}^{(3)}  = C_{3} t^{\frac{4}{5}}, \gamma_{t}^{(1)}= C_{\gamma}\mu^{t}, \ \gamma_{2}^{(t)}  = 1 - \sfrac{C_{\alpha}}{t^{\frac{2}{5}}}(1 - C_{\gamma}\mu^{t})^{2}$, for some positive constants $C_{\alpha}, C_{\beta}, C_{1}, C_{2}, C_{3}, C_{\gamma}, \mu$ such that $C_{\beta} < 1$ and $\mu \in (0,1)$. For any $\delta \in (0,1)$, Algorithm 1 running gradient-sketching (i.e., using Algorithm 2 to compute gradient products) with a sample-size $d_{t} = \mathcal{O}(1)$ outputs, in expectation, a $\delta$-approximate first-order stationary point $\tilde{\bm{\bm{\theta}}}$ of $\mathcal{L}(\bm{\bm{\theta}})$. That is: $\mathbb{E}_{\textrm{total}}[||\nabla\mathcal{L}(\tilde{\bm{\theta}})||_{2}^{2}] \leq \delta$, with ``total'' representing \emph{all} incurred randomness. Moreover, Algorithm 1 acquires $\tilde{\bm{\theta}}$ with an overall oracle complexity of the order $\mathcal{O}\left(\delta^{-\sfrac{9}{4}}\right)$. 
\end{app_theorem}

\begin{proof}
Let us study the change of the function between two consecutive iterations. Using, that function $\mathcal{L}(\cdot)$ is Lipschitz continuous (see Lemma  \ref{lip_smoothnesss_claim}):
\begin{align}\label{First_diff_equation}
    &\mathcal{L}(\boldsymbol{\theta}_{t+1}) \le \mathcal{L}(\boldsymbol{\theta}_{t}) + \nabla \mathcal{L}(\boldsymbol{\theta}_t)^{\mathsf{T}}(\boldsymbol{\theta}_{t+1} - \boldsymbol{\theta}_t) + \frac{L}{2}||\boldsymbol{\theta}_{t+1} - \boldsymbol{\theta}_t||^{2}_2 = \\\nonumber
    &\mathcal{L}(\boldsymbol{\theta}_t) - \alpha_t\sum_{i=1}^p\left[\nabla\mathcal{L}(\boldsymbol{\theta}_t)\right]_i\frac{[\boldsymbol{m}_t]_i}{\sqrt{[\boldsymbol{v}_t]_i} + \xi} + \frac{L\alpha^2_t}{2}\sum_{i=1}^p\frac{[\boldsymbol{m}_t]^2_i}{\left(\sqrt{[\boldsymbol{v}_t]_i} + \xi\right)^2}.
\end{align}
Now, let us introduce mathematical expectation with respect to all randomness at iteration $t$ given a fixed iterative value $\boldsymbol{\theta}_t$ as $\mathbb{E}_{t}\left[\cdot\right]= \mathbb{E}_{K^{(1)}_t,d_t,K^{(2)}_t}\left[ \cdot \Big| \boldsymbol{\theta}_t\right]$. This expectation taking into account all samplings which is done on iteration $t$ of Algorithm 1. Then, it is easy to see that the following variables will be $t-\text{ measurable}$: $\title{\nabla}\mathcal{L}(\boldsymbol{\theta}_t), \boldsymbol{m}_t, \boldsymbol{v}_t, \boldsymbol{\theta}_{t+1}, \boldsymbol{z}_{t+1}, \boldsymbol{y}_{t+1}$. On the other hand, the following variables will be independent from randomness introduced at iteration $t$: $\nabla\mathcal{L}(\boldsymbol{\theta}_{t-1}), \boldsymbol{m}_{t-1}, \boldsymbol{v}_{t-1}, \boldsymbol{\theta}_{t}, \boldsymbol{z}_{t}, \boldsymbol{y}_{t}$. Hence, taking expectation $\mathbb{E}_{t}\left[\right]$ from the both sides of equation (\ref{First_diff_equation}) gives:
\begin{align}\label{Secon_expectation}
    &\mathbb{E}_t\left[\mathcal{L}(\boldsymbol{\theta}_{t+1})\right] \le \\\nonumber
    &\mathcal{L}(\boldsymbol{\theta}_t) - \alpha_t\sum_{i=1}^p\left[\nabla\mathcal{L}(\boldsymbol{\theta}_t)\right]_i\mathbb{E}_t\left[\frac{[\boldsymbol{m}_t]_i}{\sqrt{[\boldsymbol{v}_t]_i} + \xi}\right] + \frac{L\alpha^2_t}{2}\sum_{i=1}^p\mathbb{E}_t\left[\frac{[\boldsymbol{m}_t]^2_i}{\left(\sqrt{[\boldsymbol{v}_t]_i} + \xi\right)^2}\right].
\end{align}
Now, let us focus on the second term in the above expression:
\begin{align*}
    &\sum_{i=1}^p\left[\nabla\mathcal{L}(\boldsymbol{\theta}_t)\right]_i\mathbb{E}_t\left[\frac{[\boldsymbol{m}_t]_i}{\sqrt{[\boldsymbol{v}_t]_i} + \xi}\right] = \\\nonumber
    &\sum_{i=1}^p\left[\nabla\mathcal{L}(\boldsymbol{\theta}_t)\right]_i\mathbb{E}_t\left[\frac{[\boldsymbol{m}_t]_i}{\sqrt{[\boldsymbol{v}_t]_i} + \xi} - \frac{[\boldsymbol{m}_t]_i}{\sqrt{\gamma^{(2)}_t[\boldsymbol{v}_{t-1}]_i} + \xi} + \frac{[\boldsymbol{m}_t]_i}{\sqrt{\gamma^{(2)}_t[\boldsymbol{v}_{t-1}]_i} + \xi}\right] = \\\nonumber
    &\underbrace{\sum_{i=1}^p\left[\nabla\mathcal{L}(\boldsymbol{\theta}_t)\right]_i\mathbb{E}_t\left[\frac{[\boldsymbol{m}_t]_i}{\sqrt{\gamma^{(2)}_t[\boldsymbol{v}_{t-1}]_i} + \xi}\right]}_{\mathcal{A}} + \underbrace{\sum_{i=1}^p\left[\nabla\mathcal{L}(\boldsymbol{\theta}_t)\right]_i\mathbb{E}_t\left[\frac{[\boldsymbol{m}_t]_i}{\sqrt{[\boldsymbol{v}_{t}]_i} + \xi} - \frac{[\boldsymbol{m}_t]_i}{\sqrt{\gamma^{(2)}_t[\boldsymbol{v}_{t-1}]_i} + \xi}\right]}_{\mathcal{B}}.
\end{align*}
Please notice, from Assumptions 1.3 and 1.4 it follows immediately:
\begin{align*}
    &||\nabla\mathcal{L}(\boldsymbol{\theta}_t)||_2 \le \mathbb{E}_{\boldsymbol{\hat{\epsilon}}}\left[\left|\left|\nabla \boldsymbol{g}_{\boldsymbol{\hat{\epsilon}}}(\boldsymbol{\theta}_t)^{\mathsf{T}} \right|\right|_2\right]\mathbb{E}_{\nu}\left[\left|\left|\nabla f_{\nu}(\mathbb{E}_{\boldsymbol{\hat{\epsilon}}}[\boldsymbol{g}_{\boldsymbol{\hat{\epsilon}}}(\boldsymbol{\theta}_t)]) \right|\right|_2\right]\le M_gM_f\\\nonumber
    &||\overline{\nabla \mathcal{L}(\boldsymbol{\theta}_{t})}||_2 = \left|\left|\frac{n}{d_tK^{(1)}_t}\sum_{a=1}^{K^{(1)}_t}\sum_{j=1}^{d_t}\nabla\boldsymbol{g}_{{\boldsymbol{\hat{\epsilon}}_{t_a}}}^{\mathsf{T}}(\boldsymbol{\theta}_t)(:,i_j)\nabla f_{\nu_{t_a}}(\boldsymbol{y}_t)(i_j)\right|\right|_2 \le \\\nonumber
    &\frac{n}{d_tK^{(1)}_t}\sum_{a=1}^{K^{(1)}_t}\sum_{j=1}^{d_t}\left|\left|\nabla\boldsymbol{g}_{{\boldsymbol{\hat{\epsilon}}_{t_a}}}^{\mathsf{T}}(\boldsymbol{\theta}_t)(:,i_j)\right|\right|_2\left|\left|\nabla f_{\nu_{t_a}}(\boldsymbol{y}_t)(i_j)\right|\right|_2\le nM_gM_f.
\end{align*}
and applying induction we have:
\begin{align}\label{bound_expressions}
    &||\boldsymbol{m}_t||_2 \le \gamma^{(1)}_tnM_gM_f + \left(1 - \gamma^{(1)}_t\right)nM_gM_f = nM_gM_f,\ \ \ \ \ \ \forall t\\\nonumber
    &||\boldsymbol{v}_t||_2 \le \gamma^{(2)}_tn^2M^2_gM^2_f + \left(1 - \gamma^{(2)}_t\right)n^2M^2_gM^2_f = n^2M^2_gM^2_f,\ \ \ \ \ \ \forall t
\end{align}
Now, let us apply (\ref{bound_expressions}) for the  expression $\mathcal{A}$:
\begin{align*}
    &\mathcal{A} = \sum_{i=1}^p\left[\nabla\mathcal{L}(\boldsymbol{\theta}_t)\right]_i\frac{\mathbb{E}_t[\boldsymbol{m}_t]_i}{\sqrt{\gamma^{(2)}_t[\boldsymbol{v}_{t-1}]_i} + \xi} = \sum_{i=1}^p\left[\nabla\mathcal{L}(\boldsymbol{\theta}_t)\right]_i\frac{\mathbb{E}_t\left[\gamma^{(1)}_t[\boldsymbol{m}_{t-1}]_i + \left(1 - \gamma^{(1)}_t\right)[\overline{\nabla \mathcal{L}(\boldsymbol{\theta}_{t})}]_i\right]}{\sqrt{\gamma^{(2)}_t[\boldsymbol{v}_{t-1}]_i} + \xi} =\\\nonumber
    &\sum_{i=1}^p\left[\nabla\mathcal{L}(\boldsymbol{\theta}_t)\right]_i\frac{\gamma^{(1)}_t[\boldsymbol{m}_{t-1}]_i + \left(1 - \gamma^{(1)}_t\right)\mathbb{E}_t\left[[\overline{\nabla \mathcal{L}(\boldsymbol{\theta}_{t})}]_i\right]}{\sqrt{\gamma^{(2)}_t[\boldsymbol{v}_{t-1}]_i} + \xi} = \gamma^{(1)}_t\nabla\mathcal{L}(\boldsymbol{\theta}_t)^{\mathsf{T}}\frac{\boldsymbol{m}_{t-1}}{\sqrt{\gamma^{(2)}_t\boldsymbol{v}_{t-1}}+ \xi} + \\\nonumber
    &\left(1 - \gamma^{(1)}_t\right)\nabla\mathcal{L}(\boldsymbol{\theta}_t)^{\mathsf{T}}\frac{\mathbb{E}_t\left[\overline{\nabla \mathcal{L}(\boldsymbol{\theta}_{t})}\right]}{\sqrt{\gamma^{(2)}_t\boldsymbol{v}_{t-1}}+ \xi}.
\end{align*}
Adding and subtracting the term $\nabla\mathcal{L}(\boldsymbol{\theta}_t)$ gives:
\begin{align*}
    &\mathcal{A} = \gamma^{(1)}_t\nabla\mathcal{L}(\boldsymbol{\theta}_t)^{\mathsf{T}}\frac{\boldsymbol{m}_{t-1}}{\sqrt{\gamma^{(2)}_t\boldsymbol{v}_{t-1}}+ \xi} + \left(1 - \gamma^{(1)}_t\right)\nabla\mathcal{L}(\boldsymbol{\theta}_t)^{\mathsf{T}}\frac{\mathbb{E}_t\left[\nabla\mathcal{L}(\boldsymbol{\theta}_t) - \nabla\mathcal{L}(\boldsymbol{\theta}_t) + \overline{\nabla \mathcal{L}(\boldsymbol{\theta}_{t})}\right]}{\sqrt{\gamma^{(2)}_t\boldsymbol{v}_{t-1}}+ \xi} = \\\nonumber
    &\gamma^{(1)}_t\nabla\mathcal{L}(\boldsymbol{\theta}_t)^{\mathsf{T}}\frac{\boldsymbol{m}_{t-1}}{\sqrt{\gamma^{(2)}_t\boldsymbol{v}_{t-1}}+ \xi} + \left(1 - \gamma^{(1)}_t\right)\sum_{i=1}^{p}\frac{[\nabla\mathcal{L}(\boldsymbol{\theta}_{t})]^2_i}{\sqrt{\gamma^{(2)}_t[\boldsymbol{v}_{t-1}]_i} + \xi} - \left(1 - \gamma^{(1)}_t\right)\times\\\nonumber
    &\nabla\mathcal{L}(\boldsymbol{\theta}_t)^{\mathsf{T}}\frac{\mathbb{E}_t\left[\nabla\mathcal{L}(\boldsymbol{\theta}_t) - \overline{\nabla \mathcal{L}(\boldsymbol{\theta}_{t})} \right]}{\sqrt{\gamma^{(2)}_t\boldsymbol{v}_{t-1}}+ \xi} = \gamma^{(1)}_t\nabla\mathcal{L}(\boldsymbol{\theta}_t)^{\mathsf{T}}\frac{\boldsymbol{m}_{t-1}}{\sqrt{\gamma^{(2)}_t\boldsymbol{v}_{t-1}}+ \xi} + \left(1 - \gamma^{(1)}_t\right)\sum_{i=1}^{p}\frac{[\nabla\mathcal{L}(\boldsymbol{\theta}_{t})]^2_i}{\sqrt{\gamma^{(2)}_t[\boldsymbol{v}_{t-1}]_i} + \xi} - \\\nonumber
    &\left(1 - \gamma^{(1)}_t\right)\underbrace{\frac{\mathbb{E}_t\left[\nabla\mathcal{L}(\boldsymbol{\theta}_t)^{\mathsf{T}}\left(\nabla\mathcal{L}(\boldsymbol{\theta}_t) - \overline{\nabla \mathcal{L}(\boldsymbol{\theta}_{t})} \right)\right]}{\sqrt{\gamma^{(2)}_t\boldsymbol{v}_{t-1}}+ \xi}}_{\mathcal{A}1}.
\end{align*}
Let us study the expression $\mathcal{A}1$ more carefully. Let us denote $\Hat{\nabla}\mathcal{L}(\boldsymbol{\theta}_t) = \frac{n}{K^{(1)}_td_t}\sum_{a=1}^{K^{(1)}_t}\sum_{j=1}^{d_t}\nabla\boldsymbol{g}_{{\boldsymbol{\hat{\epsilon}}_{t_a}}}^{\mathsf{T}}(\boldsymbol{\theta}_t)(:,i_j)\nabla f_{\nu_{t_a}}(\boldsymbol{g}(\boldsymbol{\theta}_t))(i_j)$, then:
\begin{align*}
    &\mathcal{A}1 = \frac{\mathbb{E}_t\left[\nabla\mathcal{L}(\boldsymbol{\theta}_t)^{\mathsf{T}}\left(\nabla\mathcal{L}(\boldsymbol{\theta}_t) - \Hat{\nabla}\mathcal{L}(\boldsymbol{\theta}_t) \right)\right]}{\sqrt{\gamma^{(2)}_t\boldsymbol{v}_{t-1}}+ \xi} + \frac{\mathbb{E}_t\left[\nabla\mathcal{L}(\boldsymbol{\theta}_t)^{\mathsf{T}}\left(\Hat{\nabla}\mathcal{L}(\boldsymbol{\theta}_t) - \overline{\nabla \mathcal{L}(\boldsymbol{\theta}_{t})} \right)\right]}{\sqrt{\gamma^{(2)}_t\boldsymbol{v}_{t-1}}+ \xi}.
\end{align*}
Let us denote $\boldsymbol{D}^{(j)}_t\in\mathbb{R}^{n\times n}$ be a diagonal matrix defined as follows:
\begin{equation*}
    \left[\boldsymbol{D}^{(j)}_t\right]_{kk} = \left\{
        \begin{array}{ll}
            1 & \text{ if } k = i_j,  \\
            0 & \text{ otherwise }
        \end{array}
    \right.
\end{equation*}
then notice that $\boldsymbol{D}^{(j)}_t$ is random matrix such that $\mathbb{E}_t\left[\boldsymbol{D}^{(j)}_t\right] = \frac{1}{n}\boldsymbol{I}_{n\times n}$. Using Assumption \textbf{\text{2}} and independentness of sampling in Gradient Sketching Algorithm from samplings in \text{CI-VI} Algorithm  we have
\begin{align*}
    &\mathbb{E}_t\left[\Hat{\nabla}\mathcal{L}(\boldsymbol{\theta}_t)\right] = \mathbb{E}_t\left[\frac{n}{K^{(1)}_td_t}\sum_{a=1}^{K^{(1)}_t}\sum_{j=1}^{d_t}\nabla\boldsymbol{g}_{{\boldsymbol{\hat{\epsilon}}_{t_a}}}^{\mathsf{T}}(\boldsymbol{\theta}_t)(:,i_j)\nabla f_{\nu_{t_a}}(\boldsymbol{g}(\boldsymbol{\theta}_t))(i_j)\right] = \\\nonumber
    &\frac{n}{K^{(1)}_td_t}\sum_{a=1}^{K^{(1)}_t}\sum_{j=1}^{d_t}\mathbb{E}_t\left[\nabla\boldsymbol{g}_{{\boldsymbol{\hat{\epsilon}}_{t_a}}}^{\mathsf{T}}(\boldsymbol{\theta}_t)(:,i_j)\nabla f_{\nu_{t_a}}(\boldsymbol{g}(\boldsymbol{\theta}_t))(i_j)\right] = \frac{n}{K^{(1)}_td_t}\sum_{a=1}^{K^{(1)}_t}\sum_{j=1}^{d_t}\mathbb{E}_t\left[\nabla\boldsymbol{g}_{{\boldsymbol{\hat{\epsilon}}_{t_a}}}^{\mathsf{T}}(\boldsymbol{\theta}_t)\boldsymbol{D}^{(j)}_t\nabla f_{\nu_{t_a}}(\boldsymbol{g}(\boldsymbol{\theta}_t))\right] \\\nonumber
    &=\frac{n}{K^{(1)}_td_t}\sum_{a=1}^{K^{(1)}_t}\sum_{j=1}^{d_t}\mathbb{E}_t\left[\nabla\boldsymbol{g}_{{\boldsymbol{\hat{\epsilon}}_{t_a}}}^{\mathsf{T}}(\boldsymbol{\theta}_t)\sum_{k=1}^q[\boldsymbol{D}^{(j)}_t]_{kk}\boldsymbol{e}_k\left[\nabla f_{\nu_{t_a}}(\boldsymbol{g}(\boldsymbol{\theta}_t))\right]_k\right] = \\\nonumber
    &\frac{n}{K^{(1)}_td_t}\sum_{a=1}^{K^{(1)}_t}\sum_{j=1}^{d_t}\sum_{k=1}^n\mathbb{E}_t\left[[\boldsymbol{D}^{(j)}_t]_{kk}\right]\mathbb{E}_t\left[\nabla\boldsymbol{g}_{{\boldsymbol{\hat{\epsilon}}_{t_a}}}^{\mathsf{T}}(\boldsymbol{\theta}_t)\boldsymbol{e}_k\left[\nabla f_{\nu_{t_a}}(\boldsymbol{g}(\boldsymbol{\theta}_t))\right]_k\right] = \\\nonumber
    &\frac{1}{K^{(1)}_td_t}\sum_{a=1}^{K^{(1)}_t}\sum_{j=1}^{d_t}\mathbb{E}_t\left[\nabla\boldsymbol{g}_{{\boldsymbol{\hat{\epsilon}}_{t_a}}}^{\mathsf{T}}(\boldsymbol{\theta}_t)\nabla f_{\nu_{t_a}}(\boldsymbol{g}(\boldsymbol{\theta}_t))\right] = \frac{1}{d_t}\sum_{j=1}^{d_t}\nabla\mathcal{L}(\boldsymbol{\theta}_t) =  \nabla\mathcal{L}(\boldsymbol{\theta}_t).
\end{align*}
where we used and  $\mathbb{E}_t\left[[\boldsymbol{D}^{(j)}_t]_{kk}\right] = \frac{1}{n}$. Hence, for the term $\mathcal{A}1$ we have:
\begin{align*}
    &\mathcal{A}1 =  \frac{\mathbb{E}_t\left[\nabla\mathcal{L}(\boldsymbol{\theta}_t)^{\mathsf{T}}\left(\Hat{\nabla}\mathcal{L}(\boldsymbol{\theta}_t) - \overline{\nabla \mathcal{L}(\boldsymbol{\theta}_{t})} \right)\right]}{\sqrt{\gamma^{(2)}_t\boldsymbol{v}_{t-1}}+ \xi} = \\\nonumber
    &\frac{\mathbb{E}_t\left[\nabla\mathcal{L}(\boldsymbol{\theta}_t)^{\mathsf{T}}\left(\frac{n}{d_tK^{(1)}_t}\sum_{a=1}^{K^{(1)}_t}\sum_{j=1}^{d_t}\left[\nabla\boldsymbol{g}_{{\boldsymbol{\hat{\epsilon}}_{t_a}}}^{\mathsf{T}}(\boldsymbol{\theta}_t)\boldsymbol{D}^{(j)}_t\nabla f_{\nu_{t_a}}(\boldsymbol{g}(\boldsymbol{\theta}_t)) - \nabla\boldsymbol{g}_{{\boldsymbol{\hat{\epsilon}}_{t_a}}}^{\mathsf{T}}(\boldsymbol{\theta}_t)\boldsymbol{D}^{(j)}_t\nabla f_{\nu_{t_a}}(\boldsymbol{y}_t)\right] \right)\right]}{\sqrt{\gamma^{(2)}_t\boldsymbol{v}_{t-1}}+ \xi}\\\nonumber
    &=\frac{\mathbb{E}_t\left[\nabla\mathcal{L}(\boldsymbol{\theta}_t)^{\mathsf{T}}\left(\frac{n}{d_tK^{(1)}_t}\sum_{a=1}^{K^{(1)}_t}\sum_{j=1}^{d_t}\nabla\boldsymbol{g}_{{\boldsymbol{\hat{\epsilon}}_{t_a}}}^{\mathsf{T}}(\boldsymbol{\theta}_t)\boldsymbol{D}^{(j)}_t\left[\nabla f_{\nu_{t_a}}(\boldsymbol{g}(\boldsymbol{\theta}_t)) - \nabla f_{\nu_{t_a}}(\boldsymbol{y}_t)\right] \right)\right]}{\sqrt{\gamma^{(2)}_t\boldsymbol{v}_{t-1}}+ \xi}.
\end{align*}
Hence, using $ab\le \frac{1}{2}a^2 + \frac{1}{2}b^2$ and $||\boldsymbol{a}_1 + \ldots + \boldsymbol{a}_r||^2_2 \le r(||\boldsymbol{a}_1||^2_2 + \ldots + ||\boldsymbol{a}_r||^2_2)$:
\begin{align*}
    &\mathcal{A}1 = \sum_{i=1}^p\mathbb{E}_t\left[\frac{\left[\nabla\mathcal{L}(\boldsymbol{\theta}_t)\right]_i }{\sqrt{\sqrt{\gamma^{(2)}_t\left[\boldsymbol{v}_{t-1}\right]_i}+ \xi}}\frac{\frac{n}{d_tK^{(1)}_t}\sum_{a=1}^{K^{(1)}_t}\sum_{j=1}^{d_t}\left[\nabla\boldsymbol{g}_{{\boldsymbol{\hat{\epsilon}}_{t_a}}}^{\mathsf{T}}(\boldsymbol{\theta}_t)\boldsymbol{D}^{(j)}_t\left[\nabla f_{\nu_{t_a}}(\boldsymbol{g}(\boldsymbol{\theta}_t)) - \nabla f_{\nu_{t_a}}(\boldsymbol{y}_t)\right]\right]_i}{\sqrt{\sqrt{\gamma^{(2)}_t\left[\boldsymbol{v}_{t-1}\right]_i}+ \xi}}\right] \le \\\nonumber
    &\mathbb{E}_t\left[\frac{1}{2}\sum_{i=1}^p\frac{\left[\nabla\mathcal{L}(\boldsymbol{\theta}_t)\right]^2_i }{\sqrt{\gamma^{(2)}_t\left[\boldsymbol{v}_{t-1}\right]_i}+ \xi}\right]  + \frac{1}{2\xi}\mathbb{E}_t\left[\left|\left|\frac{n}{d_tK^{(1)}_t}\sum_{a=1}^{K^{(1)}_t}\sum_{j=1}^{d_t}\nabla\boldsymbol{g}_{{\boldsymbol{\hat{\epsilon}}_{t_a}}}^{\mathsf{T}}(\boldsymbol{\theta}_t)\boldsymbol{D}^{(j)}_t\left[\nabla f_{\nu_{t_a}}(\boldsymbol{g}(\boldsymbol{\theta}_t)) - \nabla f_{\nu_{t_a}}(\boldsymbol{y}_t)\right]\right|\right|^2_2\right] \le \\\nonumber
    &\mathbb{E}_t\left[\frac{1}{2}\sum_{i=1}^p\frac{\left[\nabla\mathcal{L}(\boldsymbol{\theta}_t)\right]^2_i }{\sqrt{\gamma^{(2)}_t\left[\boldsymbol{v}_{t-1}\right]_i}+ \xi}\right]  + \frac{n^2}{2\xi d_tK^{(1)}_t}\mathbb{E}_t\left[\sum_{a=1}^{K^{(1)}_t}\sum_{j=1}^{d_t}\left|\left|\nabla\boldsymbol{g}_{{\boldsymbol{\hat{\epsilon}}_{t_a}}}^{\mathsf{T}}(\boldsymbol{\theta}_t)\boldsymbol{D}^{(j)}_t\left[\nabla f_{\nu_{t_a}}(\boldsymbol{g}(\boldsymbol{\theta}_t)) - \nabla f_{\nu_{t_a}}(\boldsymbol{y}_t)\right]\right|\right|^2_2\right] \le \\\nonumber
    &\mathbb{E}_t\left[\frac{1}{2}\sum_{i=1}^p\frac{\left[\nabla\mathcal{L}(\boldsymbol{\theta}_t)\right]^2_i }{\sqrt{\gamma^{(2)}_t\left[\boldsymbol{v}_{t-1}\right]_i}+ \xi}\right]  + \frac{n^2}{2\xi d_tK^{(1)}_t}\mathbb{E}_t\left[\sum_{a=1}^{K^{(1)}_t}\sum_{j=1}^{d_t}\left|\left|\nabla\boldsymbol{g}_{{\boldsymbol{\hat{\epsilon}}_{t_a}}}^{\mathsf{T}}(\boldsymbol{\theta}_t)\right|\right|^2_2\left|\left|\nabla f_{\nu_{t_a}}(\boldsymbol{g}(\boldsymbol{\theta}_t)) - \nabla f_{\nu_{t_a}}(\boldsymbol{y}_t)\right|\right|^2_2\right] \le \\\nonumber
    &\mathbb{E}_t\left[\frac{1}{2}\sum_{i=1}^p\frac{\left[\nabla\mathcal{L}(\boldsymbol{\theta}_t)\right]^2_i }{\sqrt{\gamma^{(2)}_t\left[\boldsymbol{v}_{t-1}\right]_i}+ \xi}\right]  + \frac{n^2M^2_g}{2\xi d_tK^{(1)}_t}\mathbb{E}_t\left[\sum_{a=1}^{K^{(1)}_t}\sum_{j=1}^{d_t}\left|\left|\nabla f_{\nu_{t_a}}(\boldsymbol{g}(\boldsymbol{\theta}_t)) - \nabla f_{\nu_{t_a}}(\boldsymbol{y}_t)\right|\right|^2_2\right] \le \\\nonumber
    &\frac{1}{2}\sum_{i=1}^p\frac{\left[\nabla\mathcal{L}(\boldsymbol{\theta}_t)\right]^2_i }{\sqrt{\gamma^{(2)}_t\left[\boldsymbol{v}_{t-1}\right]_i}+ \xi}  + \frac{n^2M^2_gL^2_f}{2\xi }\mathbb{E}_t\left[\left|\left|\boldsymbol{g}(\boldsymbol{\theta}_t) - \boldsymbol{y}_t\right|\right|^2_2\right].
\end{align*}
where we used $||\boldsymbol{D}^{(j)}_t||^2_2 = 1$ and Assumption \textbf{\text{1}} to bound $||\nabla\boldsymbol{g}_{{\boldsymbol{\hat{\epsilon}}_{t_a}}}^{\mathsf{T}}(\boldsymbol{\theta}_t)||^2_2\le M^2_g$. Hence, we arrive at the following expression for $\mathcal{A}$:
\begin{align}\label{mathcalAexpression}
    &-\mathcal{A} = \\\nonumber
    &-\gamma^{(1)}_t\nabla\mathcal{L}(\boldsymbol{\theta}_t)^{\mathsf{T}}\frac{\boldsymbol{m}_{t-1}}{\sqrt{\gamma^{(2)}_t\boldsymbol{v}_{t-1}}+ \xi} - \left(1 - \gamma^{(1)}_t\right)\sum_{i=1}^{p}\frac{[\nabla\mathcal{L}(\boldsymbol{\theta}_{t})]^2_i}{\sqrt{\gamma^{(2)}_t[\boldsymbol{v}_{t-1}]_i} + \xi} + \left(1 - \gamma^{(1)}_t\right)\mathcal{A}1\\\nonumber
    &-\gamma^{(1)}_t\nabla\mathcal{L}(\boldsymbol{\theta}_t)^{\mathsf{T}}\frac{\boldsymbol{m}_{t-1}}{\sqrt{\gamma^{(2)}_t\boldsymbol{v}_{t-1}}+ \xi} - \left(1 - \gamma^{(1)}_t\right)\sum_{i=1}^{p}\frac{[\nabla\mathcal{L}(\boldsymbol{\theta}_{t})]^2_i}{\sqrt{\gamma^{(2)}_t[\boldsymbol{v}_{t-1}]_i} + \xi} + \left(1 - \gamma^{(1)}_t\right)\times\\\nonumber
    &\left[\frac{1}{2}\sum_{i=1}^p\frac{\left[\nabla\mathcal{L}(\boldsymbol{\theta}_t)\right]^2_i }{\sqrt{\gamma^{(2)}_t\left[\boldsymbol{v}_{t-1}\right]_i}+ \xi}  + \frac{n^2M^2_gL^2_f}{2\xi }\mathbb{E}_t\left[\left|\left|\boldsymbol{g}(\boldsymbol{\theta}_t) - \boldsymbol{y}_t\right|\right|^2_2\right]\right] =\\\nonumber
    &-\gamma^{(1)}_t\nabla\mathcal{L}(\boldsymbol{\theta}_t)^{\mathsf{T}}\frac{\boldsymbol{m}_{t-1}}{\sqrt{\gamma^{(2)}_t\boldsymbol{v}_{t-1}}+ \xi} - \frac{\left(1 - \gamma^{(1)}_t\right)}{2}\sum_{i=1}^{p}\frac{[\nabla\mathcal{L}(\boldsymbol{\theta}_{t})]^2_i}{\sqrt{\gamma^{(2)}_t[\boldsymbol{v}_{t-1}]_i} + \xi}  
    + \frac{\left(1 - \gamma^{(1)}_t\right)n^2}{2\xi}M^2_gL^2_f\left|\left|\boldsymbol{g}(\boldsymbol{\theta}_t) - \boldsymbol{y}_t \right|\right|^2_2.
\end{align}
Now, let us consider more carefully the second term:
\begin{align*}
    &-\mathcal{B} = -\sum_{i=1}^p\left[\nabla\mathcal{L}(\boldsymbol{\theta}_t)\right]_i\mathbb{E}_t\left[\frac{[\boldsymbol{m}_t]_i}{\sqrt{[\boldsymbol{v}_{t}]_i} + \xi} - \frac{[\boldsymbol{m}_t]_i}{\sqrt{\gamma^{(2)}_t[\boldsymbol{v}_{t-1}]_i} + \xi}\right] \le \\\nonumber &\sum_{i=1}^p\left|\left[\nabla\mathcal{L}(\boldsymbol{\theta}_t)\right]_i\right|\mathbb{E}_t\left[\underbrace{\left|\frac{[\boldsymbol{m}_t]_i}{\sqrt{[\boldsymbol{v}_{t}]_i} + \xi} - \frac{[\boldsymbol{m}_t]_i}{\sqrt{\gamma^{(2)}_t[\boldsymbol{v}_{t-1}]_i} + \xi}\right|}_{\mathcal{B}1}\right]
\end{align*}
For the expression $\mathcal{B}1$ we have:
\begin{align*}
    &\mathcal{B}1 = \left|\frac{[\boldsymbol{m}_t]_i}{\sqrt{[\boldsymbol{v}_{t}]_i} + \xi} - \frac{[\boldsymbol{m}_t]_i}{\sqrt{\gamma^{(2)}_t[\boldsymbol{v}_{t-1}]_i} + \xi}\right| = |[\boldsymbol{m}_t]_i|\times\left|\frac{1}{\sqrt{[\boldsymbol{v}_t]_i} + \xi} - \frac{1}{\sqrt{\gamma^{(2)}_t[\boldsymbol{v}_{t-1}]_i} + \xi}\right| =\\\nonumber
    &\frac{|[\boldsymbol{m}_t]_i|}{(\sqrt{[\boldsymbol{v}_t]_i}+ \xi)(\sqrt{\gamma^{(2)}_{t}[\boldsymbol{v}_{t-1}]_i} + \xi)}\times\left|\frac{(1-\gamma^{(2)}_t)[\overline{\nabla \mathcal{L}(\boldsymbol{\theta}_{t})}]^2_i}{\sqrt{[\boldsymbol{v}_{t}]_i} + \sqrt{\gamma^{(2)}_t[\boldsymbol{v}_{t-1}]_i}}\right| = \frac{(1-\gamma^{(2)}_t)|[\boldsymbol{m}_t]_i|}{(\sqrt{[\boldsymbol{v}_t]_i}+ \xi)(\sqrt{\gamma^{(2)}_{t}[\boldsymbol{v}_{t-1}]_i} + \xi)}\times\\\nonumber
    &\frac{[\overline{\nabla \mathcal{L}(\boldsymbol{\theta}_{t})}]^2_i}{\sqrt{\gamma^{(2)}_{t}[\boldsymbol{v}_{t-1}]_i + (1 - \gamma^{(2)}_{t})[\overline{\nabla \mathcal{L}(\boldsymbol{\theta}_{t})}]^2_i} + \sqrt{\gamma^{(2)}_t[\boldsymbol{v}_{t-1}]_i}}.
\end{align*}
Using update for $\boldsymbol{m}_t$ and $|ab|\le \frac{1-\gamma^{(1)}_t}{\gamma^{(1)}_t}a^2 + \frac{\gamma^{(1)}_t}{1 - \gamma^{(1)}_t}b^2$ we have:
\begin{align*}
    &\mathcal{B}1 \le \frac{\sqrt{1 - \gamma^{(2)}_t}|[\boldsymbol{m}_t]_i||[\overline{\nabla \mathcal{L}(\boldsymbol{\theta}_{t})}]_i|}{\xi\left(\sqrt{\gamma^{(2)}_t[\boldsymbol{v}_{t-1}]_i} +\xi\right)}\le \sqrt{1 - \gamma^{(2)}_t}\frac{|\gamma^{(1)}_t[\boldsymbol{m}_{t-1}]_i| + (1 - \gamma^{(1)}_t)|[\overline{\nabla \mathcal{L}(\boldsymbol{\theta}_{t})}]_i|}{\xi\left(\sqrt{\gamma^{(2)}_t[\boldsymbol{v}_{t-1}]_i} +\xi\right)}|[\overline{\nabla \mathcal{L}(\boldsymbol{\theta}_{t})}]_i|=\\\nonumber
    &\frac{(1 - \gamma^{(1)}_t)\sqrt{1 - \gamma^{(2)}_t}}{\xi\left(\sqrt{\gamma^{(2)}_t[\boldsymbol{v}_{t-1}]_i} +\xi\right)}[\overline{\nabla \mathcal{L}(\boldsymbol{\theta}_{t})}]^2_i + \frac{\gamma^{(1)}_t\sqrt{1 - \gamma^{(2)}_t}}{\xi\left(\sqrt{\gamma^{(2)}_t[\boldsymbol{v}_{t-1}]_i} +\xi\right)}|[\overline{\nabla \mathcal{L}(\boldsymbol{\theta}_{t})}]_i||[\boldsymbol{m}_{t-1}]_i| \le 2\frac{(1 - \gamma^{(1)}_t)\sqrt{1 - \gamma^{(2)}_t}}{\xi\left(\sqrt{\gamma^{(2)}_t[\boldsymbol{v}_{t-1}]_i} +\xi\right)}\times\\\nonumber&[\overline{\nabla \mathcal{L}(\boldsymbol{\theta}_{t})}]^2_i +\frac{\left(\gamma^{(1)}_t\right)^2\sqrt{1 - \gamma^{(2)}_t}}{\xi(1 - \gamma^{(1)}_t)\left(\sqrt{\gamma^{(2)}_t[\boldsymbol{v}_{t-1}]_i} +\xi\right)}[\boldsymbol{m}_{t-1}]^2_i
\end{align*}
 Therefore, for the term $\mathcal{B}$ we have:
\begin{align}
    &-\mathcal{B} \le \\\nonumber
    &\sum_{i=1}^p|[\nabla \mathcal{L}(\boldsymbol{\theta}_t)]_i|\mathbb{E}_t\left[2\frac{(1 - \gamma^{(1)}_t)\sqrt{1 - \gamma^{(2)}_t}}{\xi\left(\sqrt{\gamma^{(2)}_t[\boldsymbol{v}_{t-1}]_i} +\xi\right)}[\overline{\nabla \mathcal{L}(\boldsymbol{\theta}_{t})}]^2_i + \frac{\left(\gamma^{(1)}_t\right)^2\sqrt{1 - \gamma^{(2)}_t}}{\xi(1 - \gamma^{(1)}_t)\left(\sqrt{\gamma^{(2)}_t[\boldsymbol{v}_{t-1}]_i} +\xi\right)}[\boldsymbol{m}_{t-1}]^2_i\right] = \\\nonumber
    &\frac{\left(\gamma^{(1)}_t\right)^2\sqrt{1 - \gamma^{(2)}_t}}{\xi(1 - \gamma^{(1)}_t)}\sum_{i=1}^{p}\frac{|[\nabla\mathcal{L}(\boldsymbol{\theta}_t)]_i|[\boldsymbol{m}_{t-1}]^2_i}{\left(\sqrt{\gamma^{(2)}_t[\boldsymbol{v}_{t-1}]_i} + +\xi\right)} + \frac{2(1 - \gamma^{(1)}_t)\sqrt{1 - \gamma^{(2)}_t}}{\xi}\sum_{i=1}^p\frac{|[\nabla\mathcal{L}(\boldsymbol{\theta}_t)]_i|\mathbb{E}_t\left[[\overline{\nabla \mathcal{L}(\boldsymbol{\theta}_{t})}]^2_i\right]}{\left(\sqrt{\gamma^{(2)}_t[\boldsymbol{v}_{t-1}]_i} +\xi\right)}. 
\end{align}
Using that $[\boldsymbol{v}_{t}]_i \ge 0$ for any $t$ and applying (\ref{bound_expressions}) we get immediately:
\begin{align*}
    &-\mathcal{B}\le\\\nonumber
    &\frac{\left(\gamma^{(1)}_t\right)^2\sqrt{1 - \gamma^{(2)}_t}}{\xi(1 - \gamma^{(1)}_t)}\sum_{i=1}^{p}\frac{|[\nabla\mathcal{L}(\boldsymbol{\theta}_t)]_i|[\boldsymbol{m}_{t-1}]^2_i}{\left(\sqrt{\gamma^{(2)}_t[\boldsymbol{v}_{t-1}]_i} + \xi\right)} + \frac{2(1 - \gamma^{(1)}_t)\sqrt{1 - \gamma^{(2)}_t}}{\xi}\sum_{i=1}^p\frac{|[\nabla\mathcal{L}(\boldsymbol{\theta}_t)]_i|\mathbb{E}_t\left[[\overline{\nabla \mathcal{L}(\boldsymbol{\theta}_{t})}]^2_i\right]}{\left(\sqrt{\gamma^{(2)}_t[\boldsymbol{v}_{t-1}]_i} +\xi\right)} \le \\\nonumber
    &\frac{nM_gM_f\left(\gamma^{(1)}_t\right)^2\sqrt{1 - \gamma^{(2)}_t}}{\xi (1 - \gamma^{(1)}_t)}\sum_{i=1}^{p}\frac{|[\nabla\mathcal{L}(\boldsymbol{\theta}_t)]_i[\boldsymbol{m}_{t-1}]_i|}{\sqrt{\gamma^{(2)}_{t}[\boldsymbol{v}_{t-1}]_i} + \xi} + \frac{2M_gM_f(1 - \gamma^{(1)}_t)\sqrt{1 - \gamma^{(2)}_t}}{\xi}\sum_{i=1}^p\frac{\mathbb{E}_t\left[[\overline{\nabla \mathcal{L}(\boldsymbol{\theta}_{t})}]^2_i\right]}{\sqrt{\gamma^{(2)}_{t}[\boldsymbol{v}_{t-1}]_i} + \xi}.
\end{align*}
Finally, we have the bound for the second term in the expression (\ref{Secon_expectation}):
\begin{align}\label{second_term_expression}
    &-\sum_{i=1}^p\left[\nabla\mathcal{L}(\boldsymbol{\theta}_t)\right]_i\mathbb{E}_t\left[\frac{[\boldsymbol{m}_t]_i}{\sqrt{[\boldsymbol{v}_t]_i} + \xi}\right] =  -\mathcal{A} - \mathcal{B}\le\\\nonumber
    &-\gamma^{(1)}_t\nabla\mathcal{L}(\boldsymbol{\theta}_t)^{\mathsf{T}}\frac{\boldsymbol{m}_{t-1}}{\sqrt{\gamma^{(2)}_t\boldsymbol{v}_{t-1}}+ \xi} - \frac{\left(1 - \gamma^{(1)}_t\right)}{2}\sum_{i=1}^{p}\frac{[\nabla\mathcal{L}(\boldsymbol{\theta}_{t})]^2_i}{\sqrt{\gamma^{(2)}_t[\boldsymbol{v}_{t-1}]_i} + \xi}  
    + \frac{\left(1 - \gamma^{(1)}_t\right)n^2M^2_gL^2_f}{2\xi}\left|\left|\boldsymbol{g}(\boldsymbol{\theta}_t) - \boldsymbol{y}_t \right|\right|^2_2\\\nonumber
    &+\frac{nM_gM_f\left(\gamma^{(1)}_t\right)^2\sqrt{1 - \gamma^{(2)}_t}}{\xi (1 - \gamma^{(1)}_t)}\sum_{i=1}^{p}\frac{|[\nabla\mathcal{L}(\boldsymbol{\theta}_t)]_i[\boldsymbol{m}_{t-1}]_i|}{\sqrt{\gamma^{(2)}_{t}[\boldsymbol{v}_{t-1}]_i} + \xi} + \frac{2M_gM_f(1 - \gamma^{(1)}_t)\sqrt{1 - \gamma^{(2)}_t}}{\xi}\sum_{i=1}^p\frac{\mathbb{E}_t\left[[\overline{\nabla \mathcal{L}(\boldsymbol{\theta}_{t})}]^2_i\right]}{\sqrt{\gamma^{(2)}_{t}[\boldsymbol{v}_{t-1}]_i} + \xi}.
\end{align}
Now, we can focus on the third term in the expression (\ref{Secon_expectation}). Applying that $[\boldsymbol{v}_t]_i \ge \gamma^{(2)}_{t}[\boldsymbol{v}_{t-1}]_i$ we have:
\begin{align*}
    &\sum_{i=1}^p\mathbb{E}_t\left[\frac{[\boldsymbol{m}_t]^2_i}{\left(\sqrt{[\boldsymbol{v}_t]_i} + \xi\right)^2}\right] \le \sum_{i=1}^p\mathbb{E}_t\left[\frac{[\boldsymbol{m}_t]^2_i}{\left(\sqrt{\gamma^{(2)}_t[\boldsymbol{v}_{t-1}]_i} + \xi\right)^2}\right] = \sum_{i=1}^p\frac{\mathbb{E}_t\left[[\boldsymbol{m}_t]^2_i\right]}{\left(\sqrt{\gamma^{(2)}_t[\boldsymbol{v}_{t-1}]_i} + \xi\right)^2} = \\\nonumber
    &\sum_{i=1}^p\frac{\mathbb{E}_t\left[[\gamma^{(1)}_t\boldsymbol{m}_{t-1} + \left(1 - \gamma^{(1)}_t\right)\overline{\nabla \mathcal{L}(\boldsymbol{\theta}_{t})}]^2_i\right]}{\left(\sqrt{\gamma^{(2)}_t[\boldsymbol{v}_{t-1}]_i} + \xi\right)^2} \le \sum_{i=1}^p\frac{2\left(\gamma^{(1)}_t\right)^2\mathbb{E}_t\left[[\boldsymbol{m}_{t-1}]^2_i\right] + 2\left(1 - \gamma^{(1)}_t\right)^2\mathbb{E}_t\left[[\overline{\nabla \mathcal{L}(\boldsymbol{\theta}_{t})}]^2_i\right]}{\left(\sqrt{\gamma^{(2)}_t[\boldsymbol{v}_{t-1}]_i} + \xi\right)^2}\\\nonumber
    &=2\left(\gamma^{(1)}_t\right)^2\sum_{i=1}^{p}\frac{[\boldsymbol{m}_{t-1}]^2_i}{\left(\sqrt{\gamma^{(2)}_t[\boldsymbol{v}_{t-1}]_i} + \xi\right)^2} + 2\left(1 - \gamma^{(1)}_t\right)^2\sum_{i=1}^p\frac{\mathbb{E}_t\left[[\overline{\nabla \mathcal{L}(\boldsymbol{\theta}_{t})}]^2_i\right]}{\left(\sqrt{\gamma^{(2)}_t[\boldsymbol{v}_{t-1}]_i} + \xi\right)^2} \le \\\nonumber
    &2\left(\gamma^{(1)}_t\right)^2\sum_{i=1}^{p}\frac{[\boldsymbol{m}_{t-1}]^2_i}{\left(\sqrt{\gamma^{(2)}_t[\boldsymbol{v}_{t-1}]_i} + \xi\right)^2} + 2\frac{\left(1 - \gamma^{(1)}_t\right)^2}{\xi}\sum_{i=1}^p\frac{\mathbb{E}_t\left[[\overline{\nabla \mathcal{L}(\boldsymbol{\theta}_{t})}]^2_i\right]}{\sqrt{\gamma^{(2)}_t[\boldsymbol{v}_{t-1}]_i} + \xi}\le\\\nonumber
    &\frac{2\left(\gamma^{(1)}_t\right)^2}{\xi}\sum_{i=1}^{p}\frac{[\boldsymbol{m}_{t-1}]^2_i}{\sqrt{\gamma^{(2)}_t[\boldsymbol{v}_{t-1}]_i} + \xi} + 2\frac{\left(1 - \gamma^{(1)}_t\right)^2}{\xi}\sum_{i=1}^p\frac{\mathbb{E}_t\left[[\overline{\nabla \mathcal{L}(\boldsymbol{\theta}_{t})}]^2_i\right]}{\sqrt{\gamma^{(2)}_t[\boldsymbol{v}_{t-1}]_i} + \xi}.
\end{align*}
Hence, denoting $\hat{M} = M_gM_f$ we arrive at the following expression for change in the function between two consecutive iterations:
\begin{align*}
    &\mathbb{E}_t\left[\mathcal{L}(\boldsymbol{\theta}_{t+1})\right] - \mathcal{L}(\boldsymbol{\theta}_t) \le -\alpha_t\gamma^{(1)}_t\nabla\mathcal{L}(\boldsymbol{\theta}_t)^{\mathsf{T}}\frac{\boldsymbol{m}_{t-1}}{\sqrt{\gamma^{(2)}_t\boldsymbol{v}_{t-1}}+ \xi} -\\\nonumber
    &\frac{\alpha_t\left(1 - \gamma^{(1)}_t\right)}{2}\sum_{i=1}^{p}\frac{[\nabla\mathcal{L}(\boldsymbol{\theta}_{t})]^2_i}{\sqrt{\gamma^{(2)}_t[\boldsymbol{v}_{t-1}]_i} + \xi}  
    + \frac{\alpha_t\left(1 - \gamma^{(1)}_t\right)n^2M^2_gL^2_f}{2\xi}\left|\left|\boldsymbol{g}(\boldsymbol{\theta}_t) - \boldsymbol{y}_t \right|\right|^2_2 + \\\nonumber
    &\frac{\alpha_tn\hat{M}\left(\gamma^{(1)}_t\right)^2\sqrt{1 - \gamma^{(2)}_t}}{\xi (1 - \gamma^{(1)}_t)}\sum_{i=1}^{p}\frac{|[\nabla\mathcal{L}(\boldsymbol{\theta}_t)]_i[\boldsymbol{m}_{t-1}]_i|}{\sqrt{\gamma^{(2)}_{t}[\boldsymbol{v}_{t-1}]_i} + \xi} + \frac{2\alpha_t\hat{M}(1 - \gamma^{(1)}_t)\sqrt{1 - \gamma^{(2)}_t}}{\xi}\sum_{i=1}^p\mathbb{E}_t\left[\frac{[\overline{\nabla \mathcal{L}(\boldsymbol{\theta}_{t})}]^2_i}{\sqrt{\gamma^{(2)}_{t}[\boldsymbol{v}_{t-1}]_i} + \xi}\right]+ \\\nonumber
    &\frac{L\alpha_t^2\left(\gamma^{(1)}_t\right)^2}{\xi}\sum_{i=1}^{p}\frac{[\boldsymbol{m}_{t-1}]^2_i}{\sqrt{\gamma^{(2)}_t[\boldsymbol{v}_{t-1}]_i} + \xi} + \frac{L\alpha_t^2\left(1 - \gamma^{(1)}_t\right)^2}{\xi}\sum_{i=1}^p\frac{\mathbb{E}_t\left[[\overline{\nabla \mathcal{L}(\boldsymbol{\theta}_{t})}]^2_i\right]}{\sqrt{\gamma^{(2)}_t[\boldsymbol{v}_{t-1}]_i} + \xi}.
\end{align*}
After re-grouping terms in the above expression we have:
\begin{align}\label{new_change_function}
    &\mathbb{E}_t\left[\mathcal{L}(\boldsymbol{\theta}_{t+1})\right] - \mathcal{L}(\boldsymbol{\theta}_t) \le \\\nonumber
    &-\frac{\alpha_t\left(1 - \gamma^{(1)}_t\right)}{2}\sum_{i=1}^{p}\frac{[\nabla\mathcal{L}(\boldsymbol{\theta}_{t})]^2_i}{\sqrt{\gamma^{(2)}_t[\boldsymbol{v}_{t-1}]_i} + \xi} + \frac{\alpha_t\left(1 - \gamma^{(1)}_t\right)n^2M^2_gL^2_f}{2\xi}\left|\left|\boldsymbol{g}(\boldsymbol{\theta}_t) - \boldsymbol{y}_t \right|\right|^2_2 + \sum_{i=1}^p\frac{\mathbb{E}_t\left[[\overline{\nabla \mathcal{L}(\boldsymbol{\theta}_{t})}]^2_i\right]}{\sqrt{\gamma^{(2)}_t[\boldsymbol{v}_{t-1}]_i} + \xi}\times\\\nonumber
    &\left[\frac{L\alpha^2_t\left(1 - \gamma^{(1)}_t\right)^2}{\xi} + \frac{2\alpha_t\hat{M}(1 - \gamma^{(1)}_t)\sqrt{1 - \gamma^{(2)}_t}}{\xi}\right]  -\alpha_t\gamma^{(1)}_t\nabla\mathcal{L}(\boldsymbol{\theta}_t)^{\mathsf{T}}\frac{\boldsymbol{m}_{t-1}}{\sqrt{\gamma^{(2)}_t\boldsymbol{v}_{t-1}}+ \xi} + \\\nonumber &\frac{L\alpha^2_t\left(\gamma^{(1)}_t\right)^2}{\xi}\sum_{i=1}^{p}\frac{[\boldsymbol{m}_{t-1}]^2_i}{\sqrt{\gamma^{(2)}_t[\boldsymbol{v}_{t-1}]_i} + \xi} + \frac{\alpha_tn\hat{M}\left(\gamma^{(1)}_t\right)^2\sqrt{1 - \gamma^{(2)}_t}}{\xi (1 - \gamma^{(1)}_t)}\sum_{i=1}^{p}\frac{|[\nabla\mathcal{L}(\boldsymbol{\theta}_t)]_i[\boldsymbol{m}_{t-1}]_i|}{\sqrt{\gamma^{(2)}_{t}[\boldsymbol{v}_{t-1}]_i} + \xi}.
\end{align}
Let us focus on the term $\mathbb{E}_t\left[[\overline{\nabla \mathcal{L}(\boldsymbol{\theta}_{t})}]^2_i\right]$. Using Assumption \textbf{\text{2}} and independentness of sampling in Gradient Sketching Algorithm from samplings of $\mathcal{FOO}_f(\boldsymbol{y}_t, K^{(1)}_t), \mathcal{FOO}_{\boldsymbol{g}}(\boldsymbol{\theta}_t, K^{(1)}_t)$ we have:
\begin{align*}
    &\mathbb{E}_t\left[\overline{\nabla \mathcal{L}(\boldsymbol{\theta}_{t})}\right] = \mathbb{E}_t\left[\frac{n}{K^{(1)}_td_t}\sum_{a=1}^{K^{(1)}_t}\sum_{j=1}^{d_t}\nabla\boldsymbol{g}_{{\boldsymbol{\hat{\epsilon}}_{t_a}}}^{\mathsf{T}}(\boldsymbol{\theta}_t)(:,i_j)\nabla f_{\nu_{t_a}}(\boldsymbol{y}_t)(i_j)\right] = \\\nonumber
    &\frac{n}{d_tK^{(1)}_t}\sum_{a=1}^{K^{(1)}_t}\sum_{j=1}^{d_t}\mathbb{E}_t\left[\nabla\boldsymbol{g}_{{\boldsymbol{\hat{\epsilon}}_{t_a}}}^{\mathsf{T}}(\boldsymbol{\theta}_t)\boldsymbol{D}^{(j)}_t\nabla f_{\nu_{t_a}}(\boldsymbol{y}_t)\right]= \frac{n}{d_tK^{(1)}_t}\sum_{a=1}^{K^{(1)}_t}\sum_{j=1}^{d_t}\sum_{k=1}^n\mathbb{E}_t\left[[\boldsymbol{D}^{(j)}_t]_{kk}\nabla\boldsymbol{g}_{{\boldsymbol{\hat{\epsilon}}_{t_a}}}^{\mathsf{T}}(\boldsymbol{\theta}_t)\boldsymbol{e}_k\left[\nabla f_{\nu_{t_a}}(\boldsymbol{y}_t)\right]_k\right]\\\nonumber
    &= \frac{1}{d_tK^{(1)}_t}\sum_{a=1}^{K^{(1)}_t}\sum_{j=1}^{d_t}\mathbb{E}_t\left[\nabla\boldsymbol{g}_{{\boldsymbol{\hat{\epsilon}}_{t_a}}}^{\mathsf{T}}(\boldsymbol{\theta}_t)\nabla f_{\nu_{t_a}}(\boldsymbol{y}_t)\right] = \nabla\boldsymbol{g}^{\mathsf{T}}(\boldsymbol{\theta}_t)\nabla f(\boldsymbol{y}_t).
\end{align*}
Hence, for the variance of the above expression we have:
\begin{align}\label{variance_est_1}
    &\mathbb{E}_t\left[\left([\overline{\nabla \mathcal{L}(\boldsymbol{\theta}_{t})}]_i - \left[\nabla \boldsymbol{g}(\boldsymbol{\theta}_t)^{\mathsf{T}}\nabla f(\boldsymbol{y}_t)\right]_i\right)^2\right] = \mathbb{E}_t\left[[\overline{\nabla \mathcal{L}(\boldsymbol{\theta}_{t})}]^2_i\right] - \left[\nabla \boldsymbol{g}(\boldsymbol{\theta}_t)^{\mathsf{T}}\nabla f(\boldsymbol{y}_t)\right]^2_i.
\end{align}
On the other hand, using $\mathbb{E}_t\left[q\nabla\boldsymbol{g}_{{\boldsymbol{\hat{\epsilon}}_{t_a}}}^{\mathsf{T}}(\boldsymbol{\theta}_t)\boldsymbol{D}^{(j)}_t\nabla f_{\nu_{t_a}}(\boldsymbol{y}_t)\right] = \nabla\boldsymbol{g}^{\mathsf{T}}(\boldsymbol{\theta}_t)\nabla f(\boldsymbol{y}_t) $:  
\begin{align*}
    &\mathbb{E}_t\left[\left([\overline{\nabla \mathcal{L}(\boldsymbol{\theta}_{t})}]_i - [\nabla\boldsymbol{g}^{\mathsf{T}}(\boldsymbol{\theta}_t)\nabla f(\boldsymbol{y}_t)]_i \right)^2\right] = \\\nonumber
    &\mathbb{E}_t\left[\left[\frac{n}{d_tK^{(1)}_t}\sum_{a=1}^{K^{(1)}_t}\sum_{j=1}^{d_t}\left[\nabla\boldsymbol{g}_{{\boldsymbol{\hat{\epsilon}}_{t_a}}}^{\mathsf{T}}(\boldsymbol{\theta}_t)\boldsymbol{D}^{(j)}_t\nabla f_{\nu_{t_a}}(\boldsymbol{y}_t)\right]_i - [\nabla\boldsymbol{g}^{\mathsf{T}}(\boldsymbol{\theta}_t)\nabla f(\boldsymbol{y}_t)]_i \right]^2\right]=\\\nonumber
    &\mathbb{E}_t\left[\left[\frac{1}{K^{(1)}_td_t}\sum_{a=1}^{K^{(1)}_t}\sum_{j=1}^{d_t}\left[n\nabla\boldsymbol{g}_{{\boldsymbol{\hat{\epsilon}}_{t_a}}}^{\mathsf{T}}(\boldsymbol{\theta}_t)\boldsymbol{D}^{(j)}_t\nabla f_{\nu_{t_a}}(\boldsymbol{y}_t) -\nabla\boldsymbol{g}^{\mathsf{T}}(\boldsymbol{\theta}_t)\nabla f(\boldsymbol{y}_t) \right]_i \right]^2\right] = \\\nonumber
    &\frac{1}{\left(K^{(1)}_td_t\right)^2}\sum_{a=1}^{K^{(1)}_t}\sum_{j=1}^{d_t}\mathbb{E}_t\left[\left[n\nabla\boldsymbol{g}_{{\boldsymbol{\hat{\epsilon}}_{t_a}}}^{\mathsf{T}}(\boldsymbol{\theta}_t)\boldsymbol{D}^{(j)}_t\nabla f_{\nu_{t_a}}(\boldsymbol{y}_t) -\nabla\boldsymbol{g}^{\mathsf{T}}(\boldsymbol{\theta}_t)\nabla f(\boldsymbol{y}_t) \right]^2_i \right] + \frac{1}{\left(d_tK^{(1)}_t\right)^2}\sum_{a=1}^{K^{(1)}_t}\sum_{j=1}^{d_t}\sum_{\mathcal{L}=1, \mathcal{L}\neq j}^{d_t}\\\nonumber
    &\mathbb{E}_t\left[\left[n\nabla\boldsymbol{g}_{{\boldsymbol{\hat{\epsilon}}_{t_a}}}^{\mathsf{T}}(\boldsymbol{\theta}_t)\boldsymbol{D}^{(j)}_t\nabla f_{\nu_{t_a}}(\boldsymbol{y}_t) -\nabla\boldsymbol{g}^{\mathsf{T}}(\boldsymbol{\theta}_t)\nabla f(\boldsymbol{y}_t) \right]_i\left[n\nabla\boldsymbol{g}_{{\boldsymbol{\hat{\epsilon}}_{t_a}}}^{\mathsf{T}}(\boldsymbol{\theta}_t)\boldsymbol{D}^{(j)}_t\nabla f_{\nu_{t_a}}(\boldsymbol{y}_t) -\nabla\boldsymbol{g}^{\mathsf{T}}(\boldsymbol{\theta}_t)\nabla f(\boldsymbol{y}_t) \right]_i\right] + \\\nonumber
    &\frac{\sum_{i=1}^{K^{(1)}_t}\sum_{b=1,b\neq a}^{K^{(1)}_t}}{\left(d_tK^{(1)}_t\right)^2}\sum_{j=1}^{d_t}\mathbb{E}_t\left[\left[n\nabla\boldsymbol{g}_{{\boldsymbol{\hat{\epsilon}}_{t_a}}}^{\mathsf{T}}(\boldsymbol{\theta}_t)\boldsymbol{D}^{(j)}_t\nabla f_{\nu_{t_a}}(\boldsymbol{y}_t) -\nabla\boldsymbol{g}^{\mathsf{T}}(\boldsymbol{\theta}_t)\nabla f(\boldsymbol{y}_t) \right]_i\right]\times\\\nonumber&\sum_{\mathcal{L}=1}^{d_t}\mathbb{E}_t\left[\left[n\nabla\boldsymbol{g}_{{\boldsymbol{\hat{\epsilon}}_{t_a}}}^{\mathsf{T}}(\boldsymbol{\theta}_t)\boldsymbol{D}^{(j)}_t\nabla f_{\nu_{t_a}}(\boldsymbol{y}_t) -\nabla\boldsymbol{g}^{\mathsf{T}}(\boldsymbol{\theta}_t)\nabla f(\boldsymbol{y}_t) \right]_i\right] = \\\nonumber
    &\frac{1}{\left(K^{(1)}_td_t\right)^2}\sum_{a=1}^{K^{(1)}_t}\sum_{j=1}^{d_t}\mathbb{E}_t\left[\left[n\nabla\boldsymbol{g}_{{\boldsymbol{\hat{\epsilon}}_{t_a}}}^{\mathsf{T}}(\boldsymbol{\theta}_t)\boldsymbol{D}^{(j)}_t\nabla f_{\nu_{t_a}}(\boldsymbol{y}_t) -\nabla\boldsymbol{g}^{\mathsf{T}}(\boldsymbol{\theta}_t)\nabla f(\boldsymbol{y}_t) \right]^2_i \right] + \frac{1}{\left(d_tK^{(1)}_t\right)^2}\sum_{a=1}^{K^{(1)}_t}\sum_{j=1}^{d_t}\sum_{\mathcal{L}=1, \mathcal{L}\neq j}^{d_t}\\\nonumber
    &\mathbb{E}_t\left[\left[n\nabla\boldsymbol{g}_{{\boldsymbol{\hat{\epsilon}}_{t_a}}}^{\mathsf{T}}(\boldsymbol{\theta}_t)\boldsymbol{D}^{(j)}_t\nabla f_{\nu_{t_a}}(\boldsymbol{y}_t) -\nabla\boldsymbol{g}^{\mathsf{T}}(\boldsymbol{\theta}_t)\nabla f(\boldsymbol{y}_t) \right]_i\left[n\nabla\boldsymbol{g}_{{\boldsymbol{\hat{\epsilon}}_{t_a}}}^{\mathsf{T}}(\boldsymbol{\theta}_t)\boldsymbol{D}^{(j)}_t\nabla f_{\nu_{t_a}}(\boldsymbol{y}_t) -\nabla\boldsymbol{g}^{\mathsf{T}}(\boldsymbol{\theta}_t)\nabla f(\boldsymbol{y}_t) \right]_i\right].
\end{align*}
Using Cauchy–Bunyakovsky–Schwarz inequality, we have:
\begin{align*}
    &\mathbb{E}_t\left[\left[n\nabla\boldsymbol{g}_{{\boldsymbol{\hat{\epsilon}}_{t_a}}}^{\mathsf{T}}(\boldsymbol{\theta}_t)\boldsymbol{D}^{(j)}_t\nabla f_{\nu_{t_a}}(\boldsymbol{y}_t) -\nabla\boldsymbol{g}^{\mathsf{T}}(\boldsymbol{\theta}_t)\nabla f(\boldsymbol{y}_t) \right]_i\left[n\nabla\boldsymbol{g}_{{\boldsymbol{\hat{\epsilon}}_{t_a}}}^{\mathsf{T}}(\boldsymbol{\theta}_t)\boldsymbol{D}^{(j)}_t\nabla f_{\nu_{t_a}}(\boldsymbol{y}_t) -\nabla\boldsymbol{g}^{\mathsf{T}}(\boldsymbol{\theta}_t)\nabla f(\boldsymbol{y}_t) \right]_i\right] \le\\\nonumber
    &\sqrt{\mathbb{E}_t\left[\left[n\nabla\boldsymbol{g}_{{\boldsymbol{\hat{\epsilon}}_{t_a}}}^{\mathsf{T}}(\boldsymbol{\theta}_t)\boldsymbol{D}^{(j)}_t\nabla f_{\nu_{t_a}}(\boldsymbol{y}_t) -\nabla\boldsymbol{g}^{\mathsf{T}}(\boldsymbol{\theta}_t)\nabla f(\boldsymbol{y}_t) \right]^2_i\right]}\times\\\nonumber&\sqrt{\mathbb{E}_t\left[\left[n\nabla\boldsymbol{g}_{{\boldsymbol{\hat{\epsilon}}_{t_a}}}^{\mathsf{T}}(\boldsymbol{\theta}_t)\boldsymbol{D}^{(j)}_t\nabla f_{\nu_{t_a}}(\boldsymbol{y}_t) -\nabla\boldsymbol{g}^{\mathsf{T}}(\boldsymbol{\theta}_t)\nabla f(\boldsymbol{y}_t) \right]^2_i\right]}.
\end{align*}
Let us study the term $\mathbb{E}_t\left[\left[n\nabla\boldsymbol{g}_{{\boldsymbol{\hat{\epsilon}}_{t_a}}}^{\mathsf{T}}(\boldsymbol{\theta}_t)\boldsymbol{D}^{(j)}_t\nabla f_{\nu_{t_a}}(\boldsymbol{y}_t) -\nabla\boldsymbol{g}^{\mathsf{T}}(\boldsymbol{\theta}_t)\nabla f(\boldsymbol{y}_t) \right]^2_i\right]$: 
\begin{align*}
    &\mathbb{E}_t\left[\left[n\nabla\boldsymbol{g}_{{\boldsymbol{\hat{\epsilon}}_{t_a}}}^{\mathsf{T}}(\boldsymbol{\theta}_t)\boldsymbol{D}^{(j)}_t\nabla f_{\nu_{t_a}}(\boldsymbol{y}_t) - \nabla\boldsymbol{g}^{\mathsf{T}}(\boldsymbol{\theta}_t)\nabla f(\boldsymbol{y}_t) \right]^2_i\right] = \\\nonumber
    &\mathbb{E}_t\left[\left[n\nabla\boldsymbol{g}_{{\boldsymbol{\hat{\epsilon}}_{t_a}}}^{\mathsf{T}}(\boldsymbol{\theta}_t)\boldsymbol{D}^{(j)}_t\nabla f_{\nu_{t_a}}(\boldsymbol{y}_t) - \nabla\boldsymbol{g}^{\mathsf{T}}(\boldsymbol{\theta}_t)\nabla f_{\nu_{t_a}}(\boldsymbol{y}_t) + \nabla\boldsymbol{g}^{\mathsf{T}}(\boldsymbol{\theta}_t)\nabla f_{\nu_{t_a}}(\boldsymbol{y}_t) -\nabla\boldsymbol{g}^{\mathsf{T}}(\boldsymbol{\theta}_t)\nabla f(\boldsymbol{y}_t) \right]^2_i \right]\\\nonumber
    &2\mathbb{E}_t\left[\left[n\nabla\boldsymbol{g}_{{\boldsymbol{\hat{\epsilon}}_{t_a}}}^{\mathsf{T}}(\boldsymbol{\theta}_t)\boldsymbol{D}^{(j)}_t\nabla f_{\nu_{t_a}}(\boldsymbol{y}_t) - \nabla\boldsymbol{g}^{\mathsf{T}}(\boldsymbol{\theta}_t)\nabla f_{\nu_{t_a}}(\boldsymbol{y}_t)\right]^2_i\right] + 2\mathbb{E}_t\left[\left[\nabla\boldsymbol{g}^{\mathsf{T}}(\boldsymbol{\theta}_t)\nabla f_{\nu_{t_a}}(\boldsymbol{y}_t) -\nabla\boldsymbol{g}^{\mathsf{T}}(\boldsymbol{\theta}_t)\nabla f(\boldsymbol{y}_t)\right]^2_i\right].
\end{align*}
The first summation in the above expression can be bounded using Assumption \textbf{\text{1}} in the following way (using $[\boldsymbol{a}]^2_i \le ||\boldsymbol{a}||^2_2$,  $||\boldsymbol{A}+\boldsymbol{B}||^2_2 \le 2||\boldsymbol{A}||^2_2 + 2||\boldsymbol{B}||^2_2$ and $||\boldsymbol{A}||^2_2 \le ||\boldsymbol{A}||^2_F = \text{Tr}(\boldsymbol{A}^{\mathsf{T}}\boldsymbol{A})$):
\begin{align*}
    &\mathbb{E}_t\left[\left[n\nabla\boldsymbol{g}_{{\boldsymbol{\hat{\epsilon}}_{t_a}}}^{\mathsf{T}}(\boldsymbol{\theta}_t)\boldsymbol{D}^{(j)}_t\nabla f_{\nu_{t_a}}(\boldsymbol{y}_t) - \nabla\boldsymbol{g}^{\mathsf{T}}(\boldsymbol{\theta}_t)\nabla f_{\nu_{t_a}}(\boldsymbol{y}_t)\right]^2_i\right] \le \\\nonumber
    &\mathbb{E}_t\left[\left|\left|n\nabla\boldsymbol{g}_{{\boldsymbol{\hat{\epsilon}}_{t_a}}}^{\mathsf{T}}(\boldsymbol{\theta}_t)\boldsymbol{D}^{(j)}_t\nabla f_{\nu_{t_a}}(\boldsymbol{y}_t) - \nabla\boldsymbol{g}^{\mathsf{T}}(\boldsymbol{\theta}_t)\nabla f_{\nu_{t_a}}(\boldsymbol{y}_t)\right|\right|^2_2\right]\le \\\nonumber
    &\mathbb{E}_t\left[\left|\left|n\nabla\boldsymbol{g}_{{\boldsymbol{\hat{\epsilon}}_{t_a}}}^{\mathsf{T}}(\boldsymbol{\theta}_t)\boldsymbol{D}^{(j)}_t - \nabla\boldsymbol{g}^{\mathsf{T}}(\boldsymbol{\theta}_t)\right|\right|^2_2\left|\left|\nabla f_{\nu_{t_a}}(\boldsymbol{y}_t)\right|\right|^2_2\right] \le M^2_f\mathbb{E}_t\left[\left|\left|n\nabla\boldsymbol{g}_{{\boldsymbol{\hat{\epsilon}}_{t_a}}}^{\mathsf{T}}(\boldsymbol{\theta}_t)\boldsymbol{D}^{(j)}_t - \nabla\boldsymbol{g}^{\mathsf{T}}(\boldsymbol{\theta}_t)\right|\right|^2_2\right] = \\\nonumber
    &M^2_f\mathbb{E}_t\left[\left|\left|n\nabla\boldsymbol{g}_{{\boldsymbol{\hat{\epsilon}}_{t_a}}}^{\mathsf{T}}(\boldsymbol{\theta}_t)\boldsymbol{D}^{(j)}_t - n\nabla\boldsymbol{g}^{\mathsf{T}}(\boldsymbol{\theta}_t)\boldsymbol{D}^{(j)}_t + n\nabla\boldsymbol{g}^{\mathsf{T}}(\boldsymbol{\theta}_t)\boldsymbol{D}_t^{(j)} -  \nabla\boldsymbol{g}^{\mathsf{T}}(\boldsymbol{\theta}_t)\right|\right|^2_2\right] \le \\\nonumber
    &2M^2_f\mathbb{E}_t\left[\left|\left|n\nabla\boldsymbol{g}_{{\boldsymbol{\hat{\epsilon}}_{t_a}}}^{\mathsf{T}}(\boldsymbol{\theta}_t)\boldsymbol{D}^{(j)}_t - n\nabla\boldsymbol{g}^{\mathsf{T}}(\boldsymbol{\theta}_t)\boldsymbol{D}_t^{(t)} \right|\right|^2_2\right] + 2M^2_f\mathbb{E}_t\left[\left|\left| n\nabla\boldsymbol{g}^{\mathsf{T}}(\boldsymbol{\theta}_t)\boldsymbol{D}^{(j)}_t -  \nabla\boldsymbol{g}^{\mathsf{T}}(\boldsymbol{\theta}_t)\right|\right|^2_2\right]\le \\\nonumber
    &2n^2M^2_f\mathbb{E}_t\left[\left|\left|\nabla\boldsymbol{g}_{{\boldsymbol{\hat{\epsilon}}_{t_a}}}^{\mathsf{T}}(\boldsymbol{\theta}_t) - \nabla\boldsymbol{g}^{\mathsf{T}}(\boldsymbol{\theta}_t) \right|\right|^2_2\right] + 2M^2_fM^2_g\mathbb{E}_t\left[\left|\left| n\boldsymbol{D}^{(j)}_t -  \boldsymbol{I}\right|\right|^2_2\right]\le \\\nonumber
    &2n^2M^2_f\sigma^2_2 + 2M^2_fM^2_g\mathbb{E}_t\left[\text{Tr}\left[ n\boldsymbol{D}^{(j)}_t -  \boldsymbol{I}\right]^2\right] = 2n^2M^2_f\sigma^2_2 + 2M^2_fM^2_gn(n-1).  
\end{align*}
where in the last step we used $\mathbb{E}_t[\boldsymbol{D}^{(j)}_t] = \frac{1}{n}\boldsymbol{I}_{n\times n}$. For the second term, we similarly can write:
\begin{align*}
    &\mathbb{E}_t\left[\left[\nabla\boldsymbol{g}^{\mathsf{T}}(\boldsymbol{\theta}_t)\nabla f_{\nu_{t_a}}(\boldsymbol{y}_t) -\nabla\boldsymbol{g}^{\mathsf{T}}(\boldsymbol{\theta}_t)\nabla f(\boldsymbol{y}_t)\right]^2_i\right] \le \mathbb{E}_t\left[\left|\left|\nabla\boldsymbol{g}^{\mathsf{T}}(\boldsymbol{\theta}_t)\nabla f_{\nu_{t_a}}(\boldsymbol{y}_t) -\nabla\boldsymbol{g}^{\mathsf{T}}(\boldsymbol{\theta}_t)\nabla f(\boldsymbol{y}_t)\right|\right|^2_2\right] \le \\\nonumber
    &M^2_g\mathbb{E}_t\left[\left|\left|\nabla f_{\nu_{t_a}}(\boldsymbol{y}_t) -\nabla f(\boldsymbol{y}_t)\right|\right|^2_2\right] \le  M^2_g\sigma^2_1.
\end{align*}
Hence, for the variance expression we immediately have:
\begin{align*}
    &\mathbb{E}_t\left[\left[n\nabla\boldsymbol{g}_{{\boldsymbol{\hat{\epsilon}}_{t_a}}}^{\mathsf{T}}(\boldsymbol{\theta}_t)\boldsymbol{D}^{(j)}_t\nabla f_{\nu_{t_a}}(\boldsymbol{y}_t) - \nabla\boldsymbol{g}^{\mathsf{T}}(\boldsymbol{\theta}_t)\nabla f(\boldsymbol{y}_t) \right]^2_i\right] \le 4n^2M^2_f\sigma^2_2 + 4M^2_fM^2_gn(n-1) + 2M^2_g\sigma^2_1.
\end{align*}
Hence,
\begin{align*}
    &\mathbb{E}_t\left[\left([\overline{\nabla \mathcal{L}(\boldsymbol{\theta}_{t})}]_i - [\nabla\boldsymbol{g}^{\mathsf{T}}(\boldsymbol{\theta}_t)\nabla f(\boldsymbol{y}_t)]_i \right)^2\right] \le \\\nonumber
    &\frac{1}{\left(K^{(1)}_td_t\right)^2}\sum_{a=1}^{K^{(1)}_t}\sum_{j=1}^{d_t}\mathbb{E}_t\left[\left[n\nabla\boldsymbol{g}_{{\boldsymbol{\hat{\epsilon}}_{t_a}}}^{\mathsf{T}}(\boldsymbol{\theta}_t)\boldsymbol{D}^{(j)}_t\nabla f_{\nu_{t_a}}(\boldsymbol{y}_t) -\nabla\boldsymbol{g}^{\mathsf{T}}(\boldsymbol{\theta}_t)\nabla f(\boldsymbol{y}_t) \right]^2_i \right] + \frac{1}{\left(d_tK^{(1)}_t\right)^2}\sum_{a=1}^{K^{(1)}_t}\sum_{j=1}^{d_t}\sum_{\mathcal{L}=1, \mathcal{L}\neq j}^{d_t}\\\nonumber
    &\sqrt{\mathbb{E}_t\left[\left[n\nabla\boldsymbol{g}_{{\boldsymbol{\hat{\epsilon}}_{t_a}}}^{\mathsf{T}}(\boldsymbol{\theta}_t)\boldsymbol{D}^{(j)}_t\nabla f_{\nu_{t_a}}(\boldsymbol{y}_t) -\nabla\boldsymbol{g}^{\mathsf{T}}(\boldsymbol{\theta}_t)\nabla f(\boldsymbol{y}_t) \right]^2_i\right]\mathbb{E}_t\left[\left[n\nabla\boldsymbol{g}_{{\boldsymbol{\hat{\epsilon}}_{t_a}}}^{\mathsf{T}}(\boldsymbol{\theta}_t)\boldsymbol{D}^{(j)}_t\nabla f_{\nu_{t_a}}(\boldsymbol{y}_t) -\nabla\boldsymbol{g}^{\mathsf{T}}(\boldsymbol{\theta}_t)\nabla f(\boldsymbol{y}_t) \right]^2_i\right]}\\\nonumber
    &\le \frac{4n^2M^2_f\sigma^2_2 + 4M^2_fM^2_gn(n-1) + 2M^2_g\sigma^2_1}{K^{(1)}_t}.
\end{align*}
Combining this result with (\ref{variance_est_1}) gives: 
\begin{equation*}
    \mathbb{E}_t\left[[\overline{\nabla \mathcal{L}(\boldsymbol{\theta}_{t})}]^2_i\right]  \le \frac{4n^2M^2_f\sigma^2_2 + 4M^2_fM^2_gn(n-1) + 2M^2_g\sigma^2_1}{K^{(1)}_t} + \left[\nabla \boldsymbol{g}(\boldsymbol{\theta}_t)^{\mathsf{T}}\nabla f(\boldsymbol{y}_t)\right]^2_i
\end{equation*}
Moreover,
\begin{align*}
    &\left[\nabla \boldsymbol{g}(\boldsymbol{\theta}_t)^{\mathsf{T}}\nabla f(\boldsymbol{y}_t)\right]^2_i = \left[\nabla \boldsymbol{g}(\boldsymbol{\theta}_t)^{\mathsf{T}}\nabla f(\boldsymbol{y}_t) - \nabla \boldsymbol{g}(\boldsymbol{\theta}_t)^{\mathsf{T}}\nabla f(\boldsymbol{g}(\boldsymbol{\theta}_t)) + \nabla \boldsymbol{g}(\boldsymbol{\theta}_t)^{\mathsf{T}}\nabla f(\boldsymbol{g}(\boldsymbol{\theta}_t))\right]^2_i = \\\nonumber
    &\left[\nabla \boldsymbol{g}(\boldsymbol{\theta}_t)^{\mathsf{T}}\left[\nabla f(\boldsymbol{y}_t) - \nabla f(\boldsymbol{g}(\boldsymbol{\theta}_t))\right] + \nabla \boldsymbol{g}(\boldsymbol{\theta}_t)^{\mathsf{T}}\nabla f(\boldsymbol{g}(\boldsymbol{\theta}_t))\right]^2_i \le 2\left[\nabla \boldsymbol{g}(\boldsymbol{\theta}_t)^{\mathsf{T}}\nabla f(\boldsymbol{g}(\boldsymbol{\theta}_t))\right]^2_i + \\\nonumber
    &2\left[\nabla \boldsymbol{g}(\boldsymbol{\theta}_t)^{\mathsf{T}}\left[\nabla f(\boldsymbol{y}_t) - \nabla f(\boldsymbol{g}(\boldsymbol{\theta}_t))\right]\right]^2_i \le 2\left[\nabla \mathcal{L}(\boldsymbol{\theta}_t)\right]^2_i + 2\left|\left|\nabla \boldsymbol{g}(\boldsymbol{\theta}_t)^{\mathsf{T}}\right|\right|^2_2\left|\left|\nabla f(\boldsymbol{y}_t) - \nabla f(\boldsymbol{g}(\boldsymbol{\theta}_t))\right|\right|^2_2 \le\\\nonumber
    &2\left[\nabla \mathcal{L}(\boldsymbol{\theta}_t)\right]^2_i + 2M^2_gL^2_f\left|\left|\boldsymbol{y}_t - \boldsymbol{g}(\boldsymbol{\theta}_t)\right|\right|^2_2.
\end{align*}
Hence, we have:
\begin{align}\label{addition_expression_1}
    &\sum_{i=1}^p\frac{\mathbb{E}_t\left[[\overline{\nabla \mathcal{L}(\boldsymbol{\theta}_{t})}]^2_i\right]}{\sqrt{\gamma^{(2)}_t[\boldsymbol{v}_{t-1}]_i} + \xi} \le \sum_{i=1}^p\frac{\frac{4n^2M^2_f\sigma^2_2 + 4M^2_fM^2_gn(n-1) + 2M^2_g\sigma^2_1}{K^{(1)}_t} + \left[\nabla \boldsymbol{g}(\boldsymbol{\theta}_t)^{\mathsf{T}}\nabla f(\boldsymbol{y}_t)\right]^2_i}{\sqrt{\gamma^{(2)}_t[\boldsymbol{v}_{t-1}]_i} + \xi} = \\\nonumber
    &\frac{4n^2M^2_f\sigma^2_2 + 4M^2_fM^2_gn(n-1) + 2M^2_g\sigma^2_1}{\xi K^{(1)}_t}p  + \sum_{i=1}^p\frac{2\left[\nabla \mathcal{L}(\boldsymbol{\theta}_t)\right]^2_i + 2M^2_gL^2_f\left|\left|\boldsymbol{y}_t - \boldsymbol{g}(\boldsymbol{\theta}_t)\right|\right|^2_2}{\sqrt{\gamma^{(2)}_t[\boldsymbol{v}_{t-1}]_i} + \xi} \le \\\nonumber
    &\frac{4n^2M^2_f\sigma^2_2 + 4M^2_fM^2_gn(n-1) + 2M^2_g\sigma^2_1}{\xi K^{(1)}_t}p + 
    \frac{2pM^2_gL^2_f}{\xi}\left|\left|\boldsymbol{y}_t - \boldsymbol{g}(\boldsymbol{\theta}_t)\right|\right|^2_2 + 
    2\sum_{i=1}^p\frac{\left[\nabla \mathcal{L}(\boldsymbol{\theta}_t)\right]^2_i}{\sqrt{\gamma^{(2)}_t[\boldsymbol{v}_{t-1}]_i} + \xi}.
\end{align}
Hence,  expression (\ref{new_change_function}) can be simplified as follows:
\begin{align}\label{newest_change_function}
    &\mathbb{E}_t\left[\mathcal{L}(\boldsymbol{\theta}_{t+1})\right] - \mathcal{L}(\boldsymbol{\theta}_t) \le \\\nonumber
    &-\frac{\alpha_t\left(1 - \gamma^{(1)}_t\right)}{2}\sum_{i=1}^{p}\frac{[\nabla\mathcal{L}(\boldsymbol{\theta}_{t})]^2_i}{\sqrt{\gamma^{(2)}_t[\boldsymbol{v}_{t-1}]_i} + \xi} + \frac{\alpha_t\left(1 - \gamma^{(1)}_t\right)n^2M^2_gL^2_f}{2\xi d^2_t}\left|\left|\boldsymbol{g}(\boldsymbol{\theta}_t) - \boldsymbol{y}_t \right|\right|^2_2 + \\\nonumber
    &\sum_{i=1}^p\frac{\mathbb{E}_t\left[[\overline{\nabla \mathcal{L}(\boldsymbol{\theta}_{t})}]^2_i\right]}{\sqrt{\gamma^{(2)}_t[\boldsymbol{v}_{t-1}]_i} + \xi}\left[\frac{L\alpha^2_t\left(1 - \gamma^{(1)}_t\right)^2}{\xi} + \frac{2\alpha_t\hat{M}(1 - \gamma^{(1)}_t)\sqrt{1 - \gamma^{(2)}_t}}{\xi}\right]   + \\\nonumber &\frac{\alpha_tn\gamma^{(1)}_t\hat{M}^2}{\xi} + \frac{L\alpha^2_t\left(\gamma^{(1)}_t\right)^2n^2\hat{M}^2}{\xi^2} + \frac{\alpha_tn^3\hat{M}^3\left(\gamma^{(1)}_t\right)^2\sqrt{1 - \gamma^{(2)}_t}}{\xi^2 (1 - \gamma^{(1)}_t)}.
\end{align}
where we used (\ref{bound_expressions}):
\begin{align*}
    &-\alpha_t\gamma^{(1)}_t\nabla\mathcal{L}(\boldsymbol{\theta}_t)^{\mathsf{T}}\frac{\boldsymbol{m}_{t-1}}{\sqrt{\gamma^{(2)}_t\boldsymbol{v}_{t-1}}+ \xi} \le \frac{\alpha_t\gamma^{(1)}_t}{\xi}||\nabla\mathcal{L}(\boldsymbol{\theta}_t)||_2||\boldsymbol{m}_{t-1}||_2 \le\frac{\alpha_tn\gamma^{(1)}_t\hat{M}^2}{\xi},\\\nonumber
    &\frac{L\alpha^2_t\left(\gamma^{(1)}_t\right)^2}{\xi}\sum_{i=1}^{p}\frac{[\boldsymbol{m}_{t-1}]^2_i}{\sqrt{\gamma^{(2)}_t[\boldsymbol{v}_{t-1}]_i} + \xi} \le \frac{L\alpha^2_t\left(\gamma^{(1)}_t\right)^2n^2\hat{M}^2}{\xi^2},\\\nonumber
    &\frac{\alpha_tn\hat{M}\left(\gamma^{(1)}_t\right)^2\sqrt{1 - \gamma^{(2)}_t}}{\xi (1 - \gamma^{(1)}_t)}\sum_{i=1}^{p}\frac{|[\nabla\mathcal{L}(\boldsymbol{\theta}_t)]_i[\boldsymbol{m}_{t-1}]_i|}{\sqrt{\gamma^{(2)}_{t}[\boldsymbol{v}_{t-1}]_i} + \xi} \le \frac{\alpha_tn\hat{M}\left(\gamma^{(1)}_t\right)^2\sqrt{1 - \gamma^{(2)}_t}}{\xi^2 (1 - \gamma^{(1)}_t)} \sum_{i=1}^p|[\nabla\mathcal{L}(\boldsymbol{\theta}_t)]_i[\boldsymbol{m}_{t-1}]_i| \le \\\nonumber
    &\frac{\alpha_tn\hat{M}\left(\gamma^{(1)}_t\right)^2\sqrt{1 - \gamma^{(2)}_t}}{2\xi^2 (1 - \gamma^{(1)}_t)}(||\nabla\mathcal{L}(\boldsymbol{\theta}_t)||^2_2 + ||\boldsymbol{m}_{t-1}||^2_2) \le \frac{\alpha_tn^3\hat{M}^3\left(\gamma^{(1)}_t\right)^2\sqrt{1 - \gamma^{(2)}_t}}{\xi^2 (1 - \gamma^{(1)}_t)}.
\end{align*}
Hence, applying (\ref{addition_expression_1}) in (\ref{newest_change_function}) gives:
\begin{align}\label{change_function_1}
    &\mathbb{E}_t\left[\mathcal{L}(\boldsymbol{\theta}_{t-1})\right] - \mathcal{L}(\boldsymbol{\theta}_t) \le \\\nonumber
    &-\frac{\alpha_t\left(1 - \gamma^{(1)}_t\right)}{2}\sum_{i=1}^{p}\frac{[\nabla\mathcal{L}(\boldsymbol{\theta}_{t})]^2_i}{\sqrt{\gamma^{(2)}_t[\boldsymbol{v}_{t-1}]_i} + \xi} + \frac{\alpha_t\left(1 - \gamma^{(1)}_t\right)n^2M^2_gL^2_f}{2\xi d^2_t}\left|\left|\boldsymbol{g}(\boldsymbol{\theta}_t) - \boldsymbol{y}_t \right|\right|^2_2 + \\\nonumber
    &\frac{\alpha_tn\gamma^{(1)}_t\hat{M}^2}{\xi}\left[1 + \frac{L\alpha_tn\gamma^{(1)}_t}{\xi} + \frac{n^2\hat{M}\gamma^{(1)}_t\sqrt{1 - \gamma^{(2)}_t}}{\xi(1 - \gamma^{(1)}_t)}\right] + \left[\frac{L\alpha^2_t\left(1 - \gamma^{(1)}_t\right)^2}{\xi} + \frac{2\alpha_t\hat{M}(1 - \gamma^{(1)}_t)\sqrt{1 - \gamma^{(2)}_t}}{\xi}\right]\times \\\nonumber
    &\left[\frac{4n^2M^2_f\sigma^2_2 + 4M^2_fM^2_gn(n-1) + 2M^2_g\sigma^2_1}{\xi K^{(1)}_t}p + 
    \frac{2pM^2_gL^2_f}{\xi}\left|\left|\boldsymbol{y}_t - \boldsymbol{g}(\boldsymbol{\theta}_t)\right|\right|^2_2 + 
    2\sum_{i=1}^p\frac{\left[\nabla \mathcal{L}(\boldsymbol{\theta}_t)\right]^2_i}{\sqrt{\gamma^{(2)}_t[\boldsymbol{v}_{t-1}]_i} + \xi}\right].
\end{align}

Grouping the terms in (\ref{change_function_1}) gives:

\begin{align}\label{change_function_1_1}
    &\mathbb{E}_t\left[\mathcal{L}(\boldsymbol{\theta}_{t+1})\right] - \mathcal{L}(\boldsymbol{\theta}_t) \le \\\nonumber
    &-\frac{\alpha_t\left(1 - \gamma^{(1)}_t\right)}{2}\left[1 - \frac{8\left[L\alpha_t\left(1 - \gamma^{(1)}_t\right) + \hat{M}\sqrt{1 - \gamma^{(2)}_t}\right]}{\xi} \right]\sum_{i=1}^{p}\frac{[\nabla\mathcal{L}(\boldsymbol{\theta}_{t})]^2_i}{\sqrt{\gamma^{(2)}_t[\boldsymbol{v}_{t-1}]_i} + \xi}  
    +  \\\nonumber
    &\frac{\alpha_t\left(1 - \gamma^{(1)}_t\right)}{2\xi}M^2_gL^2_f\left[\frac{n^2}{d^2_t} + \frac{8p\left[L\alpha_t\left(1 - \gamma^{(1)}_t\right) + \hat{M}\sqrt{1 - \gamma^{(2)}_t}\right]}{\xi} \right]\left|\left|\boldsymbol{g}(\boldsymbol{\theta}_t) - \boldsymbol{y}_t \right|\right|^2_2 +\\\nonumber
    &\frac{8\alpha_t\left(1 - \gamma^{(1)}_t\right)}{\xi^2K^{(1)}_t}\left[L\alpha_t\left(1 - \gamma^{(1)}_t\right) + \hat{M}\sqrt{1 - \gamma^{(2)}_t}\right]\left(n^2pM^2_f\sigma^2_2 + M^2_fM^2_gnp(n-1) + M^2_gp\sigma^2_1\right) + \\\nonumber
    &\frac{\alpha_tn\gamma^{(1)}_t\hat{M}^2}{\xi}\left[1 + \frac{\gamma^{(1)}_tn^2}{\xi(1 - \gamma^{(1)}_t)}\left[L\alpha_t(1 - \gamma^{(1)}_t) + \hat{M}\sqrt{1-\gamma^{(2)}_t}\right]\right] = \\\nonumber
    &-\frac{\alpha_t\left(1 - \gamma^{(1)}_t\right)}{2}\left[1 - 8\mathcal{C}_t \right]\sum_{i=1}^{p}\frac{[\nabla\mathcal{L}(\boldsymbol{\theta}_{t})]^2_i}{\sqrt{\gamma^{(2)}_t[\boldsymbol{v}_{t-1}]_i} + \xi}  
    +  \frac{\alpha_t\left(1 - \gamma^{(1)}_t\right)M^2_gL^2_f}{2\xi}\left[\frac{n^2}{d^2_t} + 8p\mathcal{C}_t\right]\left|\left|\boldsymbol{g}(\boldsymbol{\theta}_t) - \boldsymbol{y}_t \right|\right|^2_2 +\\\nonumber
    &\frac{8\alpha_t\mathcal{C}_t\left(1 - \gamma^{(1)}_t\right)}{\xi^2K^{(1)}_t}\left(n^2pM^2_f\sigma^2_2 + M^2_fM^2_gnp(n-1) + M^2_gp\sigma^2_1\right) + \frac{\alpha_tn\gamma^{(1)}_t\hat{M}^2}{\xi}\left[1 + \frac{\gamma^{(1)}_tq^2\mathcal{C}_t}{\left(1  - \gamma^{(1)}_{t}\right)}\right] \le \\\nonumber
    &-\frac{\alpha_t\left(1 - \gamma^{(1)}_t\right)}{2(n\hat{M} + \xi)}\left[1 - 8\mathcal{C}_t \right]||\nabla\mathcal{L}(\boldsymbol{\theta}_{t})||^2_2  
    +  \frac{\alpha_t\left(1 - \gamma^{(1)}_t\right)M^2_gL^2_f}{2\xi}\left[\frac{n^2}{d^2_t} + 8p\mathcal{C}_t\right]\left|\left|\boldsymbol{g}(\boldsymbol{\theta}_t) - \boldsymbol{y}_t \right|\right|^2_2 +\\\nonumber
    &\frac{8\alpha_t\mathcal{C}_t\left(1 - \gamma^{(1)}_t\right)}{\xi^2K^{(1)}_t}\left(n^2pM^2_f\sigma^2_2 + M^2_fM^2_gnp(n-1) + M^2_gp\sigma^2_1\right) + \frac{\alpha_tn\gamma^{(1)}_t\hat{M}^2}{\xi}\left[1 + \frac{\gamma^{(1)}_tn^2\mathcal{C}_t}{\left(1  - \gamma^{(1)}_{t}\right)}\right].
\end{align}
where we use notation $\mathcal{C}_{t} = \frac{L\alpha_t\left(1 - \gamma^{(1)}_t\right) + \hat{M}\sqrt{1 - \gamma^{(2)}_t}}{\xi}$ and $[\boldsymbol{v}_{t-1}]_i \le n^2\hat{M}^2$.\\
Next, let us denote $\mathbb{E}\left[\right]$ be the expectation with respect to all randomness in \text{CI-VI} Algorithm. By the low of total expectation:
\begin{equation*}
    \mathbb{E}\left[\mathbb{E}_t\left[\zeta_t\right]\right]  = \mathbb{E}\left[\mathbb{E}_{K^{(1)}_t,d_t,K^{(2)}_t }\left[\zeta_t\Big| \boldsymbol{\theta}_t\right]\right] = \mathbb{E}\left[\zeta_t\right]
\end{equation*}
for any $t-$measurable\footnote{Random variable is called  $t-$measurable if its affected by the randomness induced in  the first $t$ rounds.} random variable $\zeta_t$. Hence, taking expectation $\mathbb{E}$ from both sides of (\ref{change_function_1_1}) gives:
\begin{align}\label{change_function_1_2}
    &\mathbb{E}\left[\mathcal{L}(\boldsymbol{\theta}_{t+1}) - \mathcal{L}(\boldsymbol{\theta}_t)\right] \le \\\nonumber
    &-\frac{\alpha_t\left(1 - \gamma^{(1)}_t\right)}{2(n\hat{M} + \xi)}\left[1 - 8\mathcal{C}_t \right]\mathbb{E}\left[||\nabla\mathcal{L}(\boldsymbol{\theta}_{t})||^2_2\right]  
    +  \frac{\alpha_t\left(1 - \gamma^{(1)}_t\right)M^2_gL^2_f}{2\xi}\left[\frac{n^2}{d^2_t} + 8p\mathcal{C}_t\right]\mathbb{E}\left[\left|\left|\boldsymbol{g}(\boldsymbol{\theta}_t) - \boldsymbol{y}_t \right|\right|^2_2\right] +\\\nonumber
    &\frac{8\alpha_t\mathcal{C}_t\left(1 - \gamma^{(1)}_t\right)}{\xi^2K^{(1)}_t}\left(n^2pM^2_f\sigma^2_2 + M^2_fM^2_gnp(n-1) + M^2_gp\sigma^2_1\right) + \frac{\alpha_tq\gamma^{(1)}_t\hat{M}^2}{\xi}\left[1 + \frac{\gamma^{(1)}_tn^2\mathcal{C}_t}{\left(1  - \gamma^{(1)}_{t}\right)}\right].
\end{align}
Let $\alpha_t = \frac{C_{\alpha}}{t^a}$, $\beta_t = \frac{C_{\beta}}{t^b}$, $K^{(1)}_t = C_{1}t^{c}$, and  $K^{(2)}_t = C_{2}t^{e}$ for some constants $C_{\alpha},C_{\beta}, C_1, C_{3}, a,b,c,e > 0$ such that $(2a-2b)\notin (-1,0)$, $0 < b \le 1$. Following Corollary \ref{cor_1} we have:
\begin{equation*}
    \mathbb{E}\left[\left|\left|\boldsymbol{g}(\boldsymbol{\theta}_t) - \boldsymbol{y}_t \right|\right|^2_2\right] \le \frac{L^2_gC^2_{\mathcal{D}}}{2}\frac{1}{t^{4a-4b}} + 2C^2_{\mathcal{E}}\frac{1}{t^{b+e}}.
\end{equation*}
for some constants $C_{\mathcal{D}},C_{\mathcal{E}}> 0$, and, therefore,  (\ref{change_function_1_2}) can be written as:
\begin{align}\label{change_function_1_4}
    &\mathbb{E}\left[\mathcal{L}(\boldsymbol{\theta}_{t+1}) - \mathcal{L}(\boldsymbol{\theta}_t)\right] \le \\\nonumber
    &-\frac{C_{\alpha}\left(1 - \gamma^{(1)}_t\right)}{2t^a(n\hat{M} + \xi)}\left[1 - 8\mathcal{C}_t \right]\mathbb{E}\left[||\nabla\mathcal{L}(\boldsymbol{\theta}_{t})||^2_2\right]  
    +  \frac{C_{\alpha}\left(1 - \gamma^{(1)}_t\right)M^2_gL^2_f}{2t^a\xi}\left[\frac{n^2}{d^2_t} + 8p\mathcal{C}_t\right]\left[\frac{L^2_gC^2_{\mathcal{D}}}{2t^{4a-4b}} + \frac{2C^2_{\mathcal{E}}}{t^{b+e}}\right] +\\\nonumber
    &\frac{8C_{\alpha}\mathcal{C}_t\left(1 - \gamma^{(1)}_t\right)}{t^{a+c}\xi^2C_1}\left(n^2pM^2_f\sigma^2_2 + M^2_fM^2_gnp(n-1) + M^2_gp\sigma^2_1\right) + \frac{C_{\alpha}n\gamma^{(1)}_t\hat{M}^2}{t^a\xi}\left[1 + \frac{\gamma^{(1)}_tn^2\mathcal{C}_t}{\left(1  - \gamma^{(1)}_{t}\right)}\right].
\end{align}
with $\mathcal{C}_t = \frac{L\frac{C_{\alpha}}{t^a}\left(1 - \gamma^{(1)}_t\right) + \hat{M}\sqrt{1 - \gamma^{(2)}_t}}{\xi}$. By choosing $\gamma^{(2)}_t = 1 - \frac{C^2_{\alpha}}{t^{2a}}(1 - \gamma^{(1)}_t)^2$ such that $\sqrt{1 - \gamma^{(2)}_t} = \frac{C_{\alpha}}{t^a}(1 - \gamma^{(1)}_t)$ we have $\mathcal{C}_t = C_{\alpha}\frac{(L + \hat{M})}{\xi}\frac{1 - \gamma^{(1)}_{t}}{t^a} \le C_{\alpha}\frac{(L + \hat{M})}{\xi}$. By choosing $C_{\alpha} \le \frac{\xi}{16p(L + \hat{M})}$ we have $16p\mathcal{C}_t \le 1$, hence, (\ref{change_function_1_4}) can be simplified as:
\begin{align}\label{change_function_1_5}
    &\mathbb{E}\left[\mathcal{L}(\boldsymbol{\theta}_{t+1}) - \mathcal{L}(\boldsymbol{\theta}_t)\right] \le \\\nonumber
    &-\frac{C_{\alpha}\left(1 - \gamma^{(1)}_t\right)}{2t^a(n\hat{M} + \xi)}\left[1  - \frac{1}{2p} \right]\mathbb{E}\left[||\nabla\mathcal{L}(\boldsymbol{\theta}_{t})||^2_2\right]  
    +  \frac{C_{\alpha}\left(1 - \gamma^{(1)}_t\right)M^2_gL^2_f}{2t^a\xi}\left[\frac{n^2}{d^2_t}+\frac{1}{2}\right]\left[\frac{L^2_gC^2_{\mathcal{D}}}{2t^{4a-4b}} + \frac{2C^2_{\mathcal{E}}}{t^{b+e}}\right] +\\\nonumber
    &\frac{C_{\alpha}\left(1 - \gamma^{(1)}_t\right)}{2t^{a+c}\xi^2C_1}\left(n^2M^2_f\sigma^2_2 + \hat{M}^2n^2 + M_g^2\sigma^2_1\right) + \frac{C_{\alpha}n\hat{M}^2}{\xi}\left[1 + \frac{n^2\gamma^{(1)}_t}{16p\left(1  - \gamma^{(1)}_{t}\right)}\right]\frac{\gamma^{(1)}_t}{t^a} \le \\\nonumber
    &-\frac{C_{\alpha}\left(1 - \gamma^{(1)}_t\right)}{4t^a(n\hat{M} + \xi)}\mathbb{E}\left[||\nabla\mathcal{L}(\boldsymbol{\theta}_{t})||^2_2\right]  
    +  \frac{2n^2C_{\alpha}\left(1 - \gamma^{(1)}_t\right)M^2_gL^2_f}{t^a\xi}\left[\frac{L^2_gC^2_{\mathcal{D}}}{2t^{4a-4b}} + \frac{2C^2_{\mathcal{E}}}{t^{b+e}}\right] +\\\nonumber
    &\frac{C_{\alpha}\left(n^2M^2_f\sigma^2_2 + \hat{M}^2n^2 + M_g^2\sigma^2_1\right)}{2\xi^2C_1}\frac{\left(1 - \gamma^{(1)}_t\right)}{t^{a+c}} + \frac{C_{\alpha}n\hat{M}^2}{\xi(1 - \gamma^{(1)}_t)}\frac{\gamma^{(1)}_t}{t^a}.
\end{align}
Assuming that $\gamma^{(1)}_t = C_{\gamma}\mu^{t}$ for some $C_{\gamma}\in\left(0,\frac{1}{2}\right)$ and $\mu\in[0,1)$, then expression (\ref{change_function_1_5}) can be simplified as:
\begin{align*}
    &\mathbb{E}\left[\mathcal{L}(\boldsymbol{\theta}_{t+1}) - \mathcal{L}(\boldsymbol{\theta}_t)\right] \le \\\nonumber
    &-\frac{C_{\alpha}}{8t^a(n\hat{M} + \xi)}\mathbb{E}\left[||\nabla\mathcal{L}(\boldsymbol{\theta}_{t})||^2_2\right]  
    +  \frac{2n^2C_{\alpha}M^2_gL^2_f}{\xi}\left[\frac{L^2_gC^2_{\mathcal{D}}}{2t^{5a-4b}} + \frac{2C^2_{\mathcal{E}}}{t^{a +b+e}}\right] +\\\nonumber
    &\frac{C_{\alpha}\left(n^2M^2_f\sigma^2_2 + \hat{M}^2n^2 + M_g^2\sigma^2_1\right)}{2\xi^2C_1}\frac{1}{t^{a+c}} + \frac{C_{\alpha}n\hat{M}^2}{\xi}\frac{\mu^t}{t^a}.
\end{align*}
Finally, choosing 
\begin{equation*}
C_{\alpha} = \min\Big\{\frac{\xi}{16p(L + \hat{M})}, \frac{e}{n^2M^2_gL^2_fL^2_gC^2_{\mathcal{D}}},\frac{e}{4n^2M^2_gL^2_fC^2_{\mathcal{E}}},  \frac{2C_1\xi^2}{\left(n^2M^2_f\sigma^2_2 + \hat{M}^2n^2 + M_g^2\sigma^2_1\right)}, \frac{\xi}{n\hat{M}^2}\Big\}.
\end{equation*}
we have:
\begin{align*}
    &\mathbb{E}\left[\mathcal{L}(\boldsymbol{\theta}_{t+1}) - \mathcal{L}(\boldsymbol{\theta}_t)\right] \le \\\nonumber
    &-\frac{C_{0}}{t^a}\mathbb{E}\left[||\nabla\mathcal{L}(\boldsymbol{\theta}_{t})||^2_2\right]  
    +  \frac{1}{t^{5a-4b}} + \frac{1}{t^{a + b + e}} + \frac{1}{t^{a + c}} + \frac{\mu^{t}}{t^a}.
\end{align*}
where $C_0 = \frac{C_{\alpha}}{8(q\hat{M} + \xi)}$. Therefore,
\begin{align}\label{change_function_1_8}
    &\mathbb{E}\left[||\nabla\mathcal{L}(\boldsymbol{\theta}_{t})||^2_2\right] \le \\\nonumber &\frac{t^a}{C_0}\mathbb{E}\left[\mathcal{L}(\boldsymbol{\theta}_{t}) - \mathcal{L}(\boldsymbol{\theta}_{t+1})\right] + \frac{1}{C_0t^{4a-4b}} + \frac{1}{C_0t^{b + e}} + \frac{1}{C_0t^{c}} + \frac{\mu^{t}}{C_0}.
\end{align}
Taking summation in (\ref{change_function_1_8}) over $t=1,\ldots,T$ and dividing the result by $T$ gives:
\begin{align}\label{change_function_1_9}
    &\frac{\sum_{t=1}^T\mathbb{E}\left[||\nabla\mathcal{L}(\boldsymbol{\theta}_{t})||^2_2\right]}{T} \le \\\nonumber &\frac{\sum_{t=1}^Tt^a\mathbb{E}\left[\mathcal{L}(\boldsymbol{\theta}_{t}) - \mathcal{L}(\boldsymbol{\theta}_{t+1})\right]}{C_0T} + \frac{1}{C_0T}\sum_{t=1}^T\left[\frac{1}{t^{4a-4b}} + \frac{1}{t^{b + e}} + \frac{1}{t^{c}}\right] + \frac{1}{C_0T(1 - \mu)}.
\end{align}
Notice, using first order concavity condition for function $f(t) = t^a$ (if $a\in[0,1]$) and Assumption \textbf{\text{1}}: 
\begin{align*}
    &\sum_{t=1}^Tt^a\mathbb{E}\left[\mathcal{L}(\boldsymbol{\theta}_{t}) - \mathcal{L}(\boldsymbol{\theta}_{t+1})\right] = \\\nonumber
    &\mathcal{L}(\boldsymbol{\theta}_1) + \sum_{t=2}^{T}\left[(t+1)^a - t^a\right]\mathbb{E}\left[\mathcal{L}(\boldsymbol{\theta}_{t})\right] \le B_f + \sum_{t=2}^Tat^{a-1}\mathbb{E}\left[\mathcal{L}(\boldsymbol{\theta}_{t})\right] \le \mathcal{L}(\boldsymbol{\theta}_1) + B_f\sum_{t=2}^Tat^{a-1}\le\\\nonumber
    &B_f + B_faT^{a}.
\end{align*} 
Hence, for (\ref{change_function_1_9}) we have:
\begin{align}\label{change_function_1_10}
    &\frac{\sum_{t=1}^T\mathbb{E}\left[||\nabla\mathcal{L}(\boldsymbol{\theta}_{t})||^2_2\right]}{T} \le \frac{B_f}{C_0}\frac{1}{T} + \frac{aB_f}{C_0}\frac{1}{T^{1-a}} + \frac{1}{T}\mathcal{O}\left(T^{4b-4a+1}\mathbb{I}\{4a-4b<1\}\right) +\\\nonumber
    &\frac{1}{T}\mathcal{O}\left(\log T\mathbb{I}\{4a-4b = 1\}\right) + \frac{1}{T}\mathcal{O}\left( T^{1-b-e}\mathbb{I}\{b+e \ne 1\}+ \log T\mathbb{I}\{b+e = 1\}\right)\\\nonumber
    &\frac{1}{T}\mathcal{O}\left( T^{1-c}\mathbb{I}\{c \ne 1\}+ \log T\mathbb{I}\{c = 1\}\right) + \frac{1}{C_0(1-\mu)}\frac{1}{T} = \\\nonumber
    &\left(\frac{B_f}{C_0} + \frac{1}{C_0(1-\mu)}\right)\frac{1}{T} + \frac{aB_f}{C_0}\frac{1}{T^{1-a}} + \mathcal{O}\left(\frac{1}{T^{4a-4b}}\mathbb{I}\{4a-4b\neq 1\} + \frac{1}{T^{b+e}}\mathbb{I}\{b+e \ne 1\} + \frac{1}{T^c}\mathbb{I}\{c \ne 1\}\right)+\\\nonumber
    &\mathcal{O}\left(\frac{\log T}{T}\right)\left[\mathbb{I}\{4a-4b =  1\} + \mathbb{I}\{b+e = 1\} + \mathbb{I}\{c = 1\}\right].
\end{align}
where $\mathbb{I}\{\textbf{\text{condition}}\} = 1 $ if $\textbf{\text{condition}}$ is satisfied, and $0$ otherwise. Notice, that oracle complexity per iteration of Algorithm 1 is given by $\mathcal{O}(t^{\max\{c,e\}})$. Hence, after $T$ iterations, the total first order oracles complexity is given by $\mathcal{O}\left(T^{1 + \max\{c,e\}}\right)$. Let us denote $\phi(a,b,c,e) = \min\{1-a,  4a-4b,  b+e,  c  \}$. Then, ignoring logarithmic factors $\log T$ we have:
\begin{equation*}
    \frac{1}{T}\sum_{t=1}^T\mathbb{E}\left[||\nabla\mathcal{L}(\boldsymbol{\theta}_{t})||^2_2\right]\le \mathcal{O}\left(\frac{1}{T^{\phi(a,b,c,e)}}\right) \le \delta
\end{equation*}
implies $T = \mathcal{O}\left(\frac{1}{\delta^{\frac{1}{\phi(a,b,c,e)}}}\right)$, and for the first oracles complexity we have the following expression $\Psi(a,b,e,c) = \frac{1}{\delta^{\frac{1 + \max\{c,e\}}{\phi(a,b,c,e)}}}$. Hence, we have the following optimisation problem to find the optimal setup of parameters $a,b,c,e$:
\begin{align}
    &\min_{a,b,c,e}\frac{1 + \max\{c,e\}}{\min\{1-a,  4a-4b,  b+e,  c  \}}\\\nonumber
    &s.t.\ \ a \ge b \\\nonumber
    &0\le a\le 1\\\nonumber
    &0< b\le 1\\\nonumber
    &4a - 4b \le 1\\\nonumber
    &c \ge 0\\\nonumber
    &e \ge 0
\end{align}
Introducing slackness parameter $\varkappa\in(0,0.001]$ let us consider the following setup:
\begin{align}\label{feasible_setup}
    &a = \frac{1}{5 + 4\varkappa}; \ \ c = \frac{4}{5 + 4\varkappa}; \ \ e = \frac{4}{5 + 4\varkappa},\ \  b = \frac{\varkappa}{5 + 4\varkappa}.
\end{align}
then all inequality constraints are satisfied and the overall complexity is given by:
\begin{equation*}
    \Psi(a^*,b^*,c^*,e^*) = \delta^{-\frac{9+4\varkappa}{4-4\varkappa}}.
\end{equation*}
This implies that \text{CI-VI} Algorithm requires $\mathcal{O}\left(\frac{1}{\delta^{\frac{5+4\varkappa}{4}}}\right)$ and  in expectation outputs $\delta-$approximate first order stationary point of function $\mathcal{L}(\boldsymbol{\theta})$ and requires $\mathcal{O}\left(\delta^{-\frac{9+4\varkappa}{4-4\varkappa}}\right)$ calls to the oracles $\mathcal{FOO}_f[\cdot,\cdot]$ and $\mathcal{FOO}_g[\cdot,\cdot]$. Choosing $\varkappa$ close to $0$ establishes the statement of the theorem.
\end{proof}

\end{document}